\documentclass{article}
\usepackage[preprint]{neurips_2026}

\usepackage{microtype}
\usepackage{amsthm}
\usepackage{amssymb}
\usepackage{graphicx}
\usepackage{mathtools}
\usepackage{multirow}
\usepackage{subcaption}
\usepackage{hyperref}
\usepackage[capitalize,noabbrev]{cleveref}
\usepackage{url}
\usepackage{xcolor}
\usepackage{tabularx}
\usepackage{booktabs}
\usepackage{adjustbox}
\usepackage{color}
\usepackage{aliascnt}
\usepackage{pifont}

\theoremstyle{plain}
\newtheorem{theorem}{Theorem}[section]

\newaliascnt{proposition}{theorem}
\newtheorem{proposition}[proposition]{Proposition}
\aliascntresetthe{proposition}

\newaliascnt{lemma}{theorem}
\newtheorem{lemma}[lemma]{Lemma}
\aliascntresetthe{lemma}

\theoremstyle{definition}
\newaliascnt{definition}{theorem}
\newtheorem{definition}[definition]{Definition}
\aliascntresetthe{definition}

\newaliascnt{assumption}{theorem}
\newtheorem{assumption}[assumption]{Assumption}
\aliascntresetthe{assumption}

\definecolor{blindred}{RGB}{222,45,38}
\definecolor{blindblue}{RGB}{8,69,148}
\definecolor{blindorange}{RGB}{230,85,13}
\definecolor{blindpink}{RGB}{221,28,119}
\definecolor{blindgreen}{RGB}{44,162,95}
\definecolor{blindgold}{RGB}{254,227,145}
\definecolor{blindcyan}{RGB}{158,202,225}

\begin{document}

\def\Tableref#1{Table~\ref{#1}}
\def\Appref#1{Appendix~\ref{#1}}
\def\figref#1{figure~\ref{#1}}
\def\Figref#1{Figure~\ref{#1}}
\def\twofigref#1#2{figures \ref{#1} and \ref{#2}}
\def\quadfigref#1#2#3#4{figures \ref{#1}, \ref{#2}, \ref{#3} and \ref{#4}}
\def\secref#1{section~\ref{#1}}
\def\Secref#1{Section~\ref{#1}}
\def\twosecrefs#1#2{sections \ref{#1} and \ref{#2}}
\def\secrefs#1#2#3{sections \ref{#1}, \ref{#2} and \ref{#3}}
\def\eqref#1{equation~\ref{#1}}
\def\Eqref#1{Equation~\ref{#1}}
\def\plaineqref#1{\ref{#1}}
\def\chapref#1{chapter~\ref{#1}}
\def\Chapref#1{Chapter~\ref{#1}}
\def\rangechapref#1#2{chapters\ref{#1}--\ref{#2}}
\def\algref#1{algorithm~\ref{#1}}
\def\Algref#1{Algorithm~\ref{#1}}
\def\twoalgref#1#2{algorithms \ref{#1} and \ref{#2}}
\def\Twoalgref#1#2{Algorithms \ref{#1} and \ref{#2}}
\def\partref#1{part~\ref{#1}}
\def\Partref#1{Part~\ref{#1}}
\def\twopartref#1#2{parts \ref{#1} and \ref{#2}}

\def\ceil#1{\lceil #1 \rceil}
\def\floor#1{\lfloor #1 \rfloor}
\def\1{\bm{1}}
\newcommand{\train}{\mathcal{D}}
\newcommand{\valid}{\mathcal{D_{\mathrm{valid}}}}
\newcommand{\test}{\mathcal{D_{\mathrm{test}}}}

\newcommand{\pdata}{p_{\rm{data}}}
\newcommand{\ptrain}{\hat{p}_{\rm{data}}}
\newcommand{\Ptrain}{\hat{P}_{\rm{data}}}
\newcommand{\pmodel}{p_{\rm{m}}}
\newcommand{\Pmodel}{P_{\rm{m}}}
\newcommand{\ptildemodel}{\widetilde{p}_{\rm{m}}}

\newcommand{\pseudoexp}[2]{\widetilde{\mathbb{E}}_{#1} \left[ #2 \right]}
\newcommand{\pseudokl}[2]{\widetilde{\mathrm{KL}} \left( #1 \mid \mid #2 \right)}
\newcommand{\pseudomi}[2]{\widetilde{\mathrm{I}} \left( #1 ; #2 \right)}
\newcommand{\mi}[2]{{\mathrm{I}} \left( #1 ; #2 \right)}

\title{Robust Experimental Design via \\ Generalised Bayesian Inference}

\author{%
  Yasir Zubayr Barlas \\
  The University of Manchester\\
  \texttt{yasir.barlas@manchester.ac.uk} \\
  \And
  Sabina J. Sloman \\
  University of Birmingham\\
  \texttt{s.sloman@bham.ac.uk} \\
  \And
  Samuel Kaski \\
  ELLIS Institute Finland\\
  Aalto University\\
  The University of Manchester\\
  \texttt{samuel.kaski@aalto.fi}
}

\maketitle

\begin{abstract}
  Bayesian optimal experimental design is a principled framework for conducting experiments that leverages Bayesian inference to quantify how much information one can expect to gain from selecting a certain design. However, accurate Bayesian inference relies on the assumption that one's statistical model of the data-generating process is correctly specified. If this assumption is violated, Bayesian methods can lead to poor inference and estimates of information gain. Generalised Bayesian (or Gibbs) inference is a more robust probabilistic inference framework that replaces the likelihood in the Bayesian update by a suitable loss function. In this work, we present \textit{Generalised Bayesian Optimal Experimental Design (GBOED)}, an extension of generalised Bayesian inference to the experimental design setting which achieves robustness in both design and inference. Using a generalisation of Bayesian decision theory, we derive a new acquisition function, the \textit{Gibbs expected information gain}. Our empirical results demonstrate that GBOED can enhance robustness to outliers and incorrect assumptions about the outcome noise distribution.
\end{abstract}

\section{Introduction}
\label{introduction1}

Many real-world settings are characterised by heavy resource and time constraints on data collection.
Effective learning then requires practitioners to carefully select these scarce data.
Bayesian optimal experimental design (BOED) is a framework to optimise data acquisition in such settings \citep{atkinson1992optimum, ryanboed, rainforth2024modern, Huan_Jagalur_Marzouk_2024}. BOED has found application in scores of disciplines, such as systems biology \citep{bioinformaticsbtt436, Pauwels_Lajaunie_Vert_2014}, psychology \citep{eigpsycho1, eigpsycho2}, and (medical) imaging \citep{KARIMI2021113489, hyvonen2024bayesian}. 

BOED leverages Bayesian inference for updating beliefs about parameters of interest \citep{rainforth2024modern}. A key assumption is that the data are generated by a statistical model whose structure is known and depends on parameter values about which we are uncertain. The Bayesian framework offers a coherent approach to modelling this uncertainty by constructing a posterior distribution on the basis of observed data. 

BOED extends Bayesian inference to specify how the modeller wishes to allocate their resources to set the design. In BOED, the optimal design is the one that maximises the \textit{expected information gain} about our parameters of interest \citep{lindley}.
Thus, in the context of BOED, the modeller relies on the model \textit{twice}: to design experiments, and to make inferences.

The assumption that the assumed statistical model is well-specified --- i.e., that the model is able to accurately capture the true data-generating process (DGP) --- is often broken. If accurate domain knowledge is available, scientists often choose to represent this domain knowledge as simple and tractable models that exclude some aspects of reality. Moreover, accurate domain knowledge can simply be unavailable.

\textit{Model misspecification} refers to the case where the assumed statistical model cannot fully capture the true DGP, due to intentional simplification or unavailable domain knowledge \citep{WALKER20131621}. Model misspecification detrimentally affects inferences made through the Bayesian framework \citep{berk1966limiting, kleijn2012bernstein-von-mises}.
In the BOED context, it can lead to uninformative or misleading designs \citep{vincent2017darc, sloman2022characterizingrobustnessbayesianadaptive, tang2026representativeinformativedeamplifyingrequirements}.

The possible damaging effect of model misspecification on the effectiveness of BOED is illustrated in \Cref{locfindingplot}.
The top row shows the designs selected by BOED over the course of an experiment designed to locate two signal-emitting objects (\textcolor{blindred}{red crosses}), using the signal intensity measured at selected coordinates on a 2-dimensional space.
When the model is well-specified, BOED selects designs that concentrate closely around the objects, providing the experimenter with ample information to precisely locate the objects.
On the other hand, when data points have some probability of being corrupted (a misspecified model), BOED is incapable of accurately assessing which areas of the design space require further exploration, wasting precious resources and failing to locate the objects.
A design selection method robust to model misspecification would be less easily misled by an inaccurate model.

\begin{figure}[ht!]
    \centering
    \includegraphics[width=\linewidth]{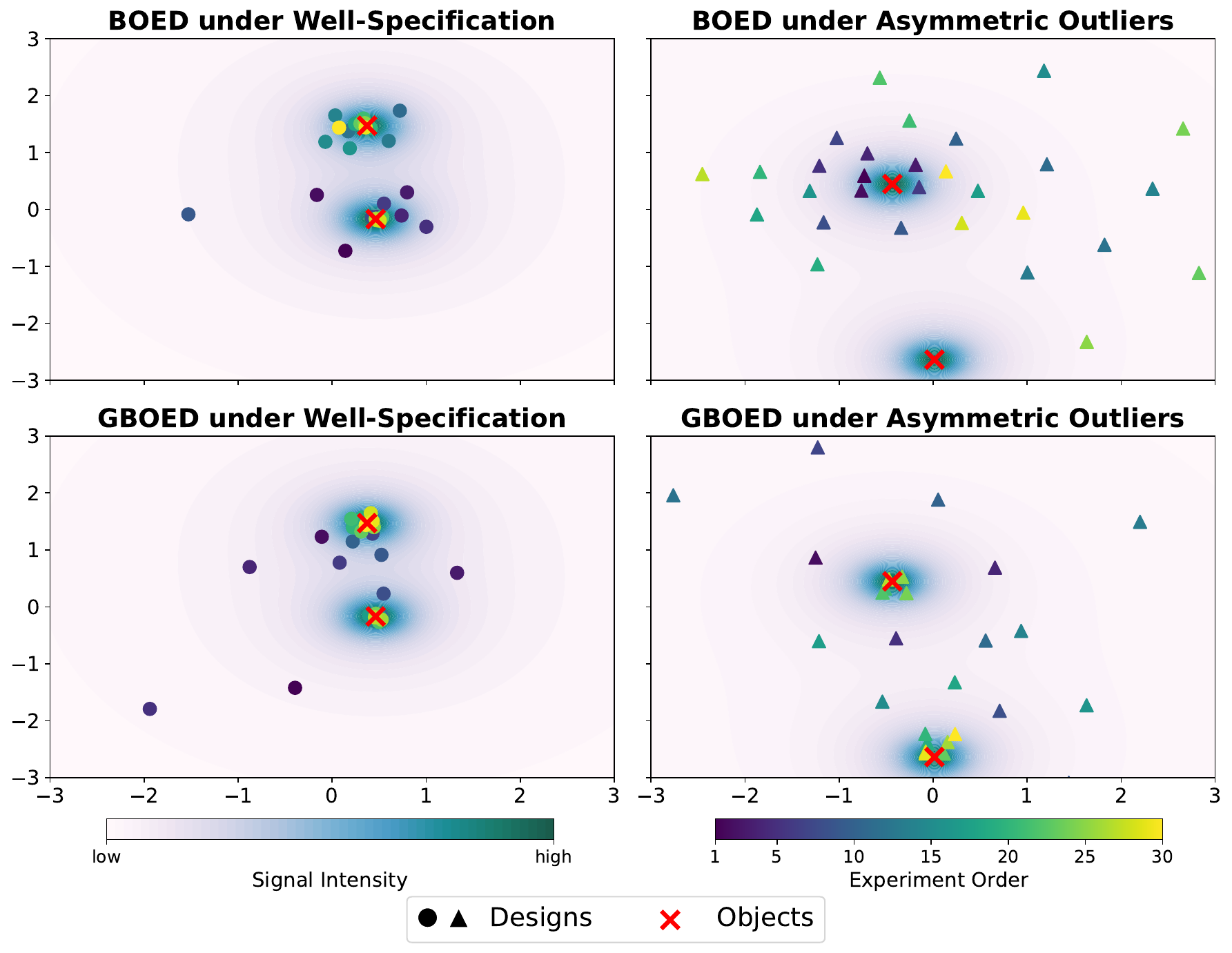}
    \caption{Designs selected by BOED (\textit{top row}) and GBOED (\textit{bottom row}) in a 2D location finding example.  
    Designs that cluster around the objects (\textcolor{blindred}{red crosses}) are most informative in determining the objects' locations.
    \textit{Left column}: Well-specified scenario. BOED and GBOED select designs that correctly cluster around and identify the objects. \textit{Right column}: Misspecified scenario. BOED is unable to explore well, failing to locate both objects. GBOED explores better and clusters around both objects, correctly locating them.}
    \label{locfindingplot}
\end{figure}

\begin{table*}
\caption{Comparison of experimental design frameworks. GBOED is coherent and inherits the robustness of GBI.}
\centering
\renewcommand{\arraystretch}{1.05}
\setlength{\tabcolsep}{4pt}
\begin{adjustbox}{max width=\textwidth}
    \centering
    \begin{tabular}{lccccc}
        \toprule
        \multirow{2}{*}{Method}
        & \multirow{2}{*}{Robust in inference}
        & \multirow{2}{*}{Robust in design}
        & Coherent in decision- \\
        &
        &
        & making and inference \\
        \midrule
        BOED 
            & \textcolor{blindred}{\ding{55}} 
            & \textcolor{blindred}{\ding{55}}
            & \textcolor{blindgreen}{\ding{51}}
            & \\
        R-IDeA \citep{tang2026representativeinformativedeamplifyingrequirements} 
            & \textcolor{blindred}{\ding{55}}
            & \textcolor{blindgreen}{\ding{51}}
            & \textcolor{blindred}{\ding{55}}
            & \\
        \citet{overstall2023gibbsoptimaldesignexperiments} 
            & \textcolor{blindgreen}{\ding{51}}
            & \textcolor{blindgreen}{\ding{51}}
            & \textcolor{blindred}{\ding{55}}
            & \\
        GBOED 
            & \textcolor{blindgreen}{\ding{51}}
            & \textcolor{blindgreen}{\ding{51}}
            & \textcolor{blindgreen}{\ding{51}}
            & \\
        \bottomrule
    \end{tabular}
\end{adjustbox}
\label{expdesigncomparison}
\end{table*}

\textit{Generalised Bayesian inference} (GBI), or \textit{Gibbs inference}, is a framework for updating beliefs about parameters of interest which has been shown to increase robustness to many forms of model misspecification \citep{Bissiri_2016, knoblauch, martin2022direct}. In GBI, a loss function replaces the likelihood in the traditional Bayesian update, often resulting in robust inference. 

In this work, we leverage GBI to address model misspecification in experimental design. We introduce \textit{Generalised Bayesian Optimal Experimental Design (GBOED)}, which uses GBI in both design and inference. 

The robustness offered by GBOED is illustrated in the bottom row of \Cref{locfindingplot}, which shows the designs selected over the course of the location finding experiment described above.
As a result of its ``awareness'' that the model that informs it has a limited ability to assess potential information gains, GBOED can better determine which regions of the design space should be explored, resulting in improved design and inference. This helps prevent wasted resources under misspecification, where BOED can perform poorly.

\paragraph{Contributions} 
We propose a generalisation of BOED, GBOED, which seeks to tackle the model misspecification problem. For this, we generalise Bayesian decision theory to allow decision-making through loss-based densities. Using this generalisation, we derive a generalised form of the Bayesian expected information gain, which we coin the \textit{Gibbs expected information gain}, and present methods to approximate this utility function. We primarily consider the weighted score matching loss function \citep{robustgpr} for GBI, and propose a novel parametrisation that tailors it to the data-scarce experimental design setting.
We provide a number of empirical results both in the main paper and appendices, detailing the contexts in which GBOED outperforms BOED, and the benefits of using our proposed utility function over alternatives. 
Overall, our results show that GBOED can perform on par with BOED when the model is well-specified, and tends to outperform BOED when the model is misspecified.

\section{Related Work}
\label{sec_relatedwork}

Many robust experimental design frameworks adopt the $\mathcal{M}$-closed setting, where the true model is assumed to lie within a known set of candidates. BOED has been applied to select the model from this set that best explains the data \citep{cavagnaro, Hainy_2022}. Other approaches compute the expected information gain as an expectation over data generated under this set \citep{catanach2023metrics}, or under a single alternative model which may better capture the true DGP \citep{overstall2022bayesian}. Alternative acquisition functions mitigate the effects of model misspecification by taking into account key components of generalisation error, namely R-IDeA \citep{tang2026representativeinformativedeamplifyingrequirements}. \citet{overstall2023gibbsoptimaldesignexperiments} use a flexible representation of the true DGP to compute expected utility in the GBI setting. Unlike ours, their approach allows for incoherence between the way inference is performed and one's expectations of the data that will be encountered. \citet{overstall2023gibbsoptimaldesignexperiments} also do not consider tackling outliers. \citet{jewson2022general} consider the decision-theoretic implications of GBI, instead applied to model selection. We discuss related work on misspecification further in \Appref{sec_relatedworkappendix}. \Tableref{expdesigncomparison} compares our framework to BOED and the frameworks of \citet{overstall2023gibbsoptimaldesignexperiments} and \citet{tang2026representativeinformativedeamplifyingrequirements}.

\section{Preliminaries}
\label{sec_prelims}

\subsection{Notation}

Bolded capital Greek or Latin letters refer to random variables (rvs) (e.g., $\boldsymbol{\Theta}$ is the rv assigned to the parameter space). Realisations of rvs are bolded lowercase (e.g., a realisation of $\boldsymbol{\Theta}$ is $\boldsymbol{\theta}$). Sets are expressed as calligraphic capital letters (e.g., $\boldsymbol{\theta}$ can take values in the parameter space $\mathcal{T}$). $\mathbb{E}_{p(\boldsymbol{x})} \left[ f(\boldsymbol{x}) \right]$ is the expectation of the function $f(\boldsymbol{x})$ with respect to (wrt) the probability density function (pdf) $p$ of values $\boldsymbol{x}$. Unbolded capital Greek or Latin letters denote probability distributions.

\subsection{Bayesian Inference}

The experimenter selects designs $\boldsymbol{\xi}$ which produce output data $\boldsymbol{y}$. Stochasticity in the value of $\boldsymbol{y}$ is captured by the rv $\boldsymbol{Y}$. 
They do not know the distribution underlying $\boldsymbol{Y}$, but presume it in the form of a likelihood function $p(\boldsymbol{y} \mid \boldsymbol{\theta}, \boldsymbol{\xi})$. We refer to the presumptive likelihood function as the \textit{statistical model} of the real world. 

In Bayesian inference \citep{gelman2013bayesian}, the learner assigns a prior density $\pi(\boldsymbol{\theta})$ to their initial beliefs about the probabilities of values $\boldsymbol{\theta} \in \mathcal{T}$.
They subsequently update this prior to a posterior density $p(\boldsymbol{\theta} \mid \boldsymbol{\xi}, \boldsymbol{y})$ on the basis of new observations of data $\boldsymbol{y} \mid \boldsymbol{\xi}$. 

More specifically, the Bayesian posterior has density
\begin{equation*}
p(\boldsymbol{\theta} \mid \boldsymbol{\xi}, \boldsymbol{y}) = \frac{\pi(\boldsymbol{\theta})p(\boldsymbol{y} \mid \boldsymbol{\theta}, \boldsymbol{\xi})}{\int_{\mathcal{T}} \pi(\boldsymbol{\theta^{\prime}})p(\boldsymbol{y} \mid \boldsymbol{\theta^{\prime}}, \boldsymbol{\xi}) d\boldsymbol{\theta^{\prime}}},
\end{equation*}
where the denominator is referred to as the marginal likelihood $p(\boldsymbol{y} \mid \boldsymbol{\xi}) = \int_{\mathcal{T}} \pi(\boldsymbol{\theta})p(\boldsymbol{y} \mid \boldsymbol{\theta}, \boldsymbol{\xi}) d\boldsymbol{\theta}$. 

The parameter values $\boldsymbol{\theta}$ that best describe the data are unknown \textit{a priori}. In the \textit{well-specified} case, there exists a $\boldsymbol{\theta}^{*} \in \mathcal{T}$ such that $p(\boldsymbol{y} \mid \boldsymbol{\theta}^{*}, \boldsymbol{\xi})$ accurately represents the probability of observing $\boldsymbol{y}$ at design $\boldsymbol{\xi}$. With the data available, the experimenter aims to learn $\boldsymbol{\theta}^{*}$.

Model misspecification is the event that the assumed statistical model is dissimilar from the true DGP --- in which case there exists no $\boldsymbol{\theta}^{*} \in \mathcal{T}$ for which $p(\boldsymbol{y} \mid \boldsymbol{\theta}^{*}, \boldsymbol{\xi})$ corresponds to the true DGP. 
Still, under misspecification, certain values of $\boldsymbol{\theta}$ will be better than others in capturing certain aspects of the DGP. GBI (see \Secref{gibbsinf}) allows one to specify context-dependent criteria for useful parameter values.
We cover two forms of model misspecification: outliers being present in the data stream and incorrect noise distributional assumptions. The existence of outliers is common in many real-world environments, often due to poor data collection practices and faulty equipment. When constructing a model, one may also make poor assumptions about the noise inherent in the data stream. 

\subsection{Bayesian Optimal Experimental Design}
\label{sec:boed}

BOED \citep{rainforth2024modern, Huan_Jagalur_Marzouk_2024} is a principled framework for conducting experiments that maximise expected utility. The utility function is often the (Bayesian) expected information gain (BEIG; \citealp{lindley}), an information-theoretic measure of how much one can expect to learn about $\boldsymbol{\Theta}$ from an experiment conducted using a design $\boldsymbol{\xi} \in \mathcal{X}$. 

\begin{definition}[BEIG]\label{eigdefinition} The BEIG is the expected KL divergence from the posterior to the prior, or equivalently the mutual information between $\boldsymbol{\Theta}$ and $\boldsymbol{Y \mid \boldsymbol{\xi}}$, 
\begin{align*}
\mathrm{BEIG}(\boldsymbol{\xi}) & = \mathbb{E}_{p(\boldsymbol{y} \mid \boldsymbol{\xi})}\left[\mathrm{KL}\left(p(\boldsymbol{\theta} \mid \boldsymbol{\xi}, \boldsymbol{y}) \mid \mid \pi(\boldsymbol{\theta})\right)\right] \\
& = \mi{\boldsymbol{\Theta}}{\boldsymbol{Y} \mid \boldsymbol{\xi}}.
\end{align*}
\end{definition}

By choosing the design that maximises the BEIG, one efficiently utilises their experimental resources to reduce uncertainty about $\boldsymbol{\Theta}$. 
Design and inference occur sequentially for $T$ iterations. The posterior computed at iteration $t$ becomes the prior at $t + 1$, meaning that designs are selected based on previous observations. 

This is known to be an effective procedure for updating beliefs about $\boldsymbol{\theta}$ when the model is well-specified \citep{optimal_bayes, paninski_asymptotic_2005}. However, model misspecification can affect BOED's effectiveness in fitting and gathering new data \citep{rainforth2024modern}, resulting in wasted resources and potentially misleading inferences. 

\subsection{Generalised Bayesian Inference}
\label{gibbsinf}

Generalised Bayesian inference (GBI; \citealp{Bissiri_2016}) generalises Bayesian updating by replacing the likelihood $p(\boldsymbol{y} \mid \boldsymbol{\theta}, \boldsymbol{\xi})$ with a loss function $\ell_{\boldsymbol{\theta}}: \mathcal{X} \times \mathcal{Y} \rightarrow \mathbb{R}$. 
Beliefs are updated through the \textit{Gibbs posterior}, which has density
\begin{align} \label{gibbsposteriorrule}
\pi(\boldsymbol{\theta} \mid \boldsymbol{\xi}, \boldsymbol{y}) & = \frac{\pi(\boldsymbol{\theta}) \exp\left(-\omega\ell_{\boldsymbol{\theta}}(\boldsymbol{\xi}, \boldsymbol{y}) \right)}{\widetilde{\pi}\left( \boldsymbol{y} \mid \boldsymbol{\xi} \right)} \\
& = \frac{\pi(\boldsymbol{\theta}) \exp\left(-\omega\ell_{\boldsymbol{\theta}}(\boldsymbol{\xi}, \boldsymbol{y}) \right)}{\int_{\mathcal{T}} \pi(\boldsymbol{\theta^{\prime}})\exp\left(-\omega\ell_{\boldsymbol{\theta^{\prime}}}(\boldsymbol{\xi}, \boldsymbol{y}) \right) d\boldsymbol{\theta^{\prime}}}, \nonumber
\end{align}
where $\pi(\boldsymbol{\theta})$ is the pdf of the prior distribution as in Bayesian inference, $\omega > 0$ is a learning rate determining the influence of the data on the final posterior, and $\exp\left(-\omega\ell_{\boldsymbol{\theta}}(\boldsymbol{\xi}, \boldsymbol{y}) \right)$ is the \textit{exponentiated loss}.
We refer to $\widetilde{\pi}\left( \boldsymbol{y} \mid \boldsymbol{\xi} \right) \coloneqq \int_{\mathcal{T}} \pi(\boldsymbol{\theta})\exp\left(-\omega\ell_{\boldsymbol{\theta}}(\boldsymbol{\xi}, \boldsymbol{y}) \right) d\boldsymbol{\theta}$ as the \textit{marginal exponentiated loss}.

Although many learning rate selection methods exist in settings where data are already available \citep{Wu_2023}, there are very few catered to the experimental design setting in which one has no data to begin with. A well-chosen learning rate is important for GBI to be interpretable and for robustness to ensue. 
We discuss learning rate selection further in \Secref{sec_experiments}.

The data influence the Gibbs posterior in a way that depends on the loss function, determining which parameter values are given higher weight in the Gibbs posterior update. When a statistical model is available, one can recover classical Bayesian inference using $\omega = 1$ and $\ell_{\boldsymbol{\theta}}(\boldsymbol{\xi}, \boldsymbol{y}) = -\log p(\boldsymbol{y} \mid \boldsymbol{\theta}, \boldsymbol{\xi})$. This would be optimal when the model is well-specified \citep{optimal_bayes}. In instances of potential misspecification, specific types of loss functions have been explored that try to induce robustness to misspecified models. \citet{knoblauch} discuss a comprehensive list of loss functions.

While the loss function need not rely on a statistical model, in many cases, the model contains some valid information. For example, it may capture a general trend that is contaminated by outliers. In such cases, the experimenter likely wants their inferences to reflect the information contained in the model. \textit{Scoring rules} are loss functions that depend on the model \citep{dawid2014theory, Giummol__2018}.

\subsection{Scoring Rules}
\label{scorrules}

A scoring rule measures the agreement between a probability distribution and observed data. 
We study the following scoring rules, which have several appealing computational properties, such as additivity (which facilitates tractability and interpretability when data are collected sequentially) and invariance to coordinate changes in the dataset (so the order in which we select designs does not affect inference). They also admit practical rules for hyperparameter selection. Our framework can accommodate the use of any scoring rule. Alternatives are discussed in \Appref{lossfunchelp}.

\paragraph{Power Likelihoods} Power likelihoods \citep{grunwald, holmes2017assigning, mclatchie2025predictiveperformancepowerposteriors} use the negative log-likelihood loss $\ell_{\boldsymbol{\theta}}(\boldsymbol{\xi}, \boldsymbol{y}) = -\log p(\boldsymbol{y} \mid \boldsymbol{\theta}, \boldsymbol{\xi})$. Only the learning rate $\omega \neq 1$ is modified, which determines how much one relies on the statistical model for the Gibbs posterior update. Power likelihoods can be used to assess how changing the learning rate affects design and inference.

\paragraph{Score Matching} Score matching is an inferential framework in which one selects parameter values that minimise the Fisher divergence between the score functions (the gradient of the log-density wrt the data $\boldsymbol{y}$) of the statistical model and the true DGP \citep{hyvarinen, barp2019minimum, Matsubaradiscrete, altamirano23a}. Denoting by $\pmodel$ and $\pdata$ the densities of the statistical model and the true DGP, respectively, the score matching loss is
\begin{align*}
\ell_{\boldsymbol{\theta}}(\boldsymbol{\xi}, \boldsymbol{y}) & = \lVert r\left(\nabla_{\boldsymbol{y}}\log \pmodel(\boldsymbol{\xi}, \boldsymbol{y}) - \nabla_{\boldsymbol{y}}\log \pdata(\boldsymbol{\xi}, \boldsymbol{y}) \right)\rVert^{2}_{2},
\end{align*}
where $r : \mathcal{X} \times \mathcal{Y} \rightarrow \mathbb{R}_{\neq 0}$ is an optional weighting function \citep{altamirano23a, robustgpr}. The dependence on the unavailable quantity $\pdata$ can be avoided under certain regularity conditions \citep{liu2022estimating, altamirano23a, robustgpr}, making score matching computable in a manner described in \Appref{scorematchingappendix}. 

\paragraph{Weighted Score Matching}\label{sec:weighted-sm} The weighting function $r$ can induce robustness in the computed loss. \citet{robustgpr} propose the inverse multi-quadric (IMQ) kernel (``bump function'') to deal with outliers in data. The IMQ kernel function $r_{\mathrm{IMQ}} : \mathcal{X} \times \mathcal{Y} \rightarrow \mathbb{R}_{> 0}$ relies on a centring function $\gamma : \mathcal{X} \rightarrow \mathbb{R}$, shrinking function $c : \mathcal{X} \rightarrow \mathbb{R}_{> 0}$, and the learning rate $\omega > 0$: 
\begin{equation}
\label{rimq}
r_{\mathrm{IMQ}}(\boldsymbol{\xi}, \boldsymbol{y}) = \omega \left(1 + \frac{\left(\boldsymbol{y} - \gamma(\boldsymbol{\xi})\right)^{2}}{c(\boldsymbol{\xi})^{2}}\right)^{-\frac{1}{2}}.
\end{equation}
$\omega$ is the largest possible weight that can be assigned by the kernel, $\gamma$ controls the position of the bump ($\boldsymbol{y}$ values far from $\gamma$ are downweighted), and $c$ determines how quickly observations are downweighted.
The effectiveness of $r_{\mathrm{IMQ}}$ depends on the choice of $\gamma$ and $c$. 
\citet{laplante2025robustconjugatespatiotemporalgaussian} suggest using the posterior predictive mean and variance for $\gamma$ and $c^{2}$, respectively.

\section{Generalised Bayesian Experimental Design}
\label{sec_GBOED}
Bayesian decision theory \citep{berger2013statistical} is a framework for assessing the value resulting from a decision, in which one makes use of the statistical model to weight outcomes according to how probable they are.

In particular, one takes the decision that maximises the expectation of their \textit{utility} $u(\boldsymbol{\xi}, \boldsymbol{y})$ of the consequence $\boldsymbol{y}$ of decision $\boldsymbol{\xi}$, or \textit{expected utility}:
\begin{equation*}
U(\boldsymbol{\xi}) = \int_{\mathcal{Y}} u( \boldsymbol{\xi}, \boldsymbol{y}) p(\boldsymbol{y} \mid \boldsymbol{\xi}) d\boldsymbol{y}.
\end{equation*}
BOED is a Bayesian decision-theoretic method, where the utility of a particular outcome is the KL divergence it induces from the posterior to prior distribution.

\subsection{Pseudo-Expected Utility}
\label{gibbsdecthe}

The implicit commitment of the Bayesian decision theorist is that the outcomes which their model assesses to be most likely will be given the most weight in their decision.
In other words, they maintain a close connection between the \textit{presumptive likelihood} and \textit{degree of influence on downstream decisions} of an outcome.
However, the awareness that one's knowledge is incomplete implies a suspension of belief over what to expect --- and the requirement for a generalised notion of ``expected'' utility, namely \textit{pseudo-expected utility}. 
    
\begin{definition}[Pseudo-expected utility $\widetilde{U}$]\label{def:pseudo-exp}
        The pseudo-expected utility $\widetilde{U}(\boldsymbol{\xi})$ of a decision $\boldsymbol{\xi}$ wrt a weighting function $w : \mathcal{X} \times \mathcal{Y} \rightarrow \mathbb{R}_{\geq 0}$ is
        \begin{equation}
        \widetilde{U}(\boldsymbol{\xi}) \coloneqq \int_{\mathcal{Y}} u(\boldsymbol{\xi}, \boldsymbol{y}) w(\boldsymbol{\xi}, \boldsymbol{y}) d\boldsymbol{y}.
        \end{equation}
\end{definition}

The weighting function $w$ determines the relative influence of a decision's consequences on its subjective value.
The following assumption rescales these weights to account for systematic effects of the action $\boldsymbol{\xi}$:
\begin{assumption}\label{as:w-normalised}
    For some $\kappa \in \mathbb{R}_{\geq 0}$,
    $\forall \boldsymbol{\xi} \in \mathcal{X}, \int_{\mathcal{Y}} w(\boldsymbol{\xi}, \boldsymbol{y}) d\boldsymbol{y} = \kappa.$
\end{assumption}
\Cref{as:w-normalised} implies that, absent information from $\boldsymbol{y}$, all actions should be considered equally valuable.
If it holds for $\kappa = 1$, $w$ is a pdf, i.e., implies a rv.
If $w$ \textit{is} the pdf of the rv $\boldsymbol{Y \mid \xi}$, \Cref{def:pseudo-exp} is the expected utility.
However, in general, the implied rv is not characterised by the experimenter's beliefs about the distribution of outcomes. For example, below, we specify $w$ to convey the degree to which the experimenter trusts information about $\boldsymbol{\theta}$ conveyed by $\boldsymbol{y} \mid \boldsymbol{\xi}$; this need not correspond to how likely one believes they will encounter $\boldsymbol{y} \mid \boldsymbol{\xi}$.

\subsection{Gibbs Expected Information Gain}
\label{gibbseigsubsection}

In the context of BOED, the statistical model determines both an outcome's influence on decisions and on inference.
The key assumption of GBOED is a similar coherence in acquisition and inference: \textit{the loss function} determines both an outcome's influence on decisions and on inference.
When the loss function is a scoring rule, the loss is often interpreted as the degree to which an outcome is suspected to arise from a contamination process \citep{knoblauch}. Adherence to the key assumption of GBOED implies data that one suspects to be contaminated is afforded less weight in one's decisions: the processes producing contaminations may be less persistent, predictable, or stable.

The Gibbs EIG is a pseudo-expected utility measure. The Gibbs EIG and posterior exist by \Cref{assumption1}:

\begin{assumption} \label{assumption1} 
For all $\boldsymbol{\xi} \in \mathcal{X}$, the marginal exponentiated loss is an integrable function wrt $\boldsymbol{y}$ and satisfies $0 < \widetilde{\pi}(\boldsymbol{y} \mid \boldsymbol{\xi}) < \infty$.
\end{assumption}

\begin{definition}[Gibbs EIG] \label{definitiongibbseig}
With \Cref{assumption1} satisfied, the Gibbs EIG (GEIG) is the pseudo-expected KL divergence from the Gibbs posterior to the prior:
\begin{align*}
\mathrm{GEIG}(\boldsymbol{\xi}) & \coloneqq \int_{\mathcal{Y}} \underbrace{\mathrm{KL}(\pi(\boldsymbol{\theta} \mid \boldsymbol{\xi}, \boldsymbol{y}) \mid \mid \pi(\boldsymbol{\theta}))}_{u(\boldsymbol{\xi}, \boldsymbol{y})} \underbrace{\pi(\boldsymbol{y} \mid \boldsymbol{\xi})}_{w(\boldsymbol{\xi}, \boldsymbol{y})} d\boldsymbol{y},
\end{align*}
where $\pi(\boldsymbol{y} \mid \boldsymbol{\xi}) = \frac{\widetilde{\pi}(\boldsymbol{y} \mid \boldsymbol{\xi})}{\int_{\mathcal{Y}} \widetilde{\pi}(\boldsymbol{y}^{\prime} \mid \boldsymbol{\xi}) d\boldsymbol{y}^{\prime}}$.
\end{definition}
Notice that the GEIG satisfies \Cref{as:w-normalised} for $\kappa = 1$. \Appref{relaxingassump} discusses violation of \Cref{as:w-normalised}. 
\Cref{eigdefinition} shows that the BEIG is a measure of mutual information; \Appref{propertiesproof} explains that this is generally not the case for the GEIG.
\Cref{geigtendzero} explains that $\mathrm{GEIG} \to 0$ when $\omega \to 0$: when one is unwilling to learn from data, no information is gained.

Analogously to the BEIG, $u$ in \Cref{definitiongibbseig} measures the information gained in the GBI framework as the KL divergence from the Gibbs posterior to the prior. The weighting function $w$ affords the most influence to outcomes that are least suspected of being contaminated, enforcing the key assumption of GBOED.

Our framework, Generalised BOED (GBOED), refers to the combination of design selection via sequential maximisation of the GEIG and inference via GBI.

\paragraph{Computability of the Gibbs EIG}
The GEIG is difficult to evaluate in practice: (i) Gibbs posteriors are often unavailable in closed-form, requiring expensive approximation methods. 
(ii) The intractability of the normaliser in the weighting function renders the implied rv intractable to sample from. To tackle these issues, \Cref{theorem1} re-expresses the GEIG in a way that (i) does not explicitly include the Gibbs posterior, and (ii) is an expectation wrt a pdf that can be sampled from.

\begin{theorem}
    Given a \textnormal{proposal distribution} with pdf $q(\boldsymbol{y} \mid \boldsymbol{\theta}, \boldsymbol{\xi})$, the Gibbs EIG can be expressed as \label{theorem1}
\begin{equation}
\label{theoremeq}
\mathrm{GEIG}(\boldsymbol{\xi}) = \mathbb{E}_{}\left[\frac{g_{\boldsymbol{\theta}}(\boldsymbol{\xi}, \boldsymbol{y})\exp(-\omega\ell_{\boldsymbol{\theta}}(\boldsymbol{\xi}, \boldsymbol{y}))}{C(\boldsymbol{\xi})q(\boldsymbol{y} \mid \boldsymbol{\theta}, \boldsymbol{\xi})}\right],
\end{equation}
where the expectation is wrt $\pi(\boldsymbol{\theta})q(\boldsymbol{y} \mid \boldsymbol{\theta}, \boldsymbol{\xi})$, $g_{\boldsymbol{\theta}}(\boldsymbol{\xi}, \boldsymbol{y}) = -\omega\ell_{\boldsymbol{\theta}}(\boldsymbol{\xi}, \boldsymbol{y}) - \log \widetilde{\pi}(\boldsymbol{y} \mid \boldsymbol{\xi})$, and $C(\boldsymbol{\xi}) = \int_{\mathcal{Y}} \widetilde{\pi}(\boldsymbol{y} \mid \boldsymbol{\xi}) d\boldsymbol{y}$.
\end{theorem}

The proof is in \Appref{theorem1proof}. The expression in \Cref{theorem1} shows that the GEIG is computable via importance sampling, namely, using draws from a proposal distribution with pdf $q$. In particular, we can use the following nested Monte Carlo (NMC) estimator:

\begin{definition}[Gibbs EIG NMC estimator]\label{nmcdef} 
The Gibbs EIG can be estimated using NMC by 
\begin{align}
\label{nmcestimator}
U^{\mathrm{Gibbs}}_{\mathrm{NMC}}(\boldsymbol{\xi}) & \triangleq \sum_{i = 1}^{N}\left[{g}_{\boldsymbol{\theta}_{i}}(\boldsymbol{\xi}, \boldsymbol{y}_{i}) \cdot Z_{\boldsymbol{\theta}_{i}}(\boldsymbol{\xi}, \boldsymbol{y}_{i})\right], \nonumber \\
\text{where } g_{\boldsymbol{\theta}_i}(\boldsymbol{\xi}, \boldsymbol{y}_i) &= -\omega \ell_{\boldsymbol{\theta}_i}(\boldsymbol{\xi}, \boldsymbol{y}_i) - \log{\widetilde{\pi}_{\mathrm{NMC}}(\boldsymbol{y}_i \mid \boldsymbol{\xi})}, \nonumber \\
\widetilde{\pi}_{\mathrm{NMC}}(\boldsymbol{y}_{i} \mid \boldsymbol{\xi}) & = \frac{1}{M}\sum_{j = 1}^{M} \exp\left(-\omega\ell_{\boldsymbol{\theta}_{j}}(\boldsymbol{\xi}, \boldsymbol{y}_{i}) \right), \nonumber \\
Z_{\boldsymbol{\theta}_{i}}(\boldsymbol{\xi}, \boldsymbol{y}_{i}) & \propto \frac{\exp(-\omega\ell_{\boldsymbol{\theta}_{i}}(\boldsymbol{\xi}, \boldsymbol{y}_{i}))}{q(\boldsymbol{y}_{i} \mid \boldsymbol{\theta}_{i}, \boldsymbol{\xi})}, \nonumber
\end{align}
$\boldsymbol{\theta}_{i}, \boldsymbol{y}_{i} \sim \pi(\boldsymbol{\theta})q(\boldsymbol{y} \mid \boldsymbol{\theta}, \boldsymbol{\xi})$, and $\boldsymbol{\theta}_{j} \sim \pi(\boldsymbol{\theta})$.
\end{definition}

In our experiments, we set $q(\boldsymbol{y} \mid \boldsymbol{\theta}, \boldsymbol{\xi}) = p(\boldsymbol{y} \mid \boldsymbol{\theta}, \boldsymbol{\xi})$. The motivation for using the statistical model as our proposal is primarily computational: it is tractable and design-dependent. 
Our framework is open to alternative proposals. 
\Appref{impsampdiagnostics} compares the use of the model to heavy-tailed proposals, and includes diagnostics for the variance and effective sample size of the GEIG estimate as a function of loss function, learning rate, and chosen proposal distribution.\footnote{The expression in \Cref{theorem1} elucidates one key difference between our approach and that of \citet{overstall2023gibbsoptimaldesignexperiments}: the presence of the importance weight. We discuss this in detail in \Appref{importanceofimportanceweight}.}

Our NMC estimator avoids evaluating the normalising constant $C(\boldsymbol{\xi})$ attached to $\pi(\boldsymbol{y} \mid \boldsymbol{\xi})$ through self-normalised importance sampling \citep{Elvira_2021}. \Appref{nmcappendix} contains complete NMC detail. 

\subsection{Exponential Decay for IMQ Parameters} 
\label{ourimqmethod}

Recall from \Cref{sec:weighted-sm} that in the context of the weighted score matching loss function, \citet{laplante2025robustconjugatespatiotemporalgaussian} suggest using the posterior predictive mean and variance for the hyperparameters $\gamma$ and $c$, respectively.
When the posterior is informed by a large amount of data, these hyperparameters can work well.

In the experimental design setting, however, data are extremely scarce.
In combination with \citet{laplante2025robustconjugatespatiotemporalgaussian}'s method for choosing $c$, this can lead to observations being downweighted too quickly. The posterior predictive variance (PPV) can fall quickly as new data are gathered, but the centring function $\gamma$ will not necessarily have been learnt in that time. This leads to inlier data incorrectly being treated as outliers.
Rather than using the PPV for selecting $c^{2}$, we reduce $c$ according to an \textit{exponential decay} schedule during experimentation. This remains adaptive, and allows one to have control over the decreases in $c$ over time. See \Appref{expdecayguidance} for further detail and intuition about our method. 

\section{Experiments}
\label{sec_experiments}

We empirically assess performance across four experimental design problems of varying difficulty. Robustness is tested under two misspecified scenarios: \textit{Asymmetric Outliers} (outlier-contaminated data) and \textit{Misspecified Error Distribution} (incorrect noise model). The performance of each method here is evaluated using the maximum mean discrepancy (MMD; \citealp{mmd}) between values sampled from the predictive distribution and those from the true DGP.\footnote{\textit{CES} uses the root mean square error (RMSE) as the MMD cannot be computed; see \Appref{cesappendix}.} Full metric discussion in \Appref{metrics}. Results are plotted in \Figref{allplotsmmd}, with complementary tables in \Appref{tabresults}.

\Appref{pracguidelines} provides practical guidelines for using GBOED. \Appref{experimentdetails} presents experimental and implementational details, such as how we approximate Gibbs posteriors, the experimental design problems, the misspecified scenarios, deployment times, and performance metrics. \Appref{additionalexperappendix} offers a set of additional experiments, including visualisations of the GEIG.

\paragraph{Summary} The experimental results support our claims of robustness; a well-chosen learning rate and loss function can outperform BOED, and actively collecting data with the GEIG contributes to performance over na\"ively performing only GBI on a previously obtained dataset. As a general framework, certain learning rate and loss function choices are better than others depending on the setting, so we test various choices. Our exponential decay method for weighted score matching can achieve substantial gains over using the PPV to tune $c$. GBOED can perform better than BOED under misspecification, especially in scenarios of outlier-contamination. 

\paragraph{Method comparison} We compare GBOED to BOED, random design selection, and R-IDeA \citep{tang2026representativeinformativedeamplifyingrequirements}. Comparison to \citet{overstall2023gibbsoptimaldesignexperiments} is presented in \Appref{importanceofimportanceweight}. We also perform an ablation study to understand the effects of using the novel Gibbs EIG for design selection: we additionally compare GBOED to the use of GBI combined with the BEIG (BEIG + GBI) and random design selection (Random + GBI). The loss function used for the ablation study is our exponential decay method for weighted score matching, because it often performs best across all scenarios and experimental design problems.

\paragraph{Loss function} 
We evaluate the performance of GBOED using a power likelihood, the weighted score matching loss parametrisation proposed by \citet{laplante2025robustconjugatespatiotemporalgaussian}, and the parametrisation of the same using our novel exponential decay schedule for $c$ (\Cref{ourimqmethod}).
Note that our exponential decay schedule for $c$ requires tuning three hyperparameters; see \Appref{expdecayguidance}.

\paragraph{Learning rate} 
We test various learning rates, and present the results of these analyses in \Appref{altlrapp}.\footnote{Due to computational costs (see \Appref{locdeploymenttime}), we only tested one learning rate in the \textit{location finding} setting.} 
Here, we show the comparative performance of methods that use the same learning rate, according to each experimental setting (in ways described in the corresponding subsection) but fixed across all scenarios.

\begin{figure*}
    \centering
    \includegraphics[width=\linewidth]{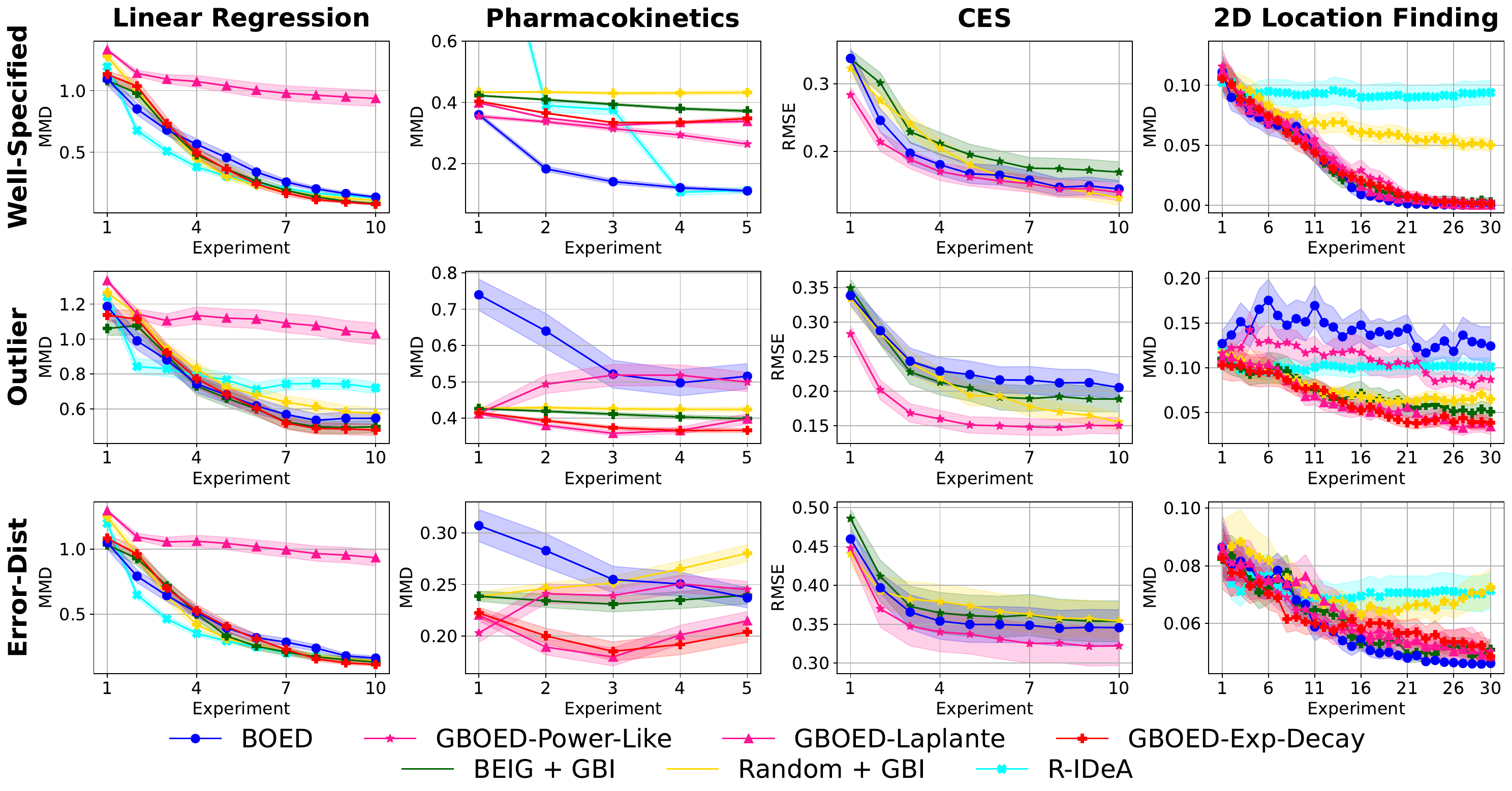}
    \caption{Performance across four experimental design problems (columns) and three model specification scenarios (rows); mean MMD/RMSE per experiment conducted and standard error shaded. Smaller $y$-axis values are better. Methods compared: BOED (\textcolor{blindblue}{blue}), GBOED with different loss functions (\textcolor{blindred}{red}/\textcolor{blindpink}{pink}), BEIG + GBI (\textcolor{blindgreen}{green}), Random + GBI (\textcolor{blindgold}{gold}), and R-IDeA (\textcolor{blindcyan}{cyan}). For GBI methods, markers distinguish the loss function used. GBOED with our proposed weighted score matching parametrisation (GBOED-Exp-Decay; \textcolor{blindred}{red}) often achieves low MMD where applicable.}
    \label{allplotsmmd}
\end{figure*}

\subsection{Linear Regression}

In the \textit{linear regression} setting, one selects covariates to estimate the coefficients of a linear regression model with Gaussian errors. To assess the robustness of our exponential decay method to poorly specified priors, the true parameter values are selected to be far from the mode of the prior. Each method has $T = 10$ experiments to learn what the coefficients are. 

With poorly specified priors, the relative performance of GBOED depends on its ability to learn from data.
Most methods perform similarly across the scenarios, querying designs at the extremes. \Appref{beignllregression} explains that performing GBI on a dataset acquired through BOED incurs worse predictive performance than using GBOED actively during experimentation. Performance falters without using the GEIG to collect data.

\paragraph{Learning rate} GBOED uses $\omega = 1$. Since the priors are poorly specified, we chose a high learning rate to enable fast learning. Performance can worsen as the learning rate is lowered due to slower learning.\footnote{Note that \Cref{allplotsmmd} for linear regression omits power likelihood results since $\omega = 1$ causes equivalence to BOED.}

\paragraph{Loss function}  
The performance of GBOED with weighted score matching relies on how $c$ in the IMQ kernel is chosen. Using our exponential decay method results in substantial improvements over \citet{laplante2025robustconjugatespatiotemporalgaussian}'s method, which is too robust and slower at learning. This results in worse performance than BOED, whereas exponential decay often outperforms BOED.

\subsection{Pharmacokinetics}

In the \textit{pharmacokinetics} setting, the learner uses a pharmacokinetic model \citep{RYAN201445} to study drug concentration over time, choosing $T = 5$ administration times for a small patient cohort to learn model parameters. The model is now nonlinear, resulting in the GEIG surfaces qualitatively differing much more from the BEIG surface (compare the figures in \Appref{visgibbseig}), which leads to widely different designs. 

GBOED falls behind BOED in the well-specified setting due to slow learning and a focus on robustness. When the model is misspecified, this robustness is advantageous: GBOED with score matching often exhibits the strongest performance in the two misspecified scenarios. 
GBOED achieves lower MMD than the BEIG and Random: selecting designs according to the GEIG contributes to performance in addition to the robust inference afforded by GBI. R-IDeA's misspecified scenario MMD is much higher than the others and omitted for pharmacokinetics; see \Appref{pharmaappendix} for results.

\paragraph{Learning rate} GBOED uses $\omega = 0.2$. In the Asymmetric Outliers scenario, the learner has a high probability of encountering an outlier (50\%; see \Cref{pharmacokineticsappendix}), and so we chose a small learning rate to facilitate robustness. GBOED with score matching takes advantage of the small learning rate in the misspecified scenarios. This is disadvantageous in the well-specified scenario, where robustness is not needed. Larger learning rates can improve performance in the well-specified setting, often sacrificing performance in misspecified settings. 

\paragraph{Loss function} Power likelihoods, which are not tailored to tackle outliers, perform better than score matching in the well-specified setting. This pattern reverses under misspecification, where the performance of GBOED with power likelihoods actually \textit{worsens} as one collects more data. This is likely because $\omega = 0.2$ is not suitable for the power likelihood in doing inference. Larger learning rates return improved performance.

\subsection{Constant Elasticity of Substitution}

\textit{Constant elasticity of substitution (CES)} is a 6-dimensional design problem for which the goal is to learn the preferences of an agent. Each design consists of a pair of baskets between which the agent is asked to choose. This is a higher-dimensional problem than the previous two, with $T = 10$ experiments to conduct.

Here, the likelihood function is non-differentiable, meaning that score matching cannot be used. GBOED can accommodate many loss functions, including those that do not require differentiability. Power likelihoods are one such choice. In the absence of prior misspecification, but the presence of relatively severe model misspecification, we anticipated that a moderate learning rate would facilitate robustness. 
\Figref{allplotsmmd} shows that GBOED with power likelihoods and $\omega = 0.6$ leads to low RMSE in all considered scenarios. Smaller learning rates can contribute to better predictive performance.

\subsection{Location Finding}
\label{locfindmainpaper}

As introduced in \Cref{introduction1}, the task in the \textit{location finding} setting is to infer the positions of two objects in a $d$-dimensional space from signal intensities observed at selected design coordinates. Here, the learner learns from $T = 30$ experiments, which presents a larger data setting.
\Figref{allplotsmmd} shows the results when $d = 2$.
GBOED uses learning rate $\omega = 0.4$. A relatively small learning rate facilitates robustness, while the relatively long experimental horizon facilitates learning.

When the model is well-specified, all methods perform similarly, except Random + GBI and R-IDeA. When faced with outliers, GBOED achieves lower MMD than BOED, with score matching loss functions performing strongest. GBOED with score matching loss functions also achieves lower MMD than BEIG + GBI, which suggests the GEIG selects more robust designs. When the noise model is incorrect, the deviation between the data-generating and specified noise models is relatively small (\Appref{locfindexpdetails}), so BOED effectively exploits the relatively reliable information in the data. These results hold in higher dimensions ($d = 4$); see \Appref{tabresults}. 

\paragraph{Loss function} Our exponential decay method and \citet{laplante2025robustconjugatespatiotemporalgaussian}'s method for weighted score matching perform similarly (here, exponential decay is parametrised in a way that both methods choose similar values of $c$). Whilst still better than BOED, power likelihoods fall behind in dealing with outliers, but are competitive in the other two scenarios. 

\section{Discussion}
\label{sec_discussion}

We introduced Generalised BOED (GBOED), a framework for performing sequential experimental design in the face of model misspecification. GBOED uses GBI for robust parameter inference, and the novel Gibbs EIG (GEIG) to select an informative sequence of designs according to Gibbs posteriors. 
In multiple misspecified settings, GBOED yields better predictive performance than BOED.
Our framework is not without its limitations. The quality of our importance sampling estimate of the GEIG depends on the choice of proposal distribution. A learning rate selection method tailored to experimental design would benefit GBOED. BOED and GBOED are expensive and challenged in high-dimensional problems. Amortisation can help in training design policies offline for rapid online deployment \citep{foster2021deep, blau2022optimizing}.

\subsubsection*{Acknowledgements}
The authors thank Ayush Bharti, François-Xavier Briol, Haripriya Harikumar, Jaeyoung Lee, Timothy Waite, and Zachris Björkman for helpful comments and feedback on early versions of this work. YZB was supported by a departmental studentship at The University of Manchester. SJS and SK were supported by the UKRI Turing AI World-Leading Researcher Fellowship [EP/W002973/1]. SK was also supported by EU funding (ERC ODD-ML 101201120) and the Research Council of Finland (Flagship programme: Finnish
Center for Artificial Intelligence FCAI). All experimental results were obtained through the Computational Shared Facility at The University of Manchester.

\bibliography{references}
\bibliographystyle{apalike}
\onecolumn
\appendix

\section{Gibbs Expected Information Gain Proofs and Properties}
\label{proofs}

In this appendix, we detail certain proofs for results either mentioned in the main paper, or that would be helpful to derive other proofs. Properties of the Gibbs EIG are also covered. 

\subsection{Relaxing Integrability to 1}
\label{relaxingassump}

\Cref{def:pseudo-exp} proposes the pseudo-expected utility. We discuss the outcome from relaxing the condition that $\int_{\mathcal{Y}} w(\boldsymbol{\xi}, \boldsymbol{y}) d\boldsymbol{y} = 1$. 

Let $\widetilde{w}$ be the weighting function where $\int_{\mathcal{Y}} \widetilde{w}(\boldsymbol{\xi}, \boldsymbol{y}) d\boldsymbol{y} \neq 1$. Then $w(\boldsymbol{\xi}, \boldsymbol{y}) = \frac{\widetilde{w}(\boldsymbol{\xi}, \boldsymbol{y})}{\int_{\mathcal{Y}} \widetilde{w}(\boldsymbol{\xi}, \boldsymbol{y}^{\prime}) d\boldsymbol{y}^{\prime}}$ is its normalised variant. Using this, by substitution we have pseudo-expected utility
\begin{align*}
\widetilde{A}(\boldsymbol{\xi}) = \int_{\mathcal{Y}} u(\boldsymbol{\xi}, \boldsymbol{y}) \widetilde{w}(\boldsymbol{\xi}, \boldsymbol{y}) d\boldsymbol{y}.
\end{align*}
This is not a standard expectation since $\widetilde{w}(\boldsymbol{\xi}, \boldsymbol{y})$ is not normalised. Now substituting instead the normalised variant,
\begin{align*}
\widetilde{B}(\boldsymbol{\xi}) = \int_{\mathcal{Y}} u(\boldsymbol{\xi}, \boldsymbol{y}) \frac{\widetilde{w}(\boldsymbol{\xi}, \boldsymbol{y})}{\int_{\mathcal{Y}} \widetilde{w}(\boldsymbol{\xi}, \boldsymbol{y}^{\prime}) d\boldsymbol{y}^{\prime}} d\boldsymbol{y}.
\end{align*}
A standard expectation ensues. The difference between the two is in the normalising term $\int_{\mathcal{Y}} \widetilde{w}(\boldsymbol{\xi}, \boldsymbol{y}) d\boldsymbol{y}$:
\begin{align*}
\widetilde{B}(\boldsymbol{\xi}) = \frac{\widetilde{A}(\boldsymbol{\xi})}{\int_{\mathcal{Y}} \widetilde{w}(\boldsymbol{\xi}, \boldsymbol{y}^{\prime}) d\boldsymbol{y}^{\prime}}.
\end{align*}
This means that, while proportional, the utility computed through the unnormalised weighting function $\widetilde{w}$ is in general not equal to its normalised variant. As $\int_{\mathcal{Y}} \widetilde{w}(\boldsymbol{\xi}, \boldsymbol{y}) d\boldsymbol{y}$ depends on $\boldsymbol{\xi}$, the ranking of optimal designs can change when computed across the entire design space.

Moving to the Gibbs EIG, following the above, we assume there is no need to normalise $\widetilde{\pi}(\boldsymbol{y} \mid \boldsymbol{\xi})$, and so the Gibbs EIG becomes
\begin{align*}
\mathrm{GEIG_{alt}}(\boldsymbol{\xi}) = \int_{\mathcal{Y}}\left(\int_{\mathcal{T}}\log \left(\frac{\pi(\boldsymbol{\theta} \mid \boldsymbol{\xi}, \boldsymbol{y})}{\pi(\boldsymbol{\theta})}\right)\pi(\boldsymbol{\theta} \mid \boldsymbol{\xi}, \boldsymbol{y})d\boldsymbol{\theta}\right)\widetilde{\pi}(\boldsymbol{y} \mid \boldsymbol{\xi}) d\boldsymbol{y}.
\end{align*}
Notice that this is not a standard expectation since $\widetilde{\pi}(\boldsymbol{y} \mid \boldsymbol{\xi})$ does not integrate to 1. 
The relationship to the Gibbs EIG in the main paper is that $$\mathrm{GEIG}(\boldsymbol{\xi}) = \frac{\mathrm{GEIG_{alt}}(\boldsymbol{\xi})}{C(\boldsymbol{\xi})},$$ where $C(\boldsymbol{\xi}) = \int_{\mathcal{Y}} \widetilde{\pi}(\boldsymbol{y} \mid \boldsymbol{\xi}) d\boldsymbol{y}$. In other words, $\mathrm{GEIG_{alt}}$ is an unnormalised measure of the GEIG.

\subsection{Computing the Gibbs EIG via Importance Sampling}
\label{corollproof}
\label{theorem1proof}

We begin with \Cref{lemma1}, which offers an important interpretation of the Gibbs EIG in terms of loss functions, rather than directly using posteriors. This is equivalent to the BEIG being expressed in terms of likelihood functions. The idea is that this is often a cheaper and more convenient way to compute the Gibbs EIG, knowing that posteriors are usually expensive to approximate (without being in closed-form). 

\begin{lemma} \label{lemma1}
The Gibbs EIG is the difference between the negative loss $-\omega\ell_{\boldsymbol{\theta}}(\boldsymbol{\xi}, \boldsymbol{y})$ and the log marginal exponentiated loss $\log \widetilde{\pi}(\boldsymbol{y} \mid \boldsymbol{\xi})$ in expectation with respect to $\frac{\pi(\boldsymbol{\theta})\exp\left(-\omega\ell_{\boldsymbol{\theta}}(\boldsymbol{\xi}, \boldsymbol{y}) \right)}{C(\boldsymbol{\xi})}$
\begin{equation}
\mathrm{GEIG}(\boldsymbol{\xi}) = \frac{1}{C(\boldsymbol{\xi})} \int_{\mathcal{Y}}\int_{\mathcal{T}} \left(-\omega\ell_{\boldsymbol{\theta}}(\boldsymbol{\xi}, \boldsymbol{y}) - \log \widetilde{\pi}(\boldsymbol{y} \mid \boldsymbol{\xi})\right) \pi(\boldsymbol{\theta})\exp\left(-\omega\ell_{\boldsymbol{\theta}}(\boldsymbol{\xi}, \boldsymbol{y}) \right) d\boldsymbol{\theta} d\boldsymbol{y},
\end{equation}
where $C(\boldsymbol{\xi}) = \int_{\mathcal{Y}} \widetilde{\pi}(\boldsymbol{y} \mid \boldsymbol{\xi}) d\boldsymbol{y}$.
\end{lemma}
\begin{proof}
We can write the Gibbs EIG as a log-ratio between the exponentiated and marginal exponentiated losses knowing that
\begin{equation}
\pi(\boldsymbol{\theta})\exp\left(-\omega\ell_{\boldsymbol{\theta}}(\boldsymbol{\xi}, \boldsymbol{y}) \right) = \pi(\boldsymbol{\theta} \mid \boldsymbol{\xi}, \boldsymbol{y})\widetilde{\pi}(\boldsymbol{y} \mid \boldsymbol{\xi}).
\end{equation} 
Starting from \Cref{definitiongibbseig}, we have
\begin{align*}
\mathrm{GEIG}(\boldsymbol{\xi}) & = {\mathbb{E}}_{\frac{\widetilde{\pi}(\boldsymbol{y} \mid \boldsymbol{\xi})}{C(\boldsymbol{\xi})}}\left[\mathrm{KL}(\pi(\boldsymbol{\theta} \mid \boldsymbol{\xi}, \boldsymbol{y}) \mid \mid \pi(\boldsymbol{\theta}))\right] \\
& = \int_{\mathcal{Y}}\left(\int_{\mathcal{T}}\log \left(\frac{\pi(\boldsymbol{\theta} \mid \boldsymbol{\xi}, \boldsymbol{y})}{\pi(\boldsymbol{\theta})}\right)\pi(\boldsymbol{\theta} \mid \boldsymbol{\xi}, \boldsymbol{y})d\boldsymbol{\theta}\right)\frac{\widetilde{\pi}(\boldsymbol{y} \mid \boldsymbol{\xi})}{C(\boldsymbol{\xi})} d\boldsymbol{y} \\
& = \int_{\mathcal{Y}}\int_{\mathcal{T}}\log \left(\frac{\pi(\boldsymbol{\theta} \mid \boldsymbol{\xi}, \boldsymbol{y})}{\pi(\boldsymbol{\theta})}\right)\pi(\boldsymbol{\theta} \mid \boldsymbol{\xi}, \boldsymbol{y})\frac{\widetilde{\pi}(\boldsymbol{y} \mid \boldsymbol{\xi})}{C(\boldsymbol{\xi})} d\boldsymbol{\theta} d\boldsymbol{y} \\
& = \int_{\mathcal{Y}}\int_{\mathcal{T}}\log \left(\frac{\pi(\boldsymbol{\theta})\exp\left(-\omega\ell_{\boldsymbol{\theta}}(\boldsymbol{\xi}, \boldsymbol{y}) \right)}{\pi(\boldsymbol{\theta})\widetilde{\pi}(\boldsymbol{y} \mid \boldsymbol{\xi})}\right)\pi(\boldsymbol{\theta} \mid \boldsymbol{\xi}, \boldsymbol{y})\frac{\widetilde{\pi}(\boldsymbol{y} \mid \boldsymbol{\xi})}{C(\boldsymbol{\xi})} d\boldsymbol{\theta} d\boldsymbol{y} \\
& = \int_{\mathcal{Y}}\int_{\mathcal{T}} \left(-\omega\ell_{\boldsymbol{\theta}}(\boldsymbol{\xi}, \boldsymbol{y}) - \log \widetilde{\pi}(\boldsymbol{y} \mid \boldsymbol{\xi})\right) \pi(\boldsymbol{\theta} \mid \boldsymbol{\xi}, \boldsymbol{y})\frac{\widetilde{\pi}(\boldsymbol{y} \mid \boldsymbol{\xi})}{C(\boldsymbol{\xi})} d\boldsymbol{\theta} d\boldsymbol{y} \\
& = \int_{\mathcal{Y}}\int_{\mathcal{T}} \left(-\omega\ell_{\boldsymbol{\theta}}(\boldsymbol{\xi}, \boldsymbol{y}) - \log \widetilde{\pi}(\boldsymbol{y} \mid \boldsymbol{\xi})\right) \frac{\pi(\boldsymbol{\theta})\exp\left(-\omega\ell_{\boldsymbol{\theta}}(\boldsymbol{\xi}, \boldsymbol{y}) \right)}{{C(\boldsymbol{\xi})}} d\boldsymbol{\theta} d\boldsymbol{y} \\
& = \frac{1}{{C(\boldsymbol{\xi})}} \int_{\mathcal{Y}}\int_{\mathcal{T}} \left(-\omega\ell_{\boldsymbol{\theta}}(\boldsymbol{\xi}, \boldsymbol{y}) - \log \widetilde{\pi}(\boldsymbol{y} \mid \boldsymbol{\xi})\right) \pi(\boldsymbol{\theta})\exp\left(-\omega\ell_{\boldsymbol{\theta}}(\boldsymbol{\xi}, \boldsymbol{y}) \right) d\boldsymbol{\theta} d\boldsymbol{y}.
\end{align*}
\end{proof}

\textbf{Theorem \ref{theorem1}.} \textit{Given a proposal distribution with pdf $q(\boldsymbol{y} \mid \boldsymbol{\theta}, \boldsymbol{\xi})$ with $q(\boldsymbol{y} \mid \boldsymbol{\theta}, \boldsymbol{\xi}) > 0$ almost everywhere on the support of $\exp(-\omega\ell_{\boldsymbol{\theta}}(\boldsymbol{\xi}, \boldsymbol{y}))$, the Gibbs EIG can be expressed as}
\begin{equation*}
\mathrm{GEIG}(\boldsymbol{\xi}) = \mathbb{E}_{\pi(\boldsymbol{\theta})q(\boldsymbol{y} \mid \boldsymbol{\theta}, \boldsymbol{\xi})}\left[\left(-\omega\ell_{\boldsymbol{\theta}}(\boldsymbol{\xi}, \boldsymbol{y}) - \log \widetilde{\pi}(\boldsymbol{{y}} \mid \boldsymbol{\xi})\right) \cdot \left(\frac{\exp(-\omega\ell_{\boldsymbol{\theta}}(\boldsymbol{\xi}, \boldsymbol{y}))}{{C(\boldsymbol{\xi})}q(\boldsymbol{y} \mid \boldsymbol{\theta}, \boldsymbol{\xi})}\right)\right],
\end{equation*}
\textit{where $C(\boldsymbol{\xi}) = \int_{\mathcal{Y}} \widetilde{\pi}(\boldsymbol{y} \mid \boldsymbol{\xi}) d\boldsymbol{y}$.}
\begin{proof}
Starting from \Cref{lemma1}, we have
\begin{align*}
\mathrm{GEIG}(\boldsymbol{\xi}) & = \frac{1}{{C(\boldsymbol{\xi})}} \int_{\mathcal{Y}}\int_{\mathcal{T}} \left(-\omega\ell_{\boldsymbol{\theta}}(\boldsymbol{\xi}, \boldsymbol{y}) - \log \widetilde{\pi}(\boldsymbol{y} \mid \boldsymbol{\xi})\right) \pi(\boldsymbol{\theta})\exp\left(-\omega\ell_{\boldsymbol{\theta}}(\boldsymbol{\xi}, \boldsymbol{y}) \right) d\boldsymbol{\theta} d\boldsymbol{y} \\
& = \frac{1}{{C(\boldsymbol{\xi})}} \int_{\mathcal{T}}\int_{\mathcal{Y}}\left(-\omega\ell_{\boldsymbol{\theta}}(\boldsymbol{\xi}, \boldsymbol{y}) - \log \widetilde{\pi}(\boldsymbol{{y}} \mid \boldsymbol{\xi})\right) \left(\frac{\pi(\boldsymbol{\theta})\exp\left(-\omega\ell_{\boldsymbol{\theta}}(\boldsymbol{\xi}, \boldsymbol{y}) \right)}{\pi(\boldsymbol{\theta})q(\boldsymbol{y} \mid \boldsymbol{\theta}, \boldsymbol{\xi})}\right) \pi(\boldsymbol{\theta})q(\boldsymbol{y} \mid \boldsymbol{\theta}, \boldsymbol{\xi}) d\boldsymbol{{y}} d\boldsymbol{\theta} \\
& = \mathbb{E}_{\pi(\boldsymbol{\theta})q(\boldsymbol{y} \mid \boldsymbol{\theta}, \boldsymbol{\xi})}\left[\left(-\omega\ell_{\boldsymbol{\theta}}(\boldsymbol{\xi}, \boldsymbol{y}) - \log \widetilde{\pi}(\boldsymbol{{y}} \mid \boldsymbol{\xi})\right) \cdot \left(\frac{\exp(-\omega\ell_{\boldsymbol{\theta}}(\boldsymbol{\xi}, \boldsymbol{y}))}{{C(\boldsymbol{\xi})}q(\boldsymbol{y} \mid \boldsymbol{\theta}, \boldsymbol{\xi})}\right)\right],
\end{align*}
which is now an expectation over a proposal distribution $q$ and exactly \Cref{theoremeq}.
\end{proof}

\subsection{Properties of the Gibbs EIG}
\label{propertiesproof}

We show that the Gibbs EIG is related to the mutual information and that it is non-negative.

\paragraph{Relationship to Mutual Information} The GEIG satisfies \Cref{as:w-normalised} for $\kappa = 1$.
As we discuss in \Cref{gibbsdecthe}, this means that the weighting function $\pi(\boldsymbol{y} \mid \boldsymbol{\xi})$ can be viewed as a pdf.
However, $w$ is \textit{not} the pdf of the rv $\boldsymbol{Y} \mid \boldsymbol{\xi}$ and the implied rv does not have an immediate, meaningful interpretation. 

As shown in \Cref{eigdefinition}, the BEIG admits an alternative interpretation as the mutual information between $\boldsymbol{\Theta}$ and $\boldsymbol{Y} \mid \boldsymbol{\xi}$.
One may wonder whether the GEIG can similarly be construed as the mutual information between the rvs implied by the weighting function.
\Cref{proposition1} shows that this is generally \textit{not} the case.
In particular, the GEIG incorporates an additional discrepancy term that captures the divergence from the parameter distribution marginalised over the weighting function and that used for inference (the prior). 

We assess in \Cref{proposition1} the relationship between the Gibbs EIG in \Cref{definitiongibbseig} and the mutual information between $\boldsymbol{\widetilde{\Theta}}$ (for the parameter space) and $\boldsymbol{\widetilde{Y}} \mid \boldsymbol{\xi}$ (for the observation space), which we denote to be the rvs implied by the weighting function. 

\begin{proposition} \label{proposition1}
The Gibbs EIG can be expressed in terms of the mutual information between $\boldsymbol{\widetilde{\Theta}}$ and $\boldsymbol{\widetilde{Y}} \mid \boldsymbol{\xi}$
\begin{equation}
\mathrm{GEIG}(\boldsymbol{\xi}) = \pseudomi{\boldsymbol{\widetilde{\Theta}}}{\boldsymbol{\widetilde{Y}} \mid \boldsymbol{\xi}} + \mathrm{KL}\left(\delta(\boldsymbol{\theta} \mid \boldsymbol{\xi}) \mid \mid \pi(\boldsymbol{\theta})\right),
\end{equation}
where $\delta(\boldsymbol{\theta} \mid \boldsymbol{\xi}) = \frac{\pi(\boldsymbol{\theta})}{C(\boldsymbol{\xi})}\int_{\mathcal{Y}} \exp\left(-\omega\ell_{\boldsymbol{\theta}}(\boldsymbol{\xi}, \boldsymbol{y}) \right) d\boldsymbol{y}$, and $C(\boldsymbol{\xi}) = \int_{\mathcal{Y}} \widetilde{\pi}(\boldsymbol{y} \mid \boldsymbol{\xi}) d\boldsymbol{y}$.
\end{proposition}
\begin{proof}
Begin by recognising that
\begin{align*}
\pi(\boldsymbol{\theta})\exp\left(-\omega\ell_{\boldsymbol{\theta}}(\boldsymbol{\xi}, \boldsymbol{y}) \right) & = \pi(\boldsymbol{\theta} \mid \boldsymbol{\xi}, \boldsymbol{y})\widetilde{\pi}(\boldsymbol{y} \mid \boldsymbol{\xi}) \\
\frac{\pi(\boldsymbol{\theta})\exp\left(-\omega\ell_{\boldsymbol{\theta}}(\boldsymbol{\xi}, \boldsymbol{y}) \right)}{C(\boldsymbol{\xi})} & = \frac{\pi(\boldsymbol{\theta} \mid \boldsymbol{\xi}, \boldsymbol{y})\widetilde{\pi}(\boldsymbol{y} \mid \boldsymbol{\xi})}{C(\boldsymbol{\xi})},
\end{align*}
where $C(\boldsymbol{\xi}) = \int_{\mathcal{Y}} \widetilde{\pi}(\boldsymbol{y} \mid \boldsymbol{\xi}) d\boldsymbol{y}$.

To measure the mutual information, one requires defining the joint density $$\pi(\boldsymbol{\theta}, \boldsymbol{y} \mid \boldsymbol{\xi}) = \frac{\pi(\boldsymbol{\theta})\exp\left(-\omega\ell_{\boldsymbol{\theta}}(\boldsymbol{\xi}, \boldsymbol{y}) \right)}{C(\boldsymbol{\xi})}.$$ We also require the two marginal distributions that follow from this joint density, which for $\boldsymbol{y}$ is $$\pi(\boldsymbol{y} \mid \boldsymbol{\xi}) = \int_{\mathcal{T}}\pi(\boldsymbol{\theta}, \boldsymbol{y} \mid \boldsymbol{\xi}) d\boldsymbol{\theta} = \frac{\widetilde{\pi}(\boldsymbol{y} \mid \boldsymbol{\xi})}{C(\boldsymbol{\xi})},$$ and for $\boldsymbol{\theta}$ is $$\delta(\boldsymbol{\theta} \mid \boldsymbol{\xi}) = \int_{\mathcal{Y}}\pi(\boldsymbol{\theta}, \boldsymbol{y} \mid \boldsymbol{\xi}) d\boldsymbol{y} = \frac{\pi(\boldsymbol{\theta})}{C(\boldsymbol{\xi})}\int_{\mathcal{Y}} \exp\left(-\omega\ell_{\boldsymbol{\theta}}(\boldsymbol{\xi}, \boldsymbol{y}) \right) d\boldsymbol{y}.$$
Simple substitution into the mutual information definition gives
\begin{align*}
\pseudomi{\boldsymbol{\widetilde{\Theta}}}{\boldsymbol{\widetilde{Y}} \mid \boldsymbol{\xi}} & = \mathbb{E}_{\pi(\boldsymbol{\theta}, \boldsymbol{y} \mid \boldsymbol{\xi})}\left[\log \left(\frac{\pi(\boldsymbol{\theta}, \boldsymbol{y} \mid \boldsymbol{\xi})}{\delta(\boldsymbol{\theta} \mid \boldsymbol{\xi})\pi(\boldsymbol{y} \mid \boldsymbol{\xi})}\right)\right] \\
& = \mathbb{E}_{\pi(\boldsymbol{\theta}, \boldsymbol{y} \mid \boldsymbol{\xi})}\left[\log \left(\frac{C(\boldsymbol{\xi})^{2}\pi(\boldsymbol{\theta})\exp\left(-\omega\ell_{\boldsymbol{\theta}}(\boldsymbol{\xi}, \boldsymbol{y}) \right)}{C(\boldsymbol{\xi})\pi(\boldsymbol{\theta})G(\boldsymbol{\theta}, \boldsymbol{\xi})\widetilde{\pi}(\boldsymbol{y} \mid \boldsymbol{\xi})}\right)\right] \\
& = \mathbb{E}_{\pi(\boldsymbol{\theta}, \boldsymbol{y} \mid \boldsymbol{\xi})}\left[\log \left(\frac{C(\boldsymbol{\xi})\exp\left(-\omega\ell_{\boldsymbol{\theta}}(\boldsymbol{\xi}, \boldsymbol{y}) \right)}{G(\boldsymbol{\theta}, \boldsymbol{\xi})\widetilde{\pi}(\boldsymbol{y} \mid \boldsymbol{\xi})}\right)\right],
\end{align*}
where $G(\boldsymbol{\theta}, \boldsymbol{\xi}) = \int_{\mathcal{Y}} \exp\left(-\omega\ell_{\boldsymbol{\theta}}(\boldsymbol{\xi}, \boldsymbol{y}) \right) d\boldsymbol{y}$.

This is the Gibbs EIG when $\delta(\boldsymbol{\theta} \mid \boldsymbol{\xi}) = \pi(\boldsymbol{\theta})$. However, in general, this will not be the case, and so the Gibbs EIG will not exactly be equal to the mutual information. That is not to say that they are not related; the Gibbs EIG can still be written in terms of the mutual information. 
\begin{align*}
\mathrm{GEIG}(\boldsymbol{\xi}) & = \pseudomi{\boldsymbol{\widetilde{\Theta}}}{\boldsymbol{\widetilde{Y}} \mid \boldsymbol{\xi}} + \mathrm{KL}\left(\delta(\boldsymbol{\theta} \mid \boldsymbol{\xi}) \mid \mid \pi(\boldsymbol{\theta})\right) \\
& = \mathbb{E}_{\pi(\boldsymbol{\theta}, \boldsymbol{y} \mid \boldsymbol{\xi})}\left[\log \left(\frac{C(\boldsymbol{\xi})\exp\left(-\omega\ell_{\boldsymbol{\theta}}(\boldsymbol{\xi}, \boldsymbol{y}) \right)}{G(\boldsymbol{\theta}, \boldsymbol{\xi})\widetilde{\pi}(\boldsymbol{y} \mid \boldsymbol{\xi})}\right)\right] + \mathbb{E}_{\delta(\boldsymbol{\theta} \mid \boldsymbol{\xi})}\left[\log \left(\frac{\delta(\boldsymbol{\theta} \mid \boldsymbol{\xi})}{\pi(\boldsymbol{\theta})}\right)\right] \\
& = \mathbb{E}_{\pi(\boldsymbol{\theta}, \boldsymbol{y} \mid \boldsymbol{\xi})}\left[\log \left(\frac{C(\boldsymbol{\xi})\exp\left(-\omega\ell_{\boldsymbol{\theta}}(\boldsymbol{\xi}, \boldsymbol{y}) \right)}{G(\boldsymbol{\theta}, \boldsymbol{\xi})\widetilde{\pi}(\boldsymbol{y} \mid \boldsymbol{\xi})}\right)\right] + \mathbb{E}_{\delta(\boldsymbol{\theta} \mid \boldsymbol{\xi})}\left[\log \left(\frac{G(\boldsymbol{\theta}, \boldsymbol{\xi})}{C(\boldsymbol{\xi})}\right)\right] \\
& = \mathbb{E}_{\pi(\boldsymbol{\theta}, \boldsymbol{y} \mid \boldsymbol{\xi})}\left[\log \left(\frac{C(\boldsymbol{\xi})\exp\left(-\omega\ell_{\boldsymbol{\theta}}(\boldsymbol{\xi}, \boldsymbol{y}) \right)}{G(\boldsymbol{\theta}, \boldsymbol{\xi})\widetilde{\pi}(\boldsymbol{y} \mid \boldsymbol{\xi})}\right)\right] + \mathbb{E}_{\pi(\boldsymbol{\theta}, \boldsymbol{y} \mid \boldsymbol{\xi})}\left[\log \left(\frac{G(\boldsymbol{\theta}, \boldsymbol{\xi})}{C(\boldsymbol{\xi})}\right)\right] \\
& = \mathbb{E}_{\pi(\boldsymbol{\theta}, \boldsymbol{y} \mid \boldsymbol{\xi})}\left[\log \left(\frac{C(\boldsymbol{\xi})\exp\left(-\omega\ell_{\boldsymbol{\theta}}(\boldsymbol{\xi}, \boldsymbol{y}) \right)}{G(\boldsymbol{\theta}, \boldsymbol{\xi})\widetilde{\pi}(\boldsymbol{y} \mid \boldsymbol{\xi})}\right) + \log \left(\frac{G(\boldsymbol{\theta}, \boldsymbol{\xi})}{C(\boldsymbol{\xi})}\right)\right] \\
& = \mathbb{E}_{\pi(\boldsymbol{\theta}, \boldsymbol{y} \mid \boldsymbol{\xi})}\left[\log \left(\frac{\exp\left(-\omega\ell_{\boldsymbol{\theta}}(\boldsymbol{\xi}, \boldsymbol{y}) \right)}{\widetilde{\pi}(\boldsymbol{y} \mid \boldsymbol{\xi})}\right)\right] \\
& = {\mathbb{E}}_{\pi(\boldsymbol{\theta}, \boldsymbol{y} \mid \boldsymbol{\xi})}\left[\log \left( \frac{\pi(\boldsymbol{\theta}, \boldsymbol{y} \mid \boldsymbol{\xi})}{\pi(\boldsymbol{\theta})\pi(\boldsymbol{y} \mid \boldsymbol{\xi})} \right)\right].
\end{align*}
In other words, the Gibbs EIG is the mutual information between $\boldsymbol{\widetilde{\Theta}}$ and $\boldsymbol{\widetilde{Y}} \mid \boldsymbol{\xi}$ after discarding the additional discrepancy term $\mathrm{KL}\left(\delta(\boldsymbol{\theta} \mid \boldsymbol{\xi}) \mid \mid \pi(\boldsymbol{\theta})\right)$.
\end{proof}

Notice that for the BEIG, for which $\boldsymbol{\widetilde{\Theta}}$ and $\boldsymbol{\widetilde{Y}} \mid \boldsymbol{\xi}$ are equivalent to $\boldsymbol{\Theta}$ and $\boldsymbol{Y} \mid \boldsymbol{\xi}$, respectively, $\delta(\boldsymbol{\theta} \mid \boldsymbol{\xi}) = \pi(\boldsymbol{\theta})$ and the discrepancy term evaluates to 0.
In this special case, we recover the equivalence between the BEIG and the standard mutual information.
In general, the discrepancy term becomes smaller as $G(\boldsymbol{\theta}, \boldsymbol{\xi}) \to C(\boldsymbol{\xi})$, and at $G(\boldsymbol{\theta}, \boldsymbol{\xi}) = C(\boldsymbol{\xi})$ the discrepancy term vanishes and we have exactly the mutual information.

\paragraph{Non-Negativity} The Gibbs EIG is a KL divergence from the Gibbs posterior to the prior, which is non-negative. It follows by \Cref{assumption1} that as the KL divergence is non-negative, so too is the Gibbs EIG. 

\section{Related Work Extended}
\label{sec_relatedworkappendix}

Alternative acquisition functions can select designs that enhance robustness to model misspecification \citep{forstermisspec, tang2026representativeinformativedeamplifyingrequirements}. GBOED not only enables robustness in design through the Gibbs EIG, but also in inference through GBI.

\citet{abdulsamad2026maximinrobustbayesianexperimental} propose a framework for robust experimental design that leverages ideas from distributional robustness and the PAC-Bayes literature. This is analogous to using a power likelihood for design and inference. However, since their framework does not rely on GBI, it cannot be used with general loss functions that may target specific aims (such as dealing with outliers through weighted score matching). Their framework also does not follow the decision-theoretic framing taken in this paper, instead formulating robust experimental design as a worst-case optimisation problem. \citet{harikumar2026robustbayesiandecisionmaking} also look into worst-case optimisation, instead tailored to downstream decision-making.

The idea of using GBI to perform experimental design was, to our knowledge, first proposed by \citet{overstall2023gibbsoptimaldesignexperiments}. However, their framework requires that an alternative model, coined the \textit{designer distribution}, is made available. This distribution is assumed to be flexible and close to the true DGP, and allows one to compute the expected utility using draws from this distribution. In many cases, we are often not able to choose a distribution that we know is certainly close to the true DGP. Our approach instead uses the (possibly) misspecified model to induce robustness into the experimental design procedure. Here, we are open to the possibility that, while misspecified, the statistical model has information relevant to an experimenter, and is our best understanding of how reality operates. This enables the use of loss functions that can directly take the statistical model into account when conducting GBI, in particular, through scoring rules \citep{dawid2014theory, Giummol__2018}. This is often the approach taken when tackling outliers. More generally, we compute the expected utility in a generalised fashion using our pseudo-expected utility proposed in \Secref{gibbsdecthe}, rather than straightforwardly taking an expectation with respect to the statistical model. See \Appref{importanceofimportanceweight} for a comparison between our approach and \citet{overstall2023gibbsoptimaldesignexperiments}. 

Our approach also departs from that of \citet{overstall2023gibbsoptimaldesignexperiments} in that \citet{overstall2023gibbsoptimaldesignexperiments} make a normal approximation of the Gibbs posterior --- utilising this approximation both in performing inference and in computing the expected utility. Although there are a number of scenarios under which normal approximations in the misspecified setting are viable (see \citet{bochkina2023bernstein} for a review), they generally require access to a large enough dataset for the approximation to be valid. The requirement of a large dataset is usually not satisfied in the experimental design setting. 

\citet{jewson2022general} explore model selection with improper models, and ultimately also consider the decision-theoretic implications of the GBI framework. They exploit score matching in order to choose between models, and form an analogue to the Bayesian marginal likelihood, and subsequently the Bayes factor. We note that their focus is on model selection, rather than decision-making and experimental design in general. Using the marginal exponentiated loss in our framework according to the assumptions in \Secref{gibbsdecthe} offers the same decision-making capabilities in \citet{jewson2022general}. This is because we assume that the marginal exponentiated loss is finite, non-negative, and integrable.

\section{Gibbs Expected Information Gain Estimation} 
\label{nmcappendix}

\newcommand{\const}[1]{\mathbf{c}_{#1}}

Many of the estimators for the BEIG \citep{foster2019variational} can be naturally generalised to methods for approximating the Gibbs EIG. This appendix explains our Gibbs EIG NMC estimator.

\subsection{Nested Monte Carlo Estimator}

The NMC estimator for the BEIG is well-studied in BOED \citep{rainforth2018nesting}. It is a simple approach to approximating the EIG which provides a na\"{i}ve NMC approximation to the marginal likelihood.

Following \Cref{theorem1}, we arrive at our NMC estimator knowing that $$\widetilde{\pi}(\boldsymbol{y} \mid \boldsymbol{\xi}) = \mathbb{E}_{\pi(\boldsymbol{\theta})}\left[\exp(-\omega\ell_{\boldsymbol{\theta}}(\boldsymbol{\xi}, \boldsymbol{y}))\right]$$ can be approximated by, for sufficiently large $M$ samples from $\pi(\boldsymbol{\theta})$, $$\widetilde{\pi}(\boldsymbol{y} \mid \boldsymbol{\xi}) \approx \frac{1}{M}\sum_{j = 1}^{M} \exp\left(-\omega\ell_{\boldsymbol{\theta}_{j}}(\boldsymbol{\xi}, \boldsymbol{y}) \right).$$

The NMC estimator then easily follows by substitution. Using proposal $q(\boldsymbol{y} \mid \boldsymbol{\theta}, \boldsymbol{\xi})$, the NMC estimator is $$U^{\mathrm{Gibbs}}_{\mathrm{NMC}}(\boldsymbol{\xi}) \triangleq \sum_{i = 1}^{N}\left[\left(-\omega\ell_{\boldsymbol{\theta}_{i}}(\boldsymbol{\xi}, \boldsymbol{y}_{i}) - \log\left(\frac{1}{M}\sum_{j = 1}^{M} \exp\left(-\omega\ell_{\boldsymbol{\theta}_{ij}}(\boldsymbol{\xi}, \boldsymbol{y}_{i}) \right)\right)\right) \cdot Z_{\boldsymbol{\theta}_{i}}(\boldsymbol{\xi}, \boldsymbol{y}_{i})\right],$$
where $$Z_{\boldsymbol{\theta}_{i}}(\boldsymbol{\xi}, \boldsymbol{y}_{i}) = \left(\frac{\exp(-\omega\ell_{\boldsymbol{\theta}_{i}}(\boldsymbol{\xi}, \boldsymbol{y}_{i}))}{q(\boldsymbol{y}_{i} \mid \boldsymbol{\theta}_{i}, \boldsymbol{\xi})}\right) \cdot \frac{1}{\const{1,i}\const{2,i}},$$ $\const{1,i}\const{2,i}$ is a suitable weight described in \Appref{ap:normalisation},
and $$\boldsymbol{\theta}_{i}, \boldsymbol{y}_{i} \sim \pi(\boldsymbol{\theta})q(\boldsymbol{y} \mid \boldsymbol{\theta}, \boldsymbol{\xi}), \boldsymbol{\theta}_{ij} \sim \pi(\boldsymbol{\theta}).$$
Proceeding in this manner avoids explicitly computing $C(\boldsymbol{\xi}) = \int_{\mathcal{Y}} \widetilde{\pi}(\boldsymbol{y} \mid \boldsymbol{\xi}) d\boldsymbol{y}$.

\subsection{Normalising the Importance Weight}\label{ap:normalisation}

We mentioned briefly that we use self-normalised importance sampling \citep{Elvira_2021}. Self-normalised importance sampling avoids explicit computation of the normalising constants attached to loss-based densities (which make them integrable to 1). This offers a practical solution to approximating the Gibbs EIG.

For $C(\boldsymbol{\xi})$ to not be computed explicitly, we use 
\begin{align*}
C(\boldsymbol{\xi}) & = \int_{\mathcal{Y}} \widetilde{\pi}(\boldsymbol{y} \mid \boldsymbol{\xi}) d\boldsymbol{y} = \int_{\mathcal{Y}} \int_{\mathcal{T}} \pi(\boldsymbol{\theta})\exp\left(-\omega\ell_{\boldsymbol{\theta}}(\boldsymbol{\xi}, \boldsymbol{y}) \right) d\boldsymbol{\theta} d\boldsymbol{y} \\
& = \int_{\mathcal{Y}} \int_{\mathcal{T}} \frac{\pi(\boldsymbol{\theta})\exp\left(-\omega\ell_{\boldsymbol{\theta}}(\boldsymbol{\xi}, \boldsymbol{y}) \right)}{\pi(\boldsymbol{\theta})q(\boldsymbol{y} \mid \boldsymbol{\theta}, \boldsymbol{\xi})} \pi(\boldsymbol{\theta})q(\boldsymbol{y} \mid \boldsymbol{\theta}, \boldsymbol{\xi})d\boldsymbol{\theta} d\boldsymbol{y} \\
& = \mathbb{E}_{\pi(\boldsymbol{\theta})q(\boldsymbol{y} \mid \boldsymbol{\theta}, \boldsymbol{\xi})}\left[\frac{\exp(-\omega\ell_{\boldsymbol{\theta}}(\boldsymbol{\xi}, \boldsymbol{y}))}{q(\boldsymbol{y} \mid \boldsymbol{\theta}, \boldsymbol{\xi})}\right].
\end{align*}

The NMC estimator follows from substituting into \Cref{theorem1}. As a result, the importance ratio is made to sum to one. In particular, we normalise the numerator of $\frac{\exp(-\omega\ell_{\boldsymbol{\theta}}(\boldsymbol{\xi}, \boldsymbol{y}))}{q(\boldsymbol{y} \mid \boldsymbol{\theta}, \boldsymbol{\xi})}$, and finally the resulting numerator normalised fraction overall. The first of these tackles the numerical instability problem of working with exponentials. As we require the importance ratio to be normalised, it is a given that we normalise the overall fraction after normalising the numerator. The denominator can optionally be normalised too --- the same result is recovered regardless.

Starting by normalising $\exp(-\omega\ell_{\boldsymbol{\theta}}(\boldsymbol{\xi}, \boldsymbol{y}))$, and then the resulting fraction, we have the importance ratio
\begin{align}
    Z_{\boldsymbol{\theta}_{i}}(\boldsymbol{\xi}, \boldsymbol{y}_{i}) &= \frac{1}{\const{2,i}} \frac{\frac{1}{\const{1,i}}\exp(-\omega\ell_{\boldsymbol{\theta}_{i}}(\boldsymbol{\xi}, \boldsymbol{y}_{i}))}{q(\boldsymbol{y}_{i} \mid \boldsymbol{\theta}_{i}, \boldsymbol{\xi})} \nonumber \\
    &= \frac{1}{\const{1,i} \const{2,i}} \frac{\exp(-\omega\ell_{\boldsymbol{\theta}_{i}}(\boldsymbol{\xi}, \boldsymbol{y}_{i}))}{q(\boldsymbol{y}_{i} \mid \boldsymbol{\theta}_{i}, \boldsymbol{\xi})} \nonumber \\
    &= \frac{1}{\sum_{j=1}^N \frac{\exp(-\omega\ell_{\boldsymbol{\theta}_{i}}(\boldsymbol{\xi}, \boldsymbol{y}_{j}))}{q(\boldsymbol{y}_{j} \mid \boldsymbol{\theta}_{i}, \boldsymbol{\xi})}} \frac{\exp(-\omega\ell_{\boldsymbol{\theta}_{i}}(\boldsymbol{\xi}, \boldsymbol{y}_{i}))}{q(\boldsymbol{y}_{i} \mid \boldsymbol{\theta}_{i}, \boldsymbol{\xi})}, \label{eq:const-cancel}
\end{align}
where $\const{1,i} \equiv \sum_{j=1}^N \exp(-\omega\ell_{\boldsymbol{\theta}_{i}}(\boldsymbol{\xi}, \boldsymbol{y}_{j}))$, $\const{2,i} \equiv \sum_{j=1}^N \frac{\frac{1}{\const{1,i}}\exp(-\omega\ell_{\boldsymbol{\theta}_{i}}(\boldsymbol{\xi}, \boldsymbol{y}_{j}))}{q(\boldsymbol{y}_{j} \mid \boldsymbol{\theta}_{i}, \boldsymbol{\xi})}$ and \Cref{eq:const-cancel} follows from
\begin{align}
    \const{2,i} &= \sum_{j=1}^N \frac{\frac{1}{\const{1,i}}\exp(-\omega\ell_{\boldsymbol{\theta}_{i}}(\boldsymbol{\xi}, \boldsymbol{y}_{j}))}{q(\boldsymbol{y}_{j} \mid \boldsymbol{\theta}_{i}, \boldsymbol{\xi})} \nonumber \\
    &= \frac{1}{\const{1,i}} \sum_{j=1}^N \frac{\exp(-\omega\ell_{\boldsymbol{\theta}_{i}}(\boldsymbol{\xi}, \boldsymbol{y}_{j}))}{q(\boldsymbol{y}_{j} \mid \boldsymbol{\theta}_{i}, \boldsymbol{\xi})}. \nonumber
\end{align}

The same result is recovered regardless of whether or not one normalises the numerator first. The point of normalising this is to protect against numerical instability (from using exponentials) and to accurately estimate the Gibbs EIG. 

We can demonstrate that the BEIG is recovered using $\omega = 1$, $\ell_{\boldsymbol{\theta}}(\boldsymbol{\xi}, \boldsymbol{y}) = -\log p(\boldsymbol{y} \mid \boldsymbol{\theta}, \boldsymbol{\xi})$, and $q(\boldsymbol{y} \mid \boldsymbol{\theta}, \boldsymbol{\xi}) = p(\boldsymbol{y} \mid \boldsymbol{\theta}, \boldsymbol{\xi})$. We have that
\begin{align*}
    Z_{\boldsymbol{\theta}_{i}}(\boldsymbol{\xi}, \boldsymbol{y}_{i}) &= \frac{1}{\sum_{j=1}^N \frac{\exp(-\omega\ell_{\boldsymbol{\theta}_{i}}(\boldsymbol{\xi}, \boldsymbol{y}_{j}))}{p(\boldsymbol{y}_{j} \mid \boldsymbol{\theta}_{i}, \boldsymbol{\xi})}} \frac{\exp(-\omega\ell_{\boldsymbol{\theta}_{i}}(\boldsymbol{\xi}, \boldsymbol{y}_{i}))}{p(\boldsymbol{y}_{i} \mid \boldsymbol{\theta}_{i}, \boldsymbol{\xi})} \\
    &= \frac{1}{\sum_{j=1}^N \frac{p(\boldsymbol{y}_{j} \mid \boldsymbol{\theta}_{i}, \boldsymbol{\xi})}{p(\boldsymbol{y}_{j} \mid \boldsymbol{\theta}_{i}, \boldsymbol{\xi})}} \frac{p(\boldsymbol{y}_{i} \mid \boldsymbol{\theta}_{i}, \boldsymbol{\xi})}{p(\boldsymbol{y}_{i} \mid \boldsymbol{\theta}_{i}, \boldsymbol{\xi})} \\
    &= \frac{1}{N}.
\end{align*}

Putting this back into the Gibbs EIG NMC estimator,
\begin{align*}
U_{\mathrm{NMC}}(\boldsymbol{\xi}) & \triangleq \sum_{i = 1}^{N}\left[\left(-\omega\ell_{\boldsymbol{\theta}_{i}}(\boldsymbol{\xi}, \boldsymbol{y}_{i}) - \log\left(\frac{1}{M}\sum_{j = 1}^{M} \exp\left(-\omega\ell_{\boldsymbol{\theta}_{ij}}(\boldsymbol{\xi}, \boldsymbol{y}_{i}) \right)\right)\right) \cdot Z_{\boldsymbol{\theta}_{i}}(\boldsymbol{\xi}, \boldsymbol{y}_{i})\right] \\
& = \sum_{i = 1}^{N}\left[\left(\log p(\boldsymbol{y}_{i} \mid \boldsymbol{\theta}_{i}, \boldsymbol{\xi}) - \log\left(\frac{1}{M}\sum_{j = 1}^{M} p(\boldsymbol{y}_{i} \mid \boldsymbol{\theta}_{ij}, \boldsymbol{\xi})\right)\right) \cdot \frac{1}{N}\right] \\
& = \frac{1}{N}\sum_{i = 1}^{N}\left[\left(\log p(\boldsymbol{y}_{i} \mid \boldsymbol{\theta}_{i}, \boldsymbol{\xi}) - \log\left(\frac{1}{M}\sum_{j = 1}^{M} p(\boldsymbol{y}_{i} \mid \boldsymbol{\theta}_{ij}, \boldsymbol{\xi})\right)\right)\right].
\end{align*}
Here, we recover the BEIG NMC estimator from \citet{rainforth2018nesting} since $Z$ becomes the constant $\frac{1}{N}$.

\subsection{Importance of the Importance Weight}
\label{importanceofimportanceweight}

One could instead consider computing the following $$U(\boldsymbol{\xi}) \triangleq \frac{1}{N}\sum_{i = 1}^{N}\left[\left(-\omega\ell_{\boldsymbol{\theta}_{i}}(\boldsymbol{y}_{i}, \boldsymbol{\xi}) - \log\left(\frac{1}{M}\sum_{j = 1}^{M} \exp\left(-\omega\ell_{\boldsymbol{\theta}_{ij}}(\boldsymbol{y}_{i}, \boldsymbol{\xi}) \right)\right)\right)\right],$$
where $$\boldsymbol{\theta}_{i}, \boldsymbol{y}_{i} \sim \pi(\boldsymbol{\theta})p(\boldsymbol{y} \mid \boldsymbol{\theta}, \boldsymbol{\xi}), \boldsymbol{\theta}_{ij} \sim \pi(\boldsymbol{\theta}),$$
i.e., the Gibbs EIG in \Cref{theorem1} without the importance weight and $q(\boldsymbol{y} \mid \boldsymbol{\theta}, \boldsymbol{\xi}) = p(\boldsymbol{y} \mid \boldsymbol{\theta}, \boldsymbol{\xi})$.

This is equivalent to the setting by \citet{overstall2023gibbsoptimaldesignexperiments}, where we directly use their Equation (4) for computing the expected utility. The connection to our setting exists when the designer distribution (which they specify as $\mathcal{D}$) is our (misspecified) statistical model, and the outer-expectation of the utility with respect to a quantity (which they specify as $\mathcal{C}$) is the prior distribution (the prior is the quantity). The utility is then the log-ratio between the posterior and prior. In other words, the expected utility is the following (non-pseudo) expectation 
\begin{align*}
\mathrm{EIG}(\boldsymbol{\xi}) & = \mathbb{E}_{\pi(\boldsymbol{\theta})p(\boldsymbol{y} \mid \boldsymbol{\theta}, \boldsymbol{\xi})}\left[\log \pi(\boldsymbol{\theta} \mid \boldsymbol{\xi}, \boldsymbol{y}) - \log \pi(\boldsymbol{\theta})\right] \\
& = \mathbb{E}_{\pi(\boldsymbol{\theta})p(\boldsymbol{y} \mid \boldsymbol{\theta}, \boldsymbol{\xi})}\left[-\omega\ell_{\boldsymbol{\theta}}(\boldsymbol{\xi}, \boldsymbol{y}) - \log \widetilde{\pi}(\boldsymbol{y} \mid \boldsymbol{\xi})\right].
\end{align*}
It is easy to see that, like our proposed Gibbs EIG, this is equivalent to the BEIG for $\omega = 1$ and $\ell_{\boldsymbol{\theta}}(\boldsymbol{\xi}, \boldsymbol{y}) = -\log p(\boldsymbol{y} \mid \boldsymbol{\theta}, \boldsymbol{\xi})$. However, this is not in general equivalent to the Gibbs EIG presented in \Cref{definitiongibbseig}. The lack of the importance weight means that we cannot rearrange to recover any of \Cref{definitiongibbseig}, \Cref{proposition1}, and \Cref{lemma1} --- in other words, we do not make use of pseudo-expectations. 

As a result, a different expectation gets computed rather than that provided by the Gibbs EIG. This therefore gives a different interpretation of the EIG, where one computes the same function but using samples solely from the statistical model. Without the importance weight, the exponentiated loss $\exp\left(-\omega\ell_{\boldsymbol{\theta}}(\boldsymbol{\xi}, \boldsymbol{y}) \right)$ has no say in how the samples from the statistical model contribute to the log-ratio we estimate. Consequently, we deviate from the intuition laid out behind developing the Gibbs EIG.

We can show empirically what the consequences of not using the importance weight in the Gibbs EIG are. Below, in \Tableref{importanceweightregres}, are the results from performing regression as in \Appref{bayesregresappendix}, using weighted score matching with exponential decay $(b = 0.04)$. The same loss function and other related parameters are used for a fair comparison. Evidence suggests that using the importance weight as in \Cref{theorem1} provides better performance, additionally suggesting that na\"ively using directly the \citet{overstall2023gibbsoptimaldesignexperiments} method of computing the expected utility is not optimal. 

\begin{table}[ht!]
\caption{Comparison of using an importance weight as in \Cref{theorem1}, or avoiding it as in \citet{overstall2023gibbsoptimaldesignexperiments}. RMSE, MMD, and NLL are recorded with mean (and standard error) over 3 different models, each replicated under 30 random seeds $(3 \times 30 = 90 \text{ replications in total})$. \textit{Top}: Well-specified setting. \textit{Middle}: Observations corrupted with Asymmetric Outliers. \textit{Bottom}: True error distribution is Laplacian.}
\centering
\begin{tabular}{lccc}
\toprule
\textbf{Rate} & \textbf{RMSE} & \textbf{MMD} & \textbf{NLL} \\
\midrule
\multicolumn{4}{c}{Well-Specified} \\[0.2em]
Importance Weight  & \textbf{1.6024\,(0.0402)} & \textbf{0.0768\,(0.0079)} & \textbf{1.6288\,(0.0371)} \\
No Importance Weight  & $1.6437\,(0.0428)$ & $0.0965\,(0.0099)$ & $1.6954\,(0.0434)$ \\
\midrule
\multicolumn{4}{c}{Asymmetric Outliers} \\[0.2em]
Importance Weight  & \textbf{2.4716\,(0.1016)} & \textbf{0.4803\,(0.0314)} & \textbf{3.4828\,(0.2064)} \\
No Importance Weight  & $2.4991\,(0.0998)$ & $0.4886\,(0.0312)$ & $3.5836\,(0.2010)$ \\
\midrule
\multicolumn{4}{c}{Misspecified Error Distribution} \\[0.2em]
Importance Weight  & \textbf{1.9652\,(0.0487)} & \textbf{0.1125\,(0.0115)} & \textbf{2.2494\,(0.0534)} \\
No Importance Weight  & $1.9878\,(0.0499)$ & $0.1227\,(0.0118)$ & $2.2883\,(0.059)$ \\
\bottomrule
\end{tabular}
\label{importanceweightregres}
\end{table}

\subsection{Importance Sampling Diagnostics}
\label{impsampdiagnostics}

For an empirical assessment of how our importance sampling approach performs, \Cref{impresults,heavytailres} compare the performance and stability of the GEIG estimates when using the statistical model and a heavy-tailed distribution, respectively, as the proposal distribution used to compute the GEIG. These tables show the mean (over all designs), effective sample size (ESS), and variance of the estimated Gibbs EIG for the Bayesian linear regression experimental design problem. These tables also show the MMD between values sampled from the predictive distribution and from the true DGP, and the estimated Gibbs EIG. We display these for both the well-specified and Asymmetric Outliers scenarios. In addition, the results are separated according to time $T$, showing the performance of the importance sampling approach at different stages of experimentation. 

\paragraph{Using the statistical model} In our experiments, we used the statistical model as our proposal for computing the Gibbs EIG. Results from using the model as our proposal can be found in \Tableref{impresults}. In summary, the variance remains small across $T$, and the $\text{ESS} / N$ can vary depending on the loss function and learning rate combination used. The loss functions used in our experiments have relatively high ESS, such as weighted score matching with exponential decay, as seen. We have also included cases where the ESS can be much lower, such as using power likelihoods with $\omega = 0.4$ and weighted score matching with $c = 2$ fixed. One can infer that if $\exp\left(-\omega\ell_{\boldsymbol{\theta}}(\boldsymbol{\xi}, \boldsymbol{y}) \right)$ is far from $p(\boldsymbol{y} \mid \boldsymbol{\theta}, \boldsymbol{\xi})$ (e.g., if a small learning rate is used in a power likelihood), the ESS falls considerably. This can lead to the Gibbs EIG approximation becoming poor. Results remain consistent between the well-specified and Asymmetric Outliers scenarios.

\begin{table*}[ht!]
\caption{Importance sampling diagnostics for the Gibbs EIG (GEIG) in the Bayesian linear regression problem, using the statistical model as the proposal. MMD at current point in time $T$ by choosing optimal design via GEIG (MMD), GEIG estimate (GEIG), variance of GEIG estimate (Var), effective sample size over $N$ ($\text{ESS} / N$). $N = 10000$ and reported are statistics as an average over 90 replications, with standard errors in brackets.}
\centering
\renewcommand{\arraystretch}{1.05}
\setlength{\tabcolsep}{4pt}

\begin{adjustbox}{max width=\textwidth}
\begin{tabular}{lcccccccccccc}
\toprule
\multirow{2}{*}{\textbf{Loss}} & 
\multicolumn{4}{c}{$T = 1$} & 
\multicolumn{4}{c}{$T = 5$} & 
\multicolumn{4}{c}{$T = 10$} \\
\cmidrule(lr){2-5} \cmidrule(lr){6-9} \cmidrule(lr){10-13}
 & \textbf{MMD} & \textbf{GEIG} & \textbf{Var} & \textbf{$\text{ESS} / N$}
 & \textbf{MMD} & \textbf{GEIG} & \textbf{Var} & \textbf{$\text{ESS} / N$}
 & \textbf{MMD} & \textbf{GEIG} & \textbf{Var} & \textbf{$\text{ESS} / N$} \\
\midrule

\multicolumn{13}{c}{{Well-Specified}} \\[0.2em]

Power-Like $\omega = 0.4$ 
& $1.135\,(0.045)$ & $0.592\,(0.012)$ & $2.654e-03\,(1.596e-04)$ & $0.318\,(0.001)$
& $0.754\,(0.044)$ & $0.110\,(0.001)$ & $8.750e-04\,(6.803e-05)$ & $0.312\,(0.001)$
& $0.403\,(0.028)$ & $0.061\,(0.000)$ & $4.843e-04\,(3.566e-05)$ & $0.315\,(0.001)$ \\

Power-Like $\omega = 0.8$ 
& \textbf{1.100\,(0.050)} & $0.836\,(0.014)$ & $1.276e-04\,(1.766e-06)$ & \textbf{0.968\,(0.000)}
& $0.523\,(0.041)$ & $0.124\,(0.001)$ & $3.058e-05\,(1.999e-07)$ & \textbf{0.968\,(0.000)}
& $0.189\,(0.016)$ & $0.064\,(0.000)$ & $1.661e-05\,(8.711e-08)$ & \textbf{0.968\,(0.000)} \\

Unweighted-SM
& $1.127\,(0.049)$ & $1.264\,(0.034)$ & $2.161e-04\,(1.634e-05)$ & $0.853\,(0.010)$
& \textbf{0.318\,(0.040)} & $0.136\,(0.001)$ & $1.961e-05\,(1.616e-07)$ & $0.853\,(0.010)$
& $0.081\,(0.008)$ & $0.067\,(0.000)$ & $1.016e-05\,(5.969e-08)$ & $0.853\,(0.010)$ \\

Laplante
& $1.334\,(0.007)$ & $0.969\,(0.026)$ & $8.781e-04\,(9.474e-05)$ & $0.432\,(0.013)$
& $1.040\,(0.053)$ & $0.249\,(0.009)$ & $7.372e-04\,(6.684e-05)$ & $0.189\,(0.010)$
& $0.938\,(0.065)$ & $0.123\,(0.005)$ & $3.587e-04\,(3.158e-05)$ & $0.155\,(0.009)$ \\

$r_{\mathrm{IMQ}} \mid c = 2$
& $1.439\,(0.005)$ & $0.899\,(0.028)$ & $1.495e-03\,(1.733e-04)$ & $0.336\,(0.015)$
& $0.923\,(0.058)$ & $0.538\,(0.014)$ & $1.222e-03\,(1.471e-04)$ & $0.405\,(0.019)$
& $0.829\,(0.062)$ & $0.110\,(0.004)$ & $2.606e-04\,(4.023e-05)$ & $0.466\,(0.020)$ \\

$r_{\mathrm{IMQ}} \mid c = 10$
& $1.134\,(0.027)$ & $1.158\,(0.030)$ & $6.743e-05\,(9.744e-07)$ & $0.846\,(0.010)$
& $0.340\,(0.040)$ & $0.142\,(0.001)$ & $1.995e-05\,(1.413e-07)$ & $0.845\,(0.010)$
& \textbf{0.075\,(0.008)} & $0.066\,(0.000)$ & $9.995e-06\,(6.256e-08)$ & $0.845\,(0.010)$ \\

Exp-Decay $b = 0.04$
& $1.134\,(0.027)$ & $1.158\,(0.030)$ & $6.743e-05\,(9.744e-07)$ & $0.846\,(0.010)$
& $0.360\,(0.041)$ & $0.143\,(0.001)$ & $2.003e-05\,(1.438e-07)$ & $0.843\,(0.010)$
& $0.077\,(0.008)$ & $0.066\,(0.000)$ & $9.897e-06\,(6.482e-08)$ & $0.839\,(0.010)$ \\

Exp-Decay $b = 0.08$
& $1.134\,(0.027)$ & $1.158\,(0.030)$ & $6.743e-05\,(9.744e-07)$ & $0.846\,(0.010)$
& $0.375\,(0.042)$ & $0.145\,(0.001)$ & $2.008e-05\,(1.465e-07)$ & $0.840\,(0.010)$
& $0.084\,(0.009)$ & $0.066\,(0.000)$ & $1.004e-05\,(8.959e-08)$ & $0.827\,(0.010)$ \\

Exp-Decay $b = 0.12$
& $1.134\,(0.027)$ & $1.158\,(0.030)$ & $6.743e-05\,(9.744e-07)$ & $0.846\,(0.010)$
& $0.402\,(0.044)$ & $0.147\,(0.001)$ & $2.024e-05\,(1.571e-07)$ & $0.836\,(0.010)$
& $0.111\,(0.017)$ & $0.067\,(0.000)$ & $1.624e-05\,(2.025e-06)$ & $0.795\,(0.008)$ \\

\midrule
\multicolumn{13}{c}{{Asymmetric Outliers}} \\[0.2em]

Power-Like $\omega = 0.4$ 
& $1.159\,(0.046)$ & $0.592\,(0.012)$ & $2.654e-03\,(1.596e-04)$ & $0.318\,(0.001)$
& $0.842\,(0.051)$ & $0.113\,(0.001)$ & $6.480e-04\,(3.413e-05)$ & $0.319\,(0.001)$
& $0.688\,(0.043)$ & $0.063\,(0.001)$ & $5.375e-04\,(4.104e-05)$ & $0.316\,(0.001)$ \\

Power-Like $\omega = 0.8$ 
& $1.175\,(0.048)$ & $0.836\,(0.014)$ & $1.276e-04\,(1.766e-06)$ & \textbf{0.968\,(0.000)}
& $0.738\,(0.047)$ & $0.123\,(0.001)$ & $3.032e-05\,(2.079e-07)$ & \textbf{0.968\,(0.000)}
& $0.571\,(0.035)$ & $0.064\,(0.000)$ & $1.675e-05\,(9.774e-08)$ & \textbf{0.968\,(0.000)} \\

Unweighted-SM
& $1.235\,(0.045)$ & $1.264\,(0.034)$ & $2.161e-04\,(1.634e-05)$ & $0.853\,(0.010)$
& \textbf{0.582\,(0.042)} & $0.136\,(0.001)$ & $1.958e-05\,(1.628e-07)$ & $0.853\,(0.010)$
& $0.519\,(0.030)$ & $0.067\,(0.000)$ & $1.026e-05\,(6.983e-08)$ & $0.853\,(0.010)$ \\

Laplante
& $1.334\,(0.008)$ & $0.969\,(0.026)$ & $8.781e-04\,(9.474e-05)$ & $0.432\,(0.013)$
& $1.120\,(0.052)$ & $0.288\,(0.010)$ & $6.374e-04\,(6.706e-05)$ & $0.208\,(0.010)$
& $1.030\,(0.061)$ & $0.164\,(0.007)$ & $4.257e-04\,(3.301e-05)$ & $0.164\,(0.009)$ \\

$r_{\mathrm{IMQ}} \mid c = 2$
& $1.440\,(0.005)$ & $0.899\,(0.028)$ & $1.495e-03\,(1.733e-04)$ & $0.336\,(0.015)$
& $0.985\,(0.053)$ & $0.535\,(0.014)$ & $1.101e-03\,(1.599e-04)$ & $0.411\,(0.019)$
& $0.898\,(0.060)$ & $0.132\,(0.006)$ & $3.007e-04\,(4.130e-05)$ & $0.464\,(0.020)$ \\

$r_{\mathrm{IMQ}} \mid c = 10$
& \textbf{1.137\,(0.031)} & $1.158\,(0.030)$ & $6.743e-05\,(9.744e-07)$ & $0.846\,(0.010)$
& $0.669\,(0.043)$ & $0.142\,(0.001)$ & $2.002e-05\,(1.826e-07)$ & $0.845\,(0.010)$
& $0.488\,(0.031)$ & $0.070\,(0.000)$ & $1.045e-05\,(7.190e-08)$ & $0.845\,(0.010)$ \\

Exp-Decay $b = 0.04$
& \textbf{1.137\,(0.031)} & $1.158\,(0.030)$ & $6.743e-05\,(9.744e-07)$ & $0.846\,(0.010)$
& $0.683\,(0.044)$ & $0.144\,(0.001)$ & $2.006e-05\,(1.889e-07)$ & $0.843\,(0.010)$
& \textbf{0.480\,(0.031)} & $0.071\,(0.001)$ & $1.058e-05\,(8.009e-08)$ & $0.839\,(0.010)$ \\

Exp-Decay $b = 0.08$
& \textbf{1.137\,(0.031)} & $1.158\,(0.030)$ & $6.743e-05\,(9.744e-07)$ & $0.846\,(0.010)$
& $0.714\,(0.046)$ & $0.145\,(0.001)$ & $2.006e-05\,(1.903e-07)$ & $0.840\,(0.010)$
& $0.481\,(0.034)$ & $0.073\,(0.001)$ & $1.125e-05\,(2.618e-07)$ & $0.827\,(0.010)$ \\

Exp-Decay $b = 0.12$
& \textbf{1.137\,(0.031)} & $1.158\,(0.030)$ & $6.743e-05\,(9.744e-07)$ & $0.846\,(0.010)$
& $0.735\,(0.047)$ & $0.146\,(0.001)$ & $2.014e-05\,(1.953e-07)$ & $0.836\,(0.010)$
& $0.500\,(0.038)$ & $0.076\,(0.001)$ & $2.302e-05\,(4.751e-06)$ & $0.794\,(0.008)$ \\
\bottomrule
\end{tabular}
\end{adjustbox}
\label{impresults}
\end{table*}

\paragraph{Using a heavy-tailed proposal} \Tableref{heavytailres} presents the results when using a heavy-tailed distribution, namely the Student's $t$ with 4 degrees of freedom, as the proposal for computing the Gibbs EIG. Results when using the model are also included for comparison. ESS is generally larger regardless of loss function when using the statistical model as the proposal. Predictive performance assessed by the MMD is also stronger by the end of experimentation when using the model.

\begin{table*}[ht!]
\caption{Importance sampling diagnostics for the Gibbs EIG (GEIG) in the Bayesian linear regression problem, using a heavy-tailed Student's $t$ distribution with 4 degrees of freedom as the proposal. MMD at current point in time $T$ by choosing optimal design via GEIG (MMD), GEIG estimate (GEIG), variance of GEIG estimate (Var), effective sample size over $N$ ($\text{ESS} / N$). $N = 10000$ and reported are statistics as an average over 90 replications, with standard errors in brackets.}
\centering
\renewcommand{\arraystretch}{1.05}
\setlength{\tabcolsep}{4pt}
\begin{adjustbox}{max width=\textwidth}
\begin{tabular}{lcccccc}
\toprule
\multirow{2}{*}{\textbf{Loss}} & 
\multicolumn{2}{c}{$T = 1$} & 
\multicolumn{2}{c}{$T = 5$} & 
\multicolumn{2}{c}{$T = 10$} \\
\cmidrule(lr){2-3} \cmidrule(lr){4-5} \cmidrule(lr){6-7}
 & \textbf{MMD} & \textbf{$\text{ESS} / N$}
 & \textbf{MMD} & \textbf{$\text{ESS} / N$}
 & \textbf{MMD} & \textbf{$\text{ESS} / N$} \\
\midrule

\multicolumn{7}{c}{{Well-Specified}} \\[0.2em]

Laplante (Model)
& $1.334\,(0.007)$ & $0.432\,(0.013)$
& $1.040\,(0.053)$ & $0.189\,(0.010)$
& $0.938\,(0.065)$ & $0.155\,(0.009)$ \\

Laplante (Heavy)
& $1.329\,(0.007)$ & $0.325\,(0.016)$
& $1.012\,(0.060)$ & $0.013\,(0.001)$
& $0.946\,(0.066)$ & $0.009\,(0.001)$ \\

Exp-Decay $b = 0.04$ (Model)
& $1.134\,(0.027)$ & \textbf{0.846\,(0.010)}
& $0.360\,(0.041)$ & \textbf{0.843\,(0.010)}
& \textbf{0.077\,(0.008)} & \textbf{0.839\,(0.010)} \\

Exp-Decay $b = 0.04$ (Heavy)
& \textbf{1.094\,(0.032)} & $0.782\,(0.010)$
& \textbf{0.320\,(0.036)} & $0.778\,(0.009)$
& $0.093\,(0.011)$ & $0.774\,(0.009)$ \\

Exp-Decay $b = 0.12$ (Model)
& $1.134\,(0.027)$ & \textbf{0.846\,(0.010)}
& $0.402\,(0.044)$ & $0.836\,(0.010)$
& $0.111\,(0.017)$ & $0.795\,(0.008)$ \\

Exp-Decay $b = 0.12$ (Heavy)
& \textbf{1.094\,(0.032)} & $0.782\,(0.010)$
& $0.393\,(0.043)$ & $0.772\,(0.009)$
& $0.137\,(0.022)$ & $0.730\,(0.007)$ \\

\midrule
\multicolumn{7}{c}{{Asymmetric Outliers}} \\[0.2em]

Laplante (Model)
& $1.334\,(0.008)$ & $0.432\,(0.013)$
& $1.120\,(0.052)$ & $0.208\,(0.010)$
& $1.030\,(0.061)$ & $0.164\,(0.009)$ \\

Laplante (Heavy)
& $1.328\,(0.007)$ & $0.325\,(0.016)$
& $1.107\,(0.052)$ & $0.017\,(0.002)$
& $1.030\,(0.061)$ & $0.010\,(0.001)$ \\

Exp-Decay $b = 0.04$ (Model)
& $1.137\,(0.031)$ & \textbf{0.846\,(0.010)}
& \textbf{0.683\,(0.044)} & \textbf{0.843\,(0.010)}
& \textbf{0.480\,(0.031)} & \textbf{0.839\,(0.010)} \\

Exp-Decay $b = 0.04$ (Heavy)
& \textbf{1.100\,(0.036)} & $0.782\,(0.010)$
& $0.775\,(0.049)$ & $0.778\,(0.009)$
& $0.595\,(0.043)$ & $0.775\,(0.009)$ \\

Exp-Decay $b = 0.12$ (Model)
& $1.137\,(0.031)$ & \textbf{0.846\,(0.010)}
& $0.735\,(0.047)$ & $0.836\,(0.010)$
& $0.500\,(0.038)$ & $0.794\,(0.008)$ \\

Exp-Decay $b = 0.12$ (Heavy)
& \textbf{1.100\,(0.036)} & $0.782\,(0.010)$
& $0.803\,(0.050)$ & $0.772\,(0.009)$
& $0.627\,(0.047)$ & $0.730\,(0.007)$ \\
\bottomrule
\end{tabular}
\end{adjustbox}
\label{heavytailres}
\end{table*}

\section{Practical Guidelines}
\label{pracguidelines}

As a generalisation of BOED, GBOED has a number of additional hyperparameters that need to be set in order to be used in practice. These include the learning rate and loss function. 

\subsection{Learning Rate}
\label{learningrateselection}

\Cref{geigtendzero} yields insight into how the learning rate affects the Gibbs EIG. \Cref{geigtendzero} explains that the Gibbs EIG shrinks as $\omega \to 0$, eventually becoming 0 at the limit.

\begin{proposition} \label{geigtendzero}
The Gibbs EIG tends to zero as the learning rate $\omega$ approaches zero: $$\lim_{\omega \to 0}\mathrm{GEIG}(\boldsymbol{\xi}) = 0.$$
\end{proposition}
\begin{proof}
Starting from \Cref{definitiongibbseig}, the Gibbs posterior update rule in \Cref{gibbsposteriorrule}, and using the limit $\omega \to 0$,
\begin{align*}
\mathrm{GEIG}(\boldsymbol{\xi}) & = \mathbb{E}_{\pi(\boldsymbol{y} \mid \boldsymbol{\xi})}\left[\mathrm{KL}(\pi(\boldsymbol{\theta} \mid \boldsymbol{\xi}, \boldsymbol{y}) \Vert \pi(\boldsymbol{\theta}))\right] \\
& = \mathbb{E}_{\pi(\boldsymbol{y} \mid \boldsymbol{\xi})}\left[\mathrm{KL}\left(\frac{\exp\left(-\omega\ell_{\boldsymbol{\theta}}(\boldsymbol{\xi}, \boldsymbol{y}) \right) \pi(\boldsymbol{\theta})}{\int_{\mathcal{T}} \exp\left(-\omega\ell_{\boldsymbol{\theta^{\prime}}}(\boldsymbol{\xi}, \boldsymbol{y}) \right) \pi(\boldsymbol{\theta^{\prime}}) d\boldsymbol{\theta^{\prime}}} \Big\Vert \pi(\boldsymbol{\theta})\right)\right] \\
& = \mathbb{E}_{\pi(\boldsymbol{y} \mid \boldsymbol{\xi})}\left[\mathrm{KL}\left(\frac{\pi(\boldsymbol{\theta})}{\int_{\mathcal{T}} \pi(\boldsymbol{\theta^{\prime}}) d\boldsymbol{\theta^{\prime}}} \Big\Vert \pi(\boldsymbol{\theta})\right)\right] \\
& = \mathbb{E}_{\pi(\boldsymbol{y} \mid \boldsymbol{\xi})}\left[\mathrm{KL}\left(\pi(\boldsymbol{\theta}) \Vert \pi(\boldsymbol{\theta})\right)\right] \\
& = 0,
\end{align*}
where $\pi(\boldsymbol{y} \mid \boldsymbol{\xi}) = \frac{\widetilde{\pi}(\boldsymbol{y} \mid \boldsymbol{\xi})}{\int_{\mathcal{Y}} \widetilde{\pi}(\boldsymbol{y} \mid \boldsymbol{\xi}) d\boldsymbol{y}}$.
\end{proof}
The same result can be found under the equivalent expressions of the Gibbs EIG presented in the paper, meaning that all designs are equally uninformative when $\omega = 0$.

This reflects the effect of the learning rate on the Gibbs posterior: A smaller learning rate leads to the posterior being closer to the prior (and exactly the prior when $\omega = 0$). Setting $\omega < 1$ has the effect of downweighting the Gibbs EIG. This effect makes intuitive sense when taken in the GBI context; a smaller learning rate results in less information being learnt through the Gibbs posterior update. The Gibbs EIG thus selects designs that maximise information according to the learning rate used in the Gibbs posterior update. 

See \Appref{visgibbseig} for examples of how the learning rate affects the Gibbs EIG. 

A learning rate too small results in slow learning for the (usually small) experimental budget given. A learning rate too high can arguably make learning under misspecification too fast; it may be of interest to slowly deviate from our prior as we learn from incoming data that may not fit well with our assumed statistical model. 

In all of our experiments, the learning rate is fixed over the experimental horizon. For the weighted score matching loss function, robustness primarily comes from the choice of weighting function $r$ in addition to $\omega$. \citet{robustgpr} also explain that optimising both $\omega$ and $c$ in $r_{\mathrm{IMQ}}$ is numerically unstable due to near non-identifiability. In the context of Gaussian process (GP) regression, \citet{robustgpr, laplante2025robustconjugatespatiotemporalgaussian} fix $\omega$ to a value that depends on the Gaussian variance. This allows for the recovery of the standard GP as $c \to \infty$ in $r_{\mathrm{IMQ}}$. 
Although we could apply a similar method, we may not always be using Gaussians or have a known variance in our model. 

In general, $\omega$ would ideally depend on how misled one's initial beliefs (prior) are, or the degree of misspecification in the statistical model.
Of course, this knowledge is unavailable in practice. We explore the impacts of several learning rate choices for each experimental design problem (bar the location finding problem; \Appref{locfindexpdetails}) in \Appref{additionalexperappendix}. 
In summary, poorly specified priors (and well-specified model scenarios) may demand larger learning rates in small-data regimes for strong performance. One needs to first find a strong enough prior before pursuing robust inference against model misspecification (or do both simultaneously sensibly). In contrast, when the true parameters are samples from the prior (the initial prior is already strong), smaller learning rates can achieve further robustness than what the loss function alone can deliver. 

Another important consideration needs to be made in relation to computing the Gibbs EIG accurately. Importance sampling as in \Cref{theorem1} relies on a well-chosen proposal distribution. If we use our model as this proposal, setting $\omega < 1$ will start to result in lower effective sample size (see \Appref{impsampdiagnostics}) --- i.e., the model becomes less of a suitable proposal distribution. The learning rate needs to be set with this property in mind; low learning rates can result in robust inference, but, when paired with the Gibbs EIG, potentially less accurate information gain estimates and resulting optimal design sequences.

\subsection{Loss Function}
\label{lossfunchelp}

The choice of loss function is important for both choosing optimal designs and in performing parameter inference. Loss functions have largely only been investigated for their abilities in making robust inferences, rather than their possible impact on decision-making as through the Gibbs EIG. \Appref{visgibbseig} illustrates how the Gibbs EIG varies according to each loss function considered in this work.

The flexibility of GBOED in allowing one to use a user-specified loss function is a key advantage of our framework. The reason the loss is user-specified is to allow the user to take into account multiple necessary criteria, such as whether the model has intractable normalising constants or is differentiable (see examples of loss functions below). For example, the constant elasticity of substitution problem (\Appref{cesexpdetails}) uses a non-differentiable likelihood, meaning that score matching cannot be used. The guidance in this subsection helps users make an informed choice, presenting possible loss functions one could consider when deploying GBOED, and includes the reasoning behind our particular choices. \Tableref{losscomparison} summarises the guidance for five possible loss functions that are widely used when performing GBI.

\begin{table}[ht!]
\caption{Properties of several possible loss functions that can be employed in GBOED.}
\centering
\renewcommand{\arraystretch}{1.05}
\setlength{\tabcolsep}{4pt}
\begin{adjustbox}{max width=\textwidth}
    \centering
    \begin{tabular}{lccccc}
        \toprule
        \multirow{2}{*}{Method}
        & Additive and
        & Accommodates
        & Doesn't require 
        & Can exploit \\
        & invariant
        & non-differentiable likelihoods
        & normalising constants
        & conjugacy \\
        \midrule
        Power Likelihoods 
            & \textcolor{blindgreen}{\ding{51}} 
            & \textcolor{blindgreen}{\ding{51}}
            & \textcolor{blindred}{\ding{55}} 
            & \textcolor{blindred}{\ding{55}}
            & \\
        (Weighted) Score Matching 
            & \textcolor{blindgreen}{\ding{51}}
            & \textcolor{blindred}{\ding{55}}
            & \textcolor{blindgreen}{\ding{51}}
            & \textcolor{blindgreen}{\ding{51}}
            & \\
        $\beta$- and $\gamma$-Divergences 
            & \textcolor{blindgreen}{\ding{51}}
            & \textcolor{blindgreen}{\ding{51}}
            & \textcolor{blindred}{\ding{55}}
            & \textcolor{blindred}{\ding{55}}
            & \\
        MMD-Bayes 
            & \textcolor{blindred}{\ding{55}}
            & \textcolor{blindgreen}{\ding{51}}
            & \textcolor{blindgreen}{\ding{51}}
            & \textcolor{blindred}{\ding{55}}
            & \\
        KSD-Bayes 
            & \textcolor{blindred}{\ding{55}}
            & \textcolor{blindgreen}{\ding{51}}
            & \textcolor{blindgreen}{\ding{51}}
            & \textcolor{blindgreen}{\ding{51}}
            & \\
        \bottomrule
    \end{tabular}
\end{adjustbox}
\label{losscomparison}
\end{table}

\subsubsection{Power Likelihoods}

Power likelihoods \citep{powerprior, grunwald, holmes2017assigning, mclatchie2025predictiveperformancepowerposteriors} use the negative log-likelihood loss, with the only modification from Bayesian inference being the choice of learning rate $\omega \neq 1$. Power likelihoods are attractive when misspecification is known to be mild, and the learning rate can be chosen according to \Appref{learningrateselection}. Power likelihoods are additive (enabling invariance to additive constants) and are invariant to changes in the coordinates of the dataset. Therefore, the advantages of power likelihoods include computational feasibility and ease of implementation.

\subsubsection{(Weighted) Score Matching}

Score matching \citep{hyvarinen, barp2019minimum, Matsubaradiscrete, altamirano23a} minimises the Fisher divergence between the score function of the statistical model and the score function of the true DGP. Like power likelihoods, score matching is additive and invariant to changes in the coordinates of the dataset. There is flexibility in choosing $r$ for robustness in weighted score matching; using $r = r_{\mathrm{IMQ}}$ allows one to tackle outliers. By taking score functions, we avoid any intractable normalising constants attached to the model. GBOED with score matching can also be made more computationally tractable for statistical models in the exponential family by leveraging conjugate priors, following, for example, \citet{altamirano23a}. This results in closed-form posteriors and can significantly improve the speed of GBOED. A drawback of score functions is that we require the statistical model to be differentiable. Not every model is differentiable, meaning that score matching cannot be used in certain real-world problems.

\paragraph{Tractable form of score matching} \label{scorematchingappendix} The score function corresponding to an outcome distribution is the gradient of the logarithm of the corresponding density with respect to outcomes $\boldsymbol{y}$.
For an outcome distribution with density $p$, we write the corresponding score function as $\nabla_{\boldsymbol{y}} \log{p}$. 

Since the score matching loss function requires $\pdata$ to be known, we desire a way to avoid needing to know this explicitly. Thankfully, under certain regularity conditions \citep{liu2022estimating, altamirano23a, robustgpr}, we can use integration by parts to write the score matching loss function as
\begin{align*}
\ell_{\boldsymbol{\theta}}(\boldsymbol{\xi}, \boldsymbol{y}) & = \lVert r\left(\nabla_{\boldsymbol{y}}\log \pmodel(\boldsymbol{\xi}, \boldsymbol{y}) - \nabla_{\boldsymbol{y}}\log \pdata(\boldsymbol{\xi}, \boldsymbol{y}) \right)\rVert^{2}_{2} \\
& = \left(r\nabla_{\boldsymbol{y}}\log \pmodel(\boldsymbol{\xi}, \boldsymbol{y})\right)^{2} + 2\nabla_{\boldsymbol{y}}\left(r^{2}\nabla_{\boldsymbol{y}}\log \pmodel(\boldsymbol{\xi}, \boldsymbol{y})\right).
\end{align*}

With $\pmodel(\boldsymbol{\xi}, \boldsymbol{y}) = p(\boldsymbol{y} \mid \boldsymbol{\theta}, \boldsymbol{\xi})$: $$\ell_{\boldsymbol{\theta}}(\boldsymbol{\xi}, \boldsymbol{y}) = \left(r\nabla_{\boldsymbol{y}}\log p(\boldsymbol{y} \mid \boldsymbol{\theta}, \boldsymbol{\xi})\right)^{2} + 2\nabla_{\boldsymbol{y}}\left(r^{2}\nabla_{\boldsymbol{y}}\log p(\boldsymbol{y} \mid \boldsymbol{\theta}, \boldsymbol{\xi})\right).$$

\paragraph{Weighted score matching} The effectiveness of the IMQ kernel $r_{\mathrm{IMQ}}$ in dealing with outliers in data heavily depends on the choice of $\gamma$ and $c$. \citet{robustgpr} propose to use the prior mean at each design $\boldsymbol{\xi}$ as $\gamma$ and a design independent value of $c$ based on how many outliers one expects to see in the data. In practice, suitable specification of $\gamma$ and $c$ is challenging: One typically would not know how often they would see an outlier, nor the prior mean. \citet{laplante2025robustconjugatespatiotemporalgaussian} highlight the disadvantages of using the \citet{robustgpr} approach, and instead suggest to use the posterior predictive mean and standard deviation for $\gamma$ and $c$ respectively. This is adaptive and particularly allows one to tackle the challenge of determining how to set the centre of the data and how quickly to downweight observations. Our approach to choosing these hyperparameters, discussed in \Secref{ourimqmethod}, adopts using the posterior predictive mean as $\gamma$, and adaptively modifies $c$ through exponential decay. 

\subsubsection{Beta- and Gamma-Divergences}

The $\beta$- and $\gamma$-divergences are modifications of the negative log-likelihood that incorporate robustness through hyperparameters $\beta$ and $\gamma$, respectively. While the two loss functions include a possibly expensive-to-compute integral, this is available in closed-form for exponential families. An advantage shared by the two loss functions is that the divergences are additive. A key disadvantage, however, is that selecting the hyperparameters $\beta$ and $\gamma$ is difficult \citep{knoblauch}. While one can assign these hyperparameters with value $1 + \epsilon$ for some small $\epsilon$, choosing $\epsilon$ still remains a challenge.

\subsubsection{Maximum Mean Discrepancy}

The maximum mean discrepancy (MMD) loss function, also called MMD-Bayes, computes the MMD \citep{mmd} between the distribution of the statistical model and the empirical distribution of the observed data \citep{cherief_abdellatif}. MMD-Bayes depends heavily on the choice of kernel used to compute the MMD, which can be difficult to choose depending on the setting. It is also non-additive, meaning that adding constants to the computed loss can change the overall posterior learnt. MMD-Bayes can, however, be attractive when it is only possible to sample from a model, which is the case in many real-world problems. Most loss functions require the likelihood to be analytically available, even if it can handle any attached normalising constants. The loss functions named above are more computationally feasible than MMD-Bayes due to additivity, which can be a problem when requiring experiments to be conducted and inferences to be made quickly. Lastly, there is no conjugacy available when using the MMD loss function, and so one often needs to resort to variational inference in complex problems.

\subsubsection{Kernel Stein Discrepancy}

The kernel Stein discrepancy (KSD) loss function, or KSD-Bayes, computes the KSD \citep{gorham2015measuring} between the distribution of the statistical model and that of the observed data \citep{matsubara}. Like MMD-Bayes, KSD-Bayes relies on a well-chosen kernel for robust inference. The standard version of KSD-Bayes is non-invariant to changes of coordinates in the dataset, but it is possible to approximate invariance to affine dataset transformations using a data-adaptive kernel. As with score matching, KSD-Bayes can avoid computing intractable normalising constants attached to likelihood functions. Conjugacy is also available for exponential family models, which can improve computational costs at least over MMD-Bayes. An important downside of KSD-Bayes is its loss of efficiency in higher dimensions, meaning that using KSD-Bayes in higher dimensions over other loss functions may not be sensible.

\subsection{Exponential Decay for IMQ Parameters}
\label{expdecayguidance}

The issue is that $c$ falls in value too quickly through being based on the posterior predictive variance. 
One way of dealing with this is to choose $c$ according to a different adaptive method, without relying on the posterior predictive distribution. We select $c$ using exponential decay: We initialise $c$ at a pre-specified value, and decrease it over the course of experimentation according to a pre-determined schedule. 
More specifically, our exponential decay method computes $c$ according to $$c(i) = q_{1} \exp(-b(i - 1)) + q_{2}$$ for experiment $i \in \{1, \ldots, T\}$. $b > 0$ is a rate parameter to be chosen, and $q_{1}, q_{2} > 0$ are parameters controlling the starting and ending values (assuming convergence) of $c$ during experimentation ($q_{1} + q_{2}$ is the value of $c$ for the first experiment). 

Small values of $b$ allow for small decreases in $c$ per experiment, while large $b$ values cause $c$ to decay and thus converge to $q_{2}$ quickly (see \Figref{bexponentialdecay} in \Appref{exponentialdecappendrate}). 
Assuming a fixed value of $b$, higher values of $q_{1}$ affect how large the value of $c$ is at the beginning of experimentation. 
The value $q_{2}$ determines the lower bound to which $c$ approaches as experimentation progresses. When $c$ is small, observations are more likely to be treated as outliers --- i.e., the loss function focuses more on robustness in the later stages of experimentation. So, setting a small value for $q_{2}$ allows one to gradually increase robustness and have more confidence in the learnt $\gamma$.
The exponential decay method presented is design independent, although one could introduce additional criteria for setting the hyperparameters that depend on the design.

\paragraph{Guidance for selecting exponential decay parameters} The decay rate $b$ is studied empirically in \Appref{exponentialdecappendrate}. Broadly, $q_1, q_2, b$ control the trade-off between learning a reliable centring function $\gamma$ and introducing robustness. When the prior is already informative and close to the posterior (the true parameters are samples from the prior), smaller values (e.g., $q_1 \approx 1, q_2 \approx 0.2, b \approx 0.04$) encourage early robustness, since $\gamma$ is accurate from the outset. The \citet{laplante2025robustconjugatespatiotemporalgaussian} method will perform strongly in this scenario, but exponential decay can outperform this due to the additional tunable parameters.

When the prior and posterior are far apart (the prior is poor and so the true data-generating parameter values are not well-represented by the prior), a different regime is preferable. In this setting, one can exploit knowledge of the experimental budget $T$ to tune the decay schedule (see \Appref{exponentialdecappendrate}). Empirically, we find that larger initial values (e.g., $q_1 \approx 9, q_2 \approx 1$) combined with a relatively small decay rate ($b \approx 0.06$) perform well. A large $q_1$ ensures sufficient time is spent learning the centring function before introducing robustness, while $b$ controls how quickly this transition occurs. This flexibility allows the method to adapt to different regimes: even without precise prior knowledge, one can choose parameters that prioritise learning early on and gradually shift towards robustness, leading to stable and effective design decisions in practice.

\paragraph{Alternatives to exponential decay} One could instead use another decay method for selecting $c$, such as linear decay.
We advise selecting a decay method that ensures that $c$ does not fall in value too quickly, but not too slowly either in order to allow for robust inference.

\subsubsection{Effect of Rate Selection on the Exponential Decay Method}
\label{exponentialdecappendrate}

A robust method for selecting $c$ should enable enough time (i.e., be large enough) to learn the centring function $\gamma$ in early experiments, and enable enough time to be robust to the acquired data (i.e., be small enough to discern between outliers). Finding the right balance is difficult, however, as we find throughout this paper, our exponential decay method is a competitive method for selecting $c$.

The effect of choosing a certain rate $b$ in our exponential decay method on the final value of $c$ in each experimental design problem can be viewed in \Figref{bexponentialdecay}. The figure displays different plots for the $q_{1}, q_{2}$ combinations used in our experiments. 

\begin{figure}[ht!]
    \centering

    \begin{subfigure}{0.32\textwidth}
        \centering
        \includegraphics[width=\linewidth]{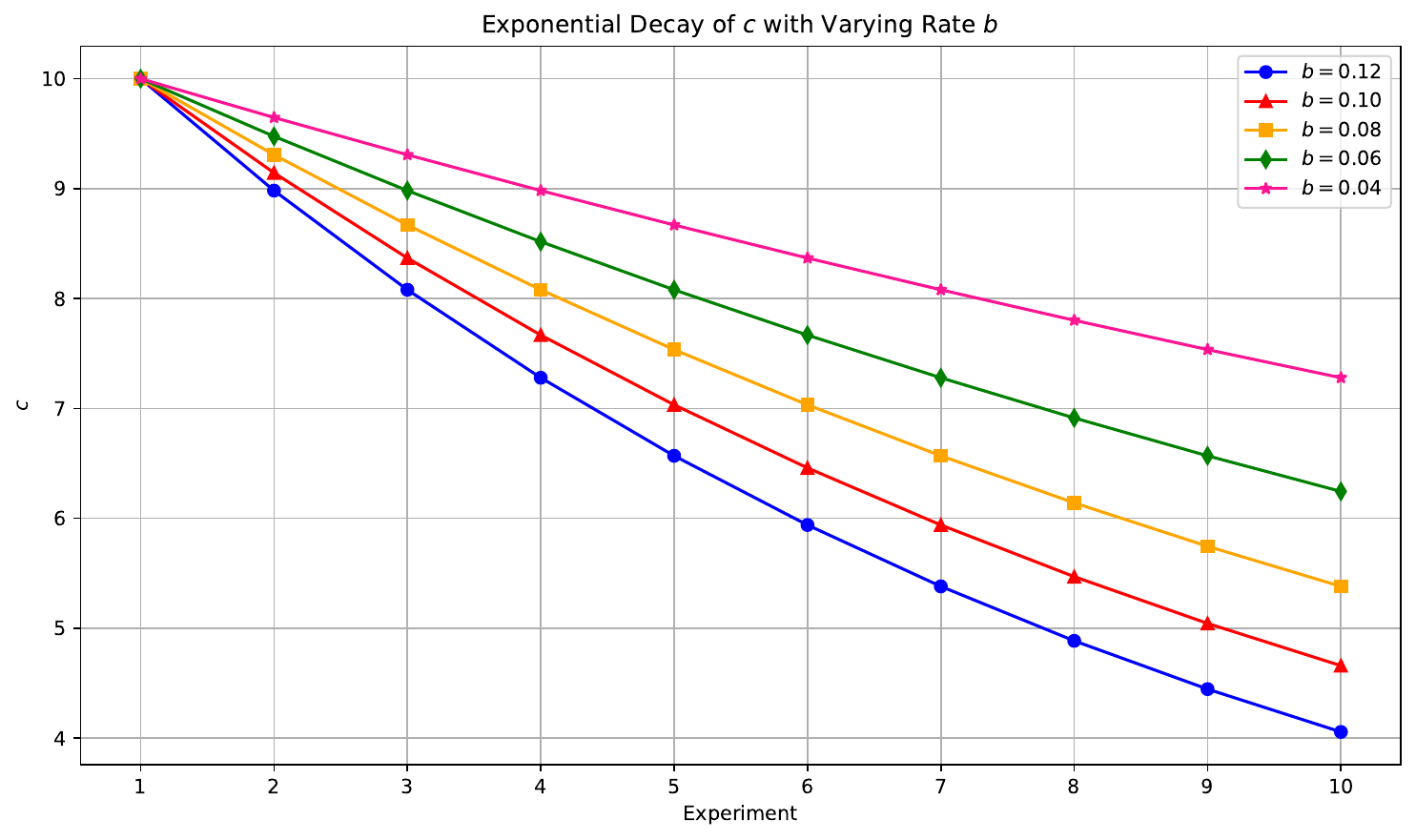}
        \caption{Linear Regression; $q_{1}=9,\; q_{2}=1$}
    \end{subfigure}
    \hfill
    \begin{subfigure}{0.32\textwidth}
        \centering
        \includegraphics[width=\linewidth]{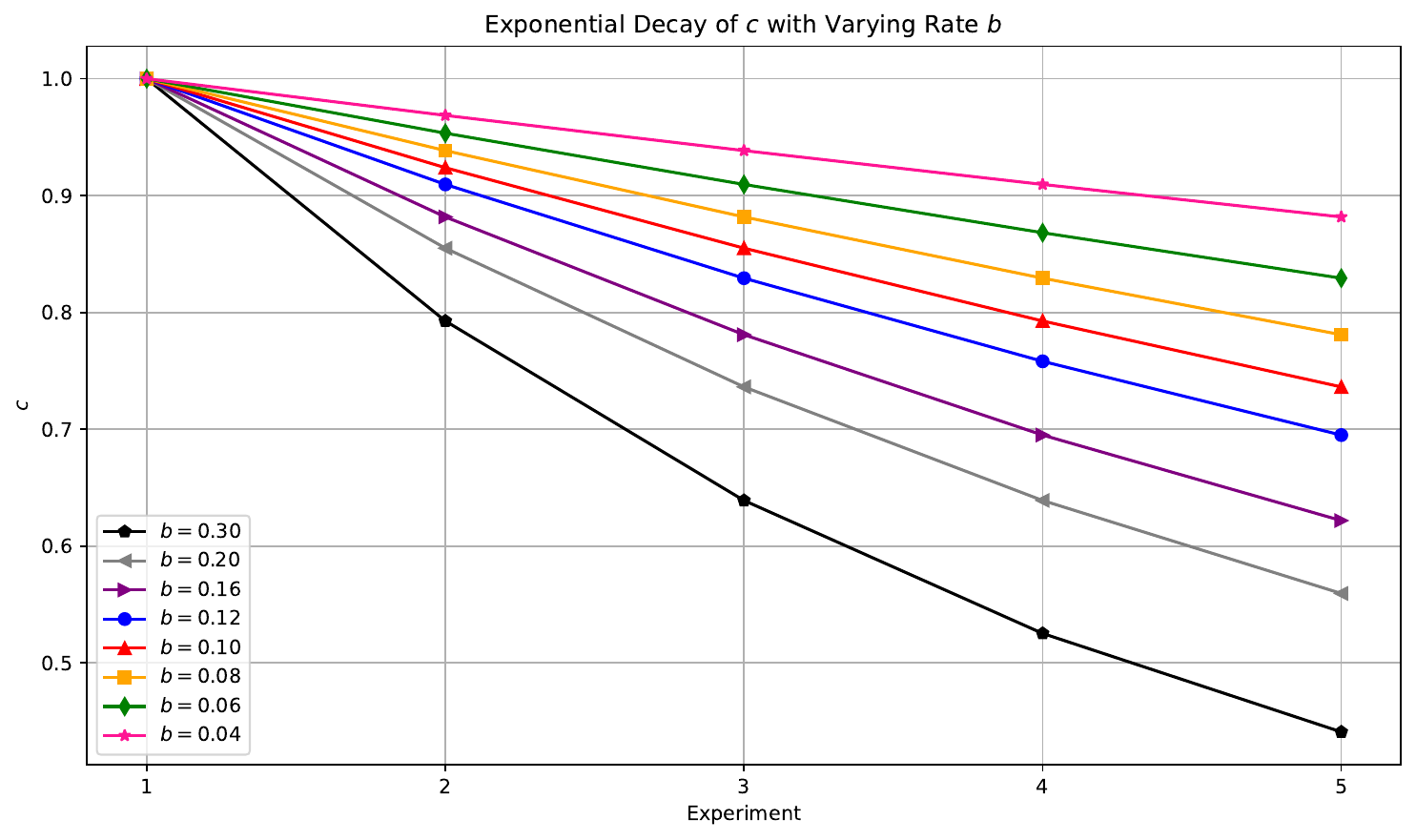}
        \caption{PK; $q_{1}=0.8,\; q_{2}=0.2$}
    \end{subfigure}
    \hfill
    \begin{subfigure}{0.32\textwidth}
        \centering
        \includegraphics[width=\linewidth]{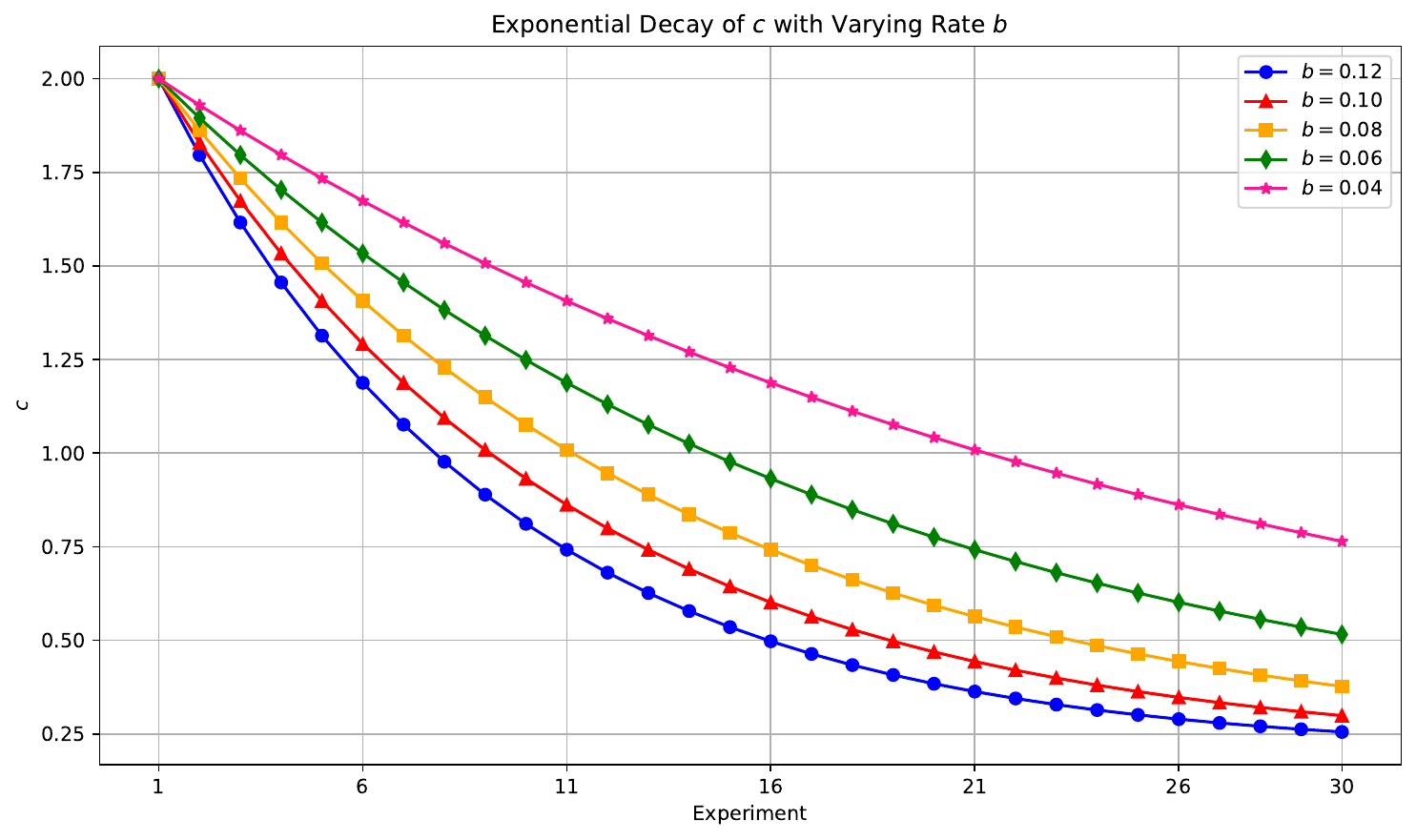}
        \caption{Location Finding; $q_{1}=1.8,\; q_{2}=0.2$}
    \end{subfigure}

    \caption{Comparison of different values of the rate $b$ in the exponential decay method for selecting $c$ in $r_{\mathrm{IMQ}}$ under different choices of $q_1$ and $q_2$.}
    \label{bexponentialdecay}
\end{figure}

Clearly, smaller values of $b$ result in smaller decreases of $c$ per experiment. This can lead to spending more time learning a sensible estimate of the centring function $\gamma$, before more rapid decreases in $c$ lead to substantially more robustness. A larger value of $b$, such as $b = 0.12$, can lead to $c$ becoming smaller faster, using fewer experiments to learn a good estimate of $\gamma$. 

\paragraph{Linear regression} We can view in practice the effect of using different rates $b$ from \Appref{exponentialdecappendrate} on the final predictive performance in the linear regression problem. The results from varying $b$ can be found in \Tableref{differentbregression}. It would appear that very small decreases in $c$ during experimentation are better than larger ones. $b = 0.04$ is generally much better than $b = 0.12$. Higher values of $b$ would behave more similarly to the \citet{laplante2025robustconjugatespatiotemporalgaussian} method, which we know from \Tableref{bayesregresultsfull} performs very poorly in our regression setting, in which the initial prior does not place high density on the true data-generating parameter values. One could also investigate whether adjusting $q_{1}$ and $q_{2}$ in the exponential decay method would improve performance.

\begin{table}[ht!]
\caption{Comparison of different rates $b$ from the exponential decay method under both well-specified and misspecified scenarios in the linear regression problem, with $\omega = 1$ and $(q_{1}, q_{2}) = (9, 1)$. RMSE, MMD, and NLL are recorded with mean (and standard error) over 3 different models, each replicated under 30 random seeds $(3 \times 30 = 90 \text{ replications in total})$.}
\centering
\begin{tabular}{lccc}
\toprule
\textbf{Rate} & \textbf{RMSE} & \textbf{MMD} & \textbf{NLL} \\
\midrule
\multicolumn{4}{c}{Well-Specified} \\[0.2em]
$b = 0.04$  & \textbf{1.6024\,(0.0402)} & \textbf{0.0768\,(0.0079)} & \textbf{1.6288\,(0.0371)} \\
$b = 0.06$  & $1.6086\,(0.0415)$ & $0.0793\,(0.0086)$ & $1.6377\,(0.0393)$ \\
$b = 0.08$  & $1.6191\,(0.0428)$ & $0.0837\,(0.0090)$ & $1.6517\,(0.0410)$ \\
$b = 0.10$  & $1.6418\,(0.0475)$ & $0.0931\,(0.0112)$ & $1.6891\,(0.0503)$ \\
$b = 0.12$  & $1.6954\,(0.0632)$ & $0.1114\,(0.0170)$ & $1.8088\,(0.0997)$ \\
\midrule
\multicolumn{4}{c}{Asymmetric Outliers} \\[0.2em]
$b = 0.04$  & \textbf{2.4716\,(0.1016) }& \textbf{0.4803\,(0.0314)} & \textbf{3.4828\,(0.2064)} \\
$b = 0.06$  & $2.4815\,(0.1072)$ & $0.4766\,(0.0326)$ & $3.5221\,(0.2244)$ \\
$b = 0.08$  & $2.5163\,(0.1152)$ & $0.4806\,(0.0342)$ & $3.6306\,(0.2473)$ \\
$b = 0.10$  & $2.5908\,(0.1310)$ & $0.4911\,(0.0367)$ & $3.9014\,(0.3102)$ \\
$b = 0.12$  & $2.6481\,(0.1421)$ & $0.4997\,(0.0384)$ & $4.0867\,(0.3496)$ \\
\midrule
\multicolumn{4}{c}{Misspecified Error Distribution} \\[0.2em]
$b = 0.04$  & \textbf{1.9652\,(0.0487)} & \textbf{0.1125\,(0.0115)} & \textbf{2.2494\,(0.0534)} \\
$b = 0.06$  & $1.9708\,(0.0504)$ & $0.1141\,(0.0120)$ & $2.2581\,(0.0567)$ \\
$b = 0.08$  & $1.9841\,(0.0537)$ & $0.1185\,(0.0132)$ & $2.2861\,(0.0663)$ \\
$b = 0.10$  & $1.9995\,(0.0604)$ & $0.1218\,(0.0147)$ & $2.3321\,(0.0998)$ \\
$b = 0.12$  & $2.1027\,(0.0947)$ & $0.1480\,(0.0210)$ & $2.6729\,(0.2668)$ \\
\bottomrule
\end{tabular}
\label{differentbregression}
\end{table}

\section{Experiment Details}
\label{experimentdetails}

\subsection{Source Code}

A GitHub repository containing the code used to generate the results presented in this paper can be found at \url{https://github.com/yasirbarlas/GBOED}. The code is written in Python and relies heavily on Pyro \citep{bingham2018pyro} and PyTorch \citep{paszke2019pytorch}. Our code is based on existing code from \citet{foster2020unified}, \citet{ivanova2021implicit}, \citet{blau2022optimizing}, and \citet{barlasrl} where relevant. The code for the R-IDeA baseline is from \citet{tang2026representativeinformativedeamplifyingrequirements}. Their recommended hyperparameters were used in our experiments.

\subsection{Hardware}

All experiments were run on a High-Performance Computer using the Slurm Workload Manager. An AMD EPYC 9634 84-Core CPU was used for the experiments. 8 CPU cores and 50 GB of RAM were assigned to each experiment. No GPUs were used.

\subsection{Computing the Gibbs Posterior}

We utilised both generalised variational inference (GVI; \citealp{knoblauch}) and Markov Chain Monte Carlo (MCMC) to learn (Gibbs) posteriors. For the former, we maximise the generalised evidence lower bound (ELBO) 
\begin{align*}
\mathrm{ELBO} & \equiv \mathbb{E}_{q_{\phi}(\boldsymbol{z})} \left[ \log \left(\exp\left(-\omega\ell_{\psi, \boldsymbol{z}}(\boldsymbol{y}) \right)\right) + \log p_{\psi}(\boldsymbol{z}) - \log q_{\phi}(\boldsymbol{z}) \right] \\
& \equiv \mathbb{E}_{q_{\phi}(\boldsymbol{z})} \left[ -\omega\ell_{\psi}(\boldsymbol{y}) + \log p_{\psi}(\boldsymbol{z}) - \log q_{\phi}(\boldsymbol{z}) \right],
\end{align*}
for a prior $p_{\psi}(\boldsymbol{z})$ and guide $q_{\phi}(\boldsymbol{z})$ with variational parameters $\psi$ and $\phi$ respectively. GVI was used in the linear regression, pharmacokinetics, and constant elasticity of substitution problems. We used stochastic (mean-field) variational inference through Pyro \citep{bingham2018pyro} with the ELBO objective above. We took 10000 steps with the Adam optimiser \citep{kingma2014adammethodstochasticoptimization}, using a learning rate of 0.005. Both Bayesian inference methods and GBI methods make use of the same variational family and parameters for a fair comparison. 

\paragraph{Markov Chain Monte Carlo} \label{mcmcvssvi}
The location finding problem is notorious for being multimodal, and thus difficult to perform inference in \citep{ivanova2021implicit, ivanova-phd}. Standard (mean-field) variational inference, whilst fast, is poor in multimodal settings. We therefore require a better posterior approximation method, leading us to MCMC. MCMC is unfortunately expensive, but provides much better inferences than a poor mean-field variational approximator in the location finding setting. We performed MCMC with 10000 burn-in samples, later gathering 100000 posterior samples. Our MCMC implementation for learning Gibbs posteriors follows \citet{Matsubaradiscrete}. 

\Figref{mcmcvsvariationplot} displays the difference between using variational inference and MCMC to compute posteriors in the location finding experimental design problem. Clearly, variational inference is unable to perform well in the multimodal location finding problem. MCMC offers far better performance, which justifies use of MCMC in our experiments.

\begin{figure}[ht!]
    \centering
    \includegraphics[width=100mm]{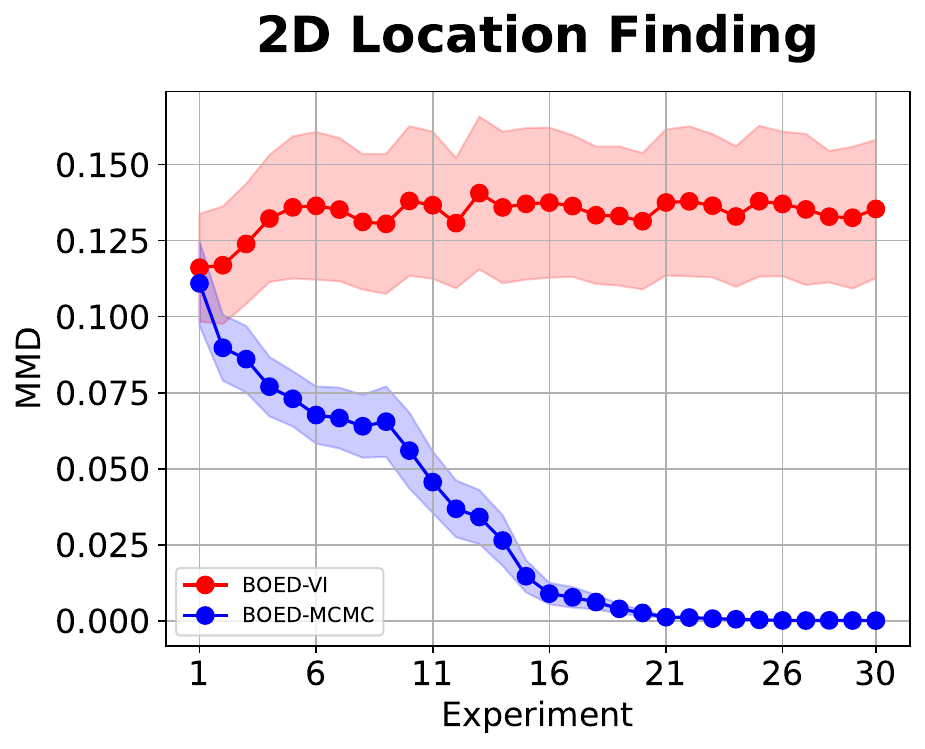}
    \caption{Comparison between variational inference (VI) and MCMC to compute posteriors when using BOED on the (well-specified) 2D location finding problem. VI fails to reduce MMD, whereas MCMC reduces it towards zero.}
    \label{mcmcvsvariationplot}
\end{figure}

\subsection{Scenarios of Misspecification}
\label{scenariomisspec}

The forms of misspecification we study are related to contamination and outliers, and misspecification in the error distribution, rather than misspecification in the functional form of the statistical model. We describe here the scenarios, with specific detail for each experimental design problem below. 

\paragraph{Asymmetric outliers} Similarly to \citet{robustgpr}, each observation $\tilde{y}_i$ generated by the real-world DGP has an $x\%$ chance of being contaminated. The observations to be contaminated have a value $z_i \sim \mathcal{U}(a, b)$ subtracted from them, where $a$ and $b$ depend on the standard deviation of the assumed model.
The dataset observed by the model then consists of observations $y_i$, where for approximately $1 - x\%$ of the observations $y_i = \tilde{y}_i$, and for approximately $x\%$ of the observations $y_i = \tilde{y}_i - z_i$.
The exact values of $x$, $a$ and $b$ are detailed in the subsections below corresponding to each experimental design problem. 

\paragraph{Misspecified Error Distribution} Each statistical model assumes a certain distribution for the errors, such as a normal distribution. When the error distribution is misspecified, the distribution itself or parametrisation of the original distribution differs from the data-generating process (e.g., the variance of the noise distribution is incorrectly specified). 

\subsection{Bayesian Linear Regression}
\label{bayesregresappendix}

We are interested in learning the relationship between covariates $\boldsymbol{\xi}$ and a dependent variable $\boldsymbol{y}$. To do so, we assume a linear regression model and seek to learn $K$ coefficients $\boldsymbol{\theta} = (\beta_{0}, \ldots, \beta_{K-1})$. We restrict our covariates in $\boldsymbol{\xi} \in [-4, 4]$. An experimenter conducts $T = 10$ experiments, meaning that we should learn $\boldsymbol{\theta}$ by selecting designs according to our choice of utility function.

For each coefficient $\beta_{i} \in \mathbb{R}$, we use a standard normal prior given by $$\beta_{i} \sim \mathcal{N}(0, 1).$$

For a given design (matrix) $\boldsymbol{\xi}$ and known standard deviation $\sigma$, the likelihood function is given by, for a single dependent variable $\boldsymbol{y}_{i}$, $$\boldsymbol{y}_{i} \mid \boldsymbol{\theta}, \boldsymbol{\xi}_{i} \sim \mathcal{N}(\boldsymbol{\xi}_{i}^{\intercal}\boldsymbol{\theta}, \sigma^{2}).$$

We assess the performance of each method as an average over 3 different true models:
$$\boldsymbol{y}_{i} = 10 -7\boldsymbol{\xi}_{i}, \boldsymbol{y}_{i} = -3 +8\boldsymbol{\xi}_{i}, \boldsymbol{y}_{i} = 9 + 9\boldsymbol{\xi}_{i},$$
with standard deviations $\sigma = 1.2$, $\sigma = 0.8$, and $\sigma = 1$ respectively.

\paragraph{Under misspecification} For Asymmetric Outliers, observations have a 30\% chance of being contaminated by $z \sim \mathcal{U}(3\sigma, 9\sigma)$. We use 30\% to limit the damage prior misspecification can cause (having outliers occur too often can fail to learn the true parameter values effectively). The Misspecified Error Distribution scenario considers that the true errors actually come from a Laplace distribution with the same location and scale parameters, rather than a normal distribution. Performance is averaged over the three different true models above $(3 \text{ models} \times 30 \text{ seeds} = 90 \text{ replications in total})$.

\paragraph{Computational details} For computing the (Gibbs) EIG, we set $N = 10000, M = 100$ in our NMC estimator. The design space $[-4, 4]$ is discretised into 100 possible designs that an experimenter can select. The EIG is computed for each design, and the design with the greatest EIG is selected for the real experiment to be performed. $\omega = 1$ in the main paper unless stated otherwise ($\omega = 0.8$ for power likelihoods to avoid reproducing BOED). $q_{1} = 9$ and $q_{2} = 1$ when using our exponential decay method.

\subsection{Pharmacokinetics}
\label{pharmacokineticsappendix}

Pharmacokinetics is concerned with finding optimal blood sampling times to accurately characterise drug concentration–time profiles whilst minimising any costs. This problem has been considered in many experimental design studies \citep{RYAN201445, kleinegesse2020bayesian, zhang2021scalable, ivanova2021implicit}. A pharmacokinetic (PK) model built by \citet{RYAN201445} is used to simulate drug concentration at a particular time. The PK model is controlled by three parameters, $\boldsymbol{\theta} = (k_{\alpha}, k_{e}, V)$, where $V$ is the volume of distribution, $k_{\alpha}$ is the absorption rate, and $k_{e}$ is the elimination rate. Our goal is to select a single blood sampling time $\boldsymbol{\xi} \in [0, 24]$ once a drug has been administered to a patient, and this process of blood sampling is done sequentially for $T = 5$ unique patients, presenting a very small data scenario. 

We set the following prior on $\boldsymbol{\theta}$, matching \citet{RYAN201445}, $$\log \boldsymbol{\theta} \sim \mathcal{N}\left(\begin{bmatrix}
           \log 1 \\
           \log 0.1 \\
           \log 20
\end{bmatrix}, \begin{bmatrix}
           0.05 & 0 & 0 \\
           0 & 0.05 & 0\\
           0 & 0 & 0.05 \\
         \end{bmatrix}\right).$$

With the constraint $k_{\alpha} > k_{e}$, the PK model is given by $$z(\boldsymbol{\theta}, \boldsymbol{\xi}) = \frac{D_{V}}{V} \cdot \frac{k_{\alpha}}{k_{\alpha} - k_e} 
\left[ \exp(-k_{e} \boldsymbol{\xi}) - \exp(-k_{\alpha} \boldsymbol{\xi}) \right], y(\boldsymbol{\theta}, \boldsymbol{\xi}) = z(\boldsymbol{\theta}, \boldsymbol{\xi})(1 + \epsilon) + \eta,$$
where $D_{V} = 400$, $\epsilon \sim \mathcal{N}(0, 0.01)$ is multiplicative noise, and $\eta \sim \mathcal{N}(0, 0.1)$ is additive noise. As both noise sources are normally distributed, the likelihood can otherwise be viewed as $$\boldsymbol{y} \mid \boldsymbol{\theta}, \boldsymbol{\xi} \sim \mathcal{N}(z(\boldsymbol{\theta}, \boldsymbol{\xi}), 0.01z(\boldsymbol{\theta}, \boldsymbol{\xi})^{2} + 0.1).$$

Following \citet{kleinegesse2020bayesian}, the true parameters for the model are $\boldsymbol{\theta} = (k_{\alpha}, k_{e}, V) = (1.5, 0.15, 15)$.

\paragraph{Under misspecification} We investigate the Asymmetric Outliers scenario with a 50\% chance of contamination. As the standard deviation is not constant for each design, we subtract $z \sim \mathcal{U}(1, 4)$ from each observation to be contaminated. We use 50\% due to the tight experimental budget $T = 5$ (to increase the chances that outliers exist in the dataset, for the sake of actually assessing robustness). Performance is averaged over a single true model $(1 \text{ model} \times 100 \text{ seeds} = 100 \text{ replications in total})$. We also investigate changes in the error distribution, where separately the additive noise is changed to $\eta \sim \mathcal{N}(0, 0.5)$, and the multiplicative noise is to be $\epsilon \sim \mathcal{N}(0, 0.1)$. Performance is then averaged over these two different true models $(2 \text{ models} \times 50 \text{ seeds} = 100 \text{ replications in total})$.

\paragraph{Computational details} For computing the (Gibbs) EIG, we set $N = 10000, M = 100$ in our NMC estimator. We used Bayesian optimisation (suitable for continuous design spaces) on the Gibbs EIG NMC estimator to select the optimal design in the design space for each experiment. A Mat\'ern52 kernel with lengthscale 20 and variance 10 was used alongside the GP-UCB1 algorithm \citep{srinivas2009gaussian}, with a UCB exploration parameter value of $\lambda = 6$. A total of 3000 function evaluations were made available of the EIG oracle (one optimal design out of 3000 randomly sampled designs within the design space is chosen for the real experiment). $\omega = 0.2$ in the main paper unless stated otherwise. $q_{1} = 0.8$ and $q_{2} = 0.2$ when using our exponential decay method.

\subsection{Constant Elasticity of Substitution}
\label{cesexpdetails}

In this problem, an agent indicates their preference of two baskets of items $\boldsymbol{x}, \boldsymbol{x'} \in [0, 100]^{3}$ on a sliding 0-1 scale. The constant elasticity of substitution (CES) model \citep{arrowchen} with parameters of interest $\boldsymbol{\theta} = (\rho, \boldsymbol{\alpha}, u)$, which all characterise the agent’s utility or preferences for the different items, is then used to measure the difference in utility of the baskets. We select the two baskets $\boldsymbol{\xi} = (\boldsymbol{x}, \boldsymbol{x'})$ in a way that allows us to infer the agent's preferences effectively. $T = 10$ experiments are conducted.

The CES model \citep{arrowchen} defines the utility $U(\boldsymbol{x})$ for a basket of goods $\boldsymbol{x}$ as $$U(\boldsymbol{x}) = \left( \sum_{i} x_i^{\rho} \alpha_{i} \right)^{\frac{1}{\rho}},$$ where $\rho$ and $\boldsymbol{\alpha}$ are latent variables defined with the prior distributions explained below. This utility function, which is a measure of satisfaction in economic terms, is then used in the likelihood function calculation.

We follow the lead of \citet{foster2020unified} and \citet{blau2022optimizing} by using the following priors for $(\rho, \boldsymbol{\alpha}, u)$,
\begin{align*}
\rho & \sim \mathrm{Beta}(1, 1) \\
\boldsymbol{\alpha} & \sim \mathrm{Dirichlet}([1, 1, 1]) \\
u & \sim \mathrm{LogNormal}(1, 3^{2}).
\end{align*}
The likelihood function is given by
\begin{align*}
\mu_{\eta} & = u \cdot (U(\boldsymbol{x}) - U(\boldsymbol{x'})) \\
\sigma_{\eta} & = \nu u \cdot (1 + \|\boldsymbol{x} - \boldsymbol{x'}\|) \\
\eta & \sim \mathcal{N}(\mu_{\eta}, \sigma_{\eta}^{2}) \\
y & = \text{clip}(s(\eta), \epsilon, 1 - \epsilon),
\end{align*}
where $\nu = 0.05$, $\epsilon = 2^{-22}$, and $s(x) = \frac{1}{1 + e^{-x}}$ is the sigmoid function. Notice that clipping makes the model not differentiable, and so score matching cannot normally be used.

The true parameters for the model are sampled from the prior at random (according to each random seed), as usually done in other work \citep{blau2022optimizing, ivanova2025stepdad}.

\paragraph{Under misspecification} We investigate the Asymmetric Outliers scenario with a 40\% chance of contamination, where observations to be contaminated are subtracted by $z \sim \mathcal{U}(10\nu, 30\nu)$. 40\% is used due to the well-specified prior and experimental budget $T = 10$. We ensure that values (including outliers) remain within $0 \leq \boldsymbol{y} \leq 1$. Performance is averaged over a single true model $(1 \text{ model} \times 50 \text{ seeds} = 50 \text{ replications in total})$. For the Misspecified Error Distribution scenario, $\nu$ is changed to be equal to a range of different values in $\{0.2, 0.5\}$. Performance is then averaged over these two different true models $(2 \text{ models} \times 25 \text{ seeds} = 50 \text{ replications in total})$.

\paragraph{Computational Details} For computing the (Gibbs) EIG, we set $N = 10000, M = 100$ in our NMC estimator. We used Bayesian optimisation (suitable for continuous design spaces) on the Gibbs EIG NMC estimator to select the optimal design in the design space for each experiment. A Mat\'ern52 kernel with lengthscale 20 and variance 8 was used alongside the GP-UCB1 algorithm \citep{srinivas2009gaussian}, with a UCB exploration parameter value of $\lambda = 2$. A total of 5000 function evaluations were made available of the EIG oracle. $\omega = 0.6$ in the main paper unless stated otherwise.

\subsection{Location Finding}
\label{locfindexpdetails}

Location finding is an experimental design problem that has been used to showcase the performance of various experimental design methods \citep{foster2021deep, ivanova2021implicit, blau2022optimizing}. There are $K$ objects on a $d$-dimensional space, and in this experiment the task is to identify their locations $\boldsymbol{\theta} = (\boldsymbol{\beta}_{1}, \ldots, \boldsymbol{\beta}_{K})$ based on the signals that the objects emit. We select designs $\boldsymbol{\xi}$, which are the coordinates chosen to observe the signal intensity, in an effort to learn the locations of the objects. Our spaces are restricted in $\boldsymbol{\xi} \in [-4, 4]^{d}$ to make the problem more tractable. An experimenter conducts $T = 30$ experiments, presenting a relatively large data scenario (compared to the other problems).

The total intensity at point $\boldsymbol{\xi}$ is the superposition of the individual intensities for each object, $$\mu(\boldsymbol{\theta}, \boldsymbol{\xi}) = b + \sum_{i = 1}^{K} \frac{\alpha}{m + \|\boldsymbol{\beta}_{i} - \boldsymbol{\xi}\|^{2}},$$ where $\alpha$ is a constant, $b > 0$ is a constant controlling the background signal, and $m > 0$ is a constant controlling the maximum signal. The total intensity is then used in the likelihood function calculation.

For an object $\boldsymbol{\beta}_{i} \in \mathbb{R}^{d}$, we use a standard normal prior given by $$\boldsymbol{\beta}_{i} \sim \mathcal{N}_{d}(\boldsymbol{0}, I),$$ where $\boldsymbol{0}$ is the mean vector, and $I$ is the covariance matrix, an identity matrix, both with dimension $d$.

The likelihood function is the logarithm of the total signal intensity $\mu(\boldsymbol{\theta}, \boldsymbol{\xi})$ with Gaussian noise $\sigma$. For a given design $\boldsymbol{\xi}$, the likelihood function is given by $$\log \boldsymbol{y} \mid \boldsymbol{\theta}, \boldsymbol{\xi} \sim \mathcal{N}(\log \mu(\boldsymbol{\theta}, \boldsymbol{\xi}), \sigma^{2}).$$

Our assumed hyperparameter choices are detailed in \Tableref{locfindparam}.

\begin{table}
\caption{Assumed location finding parameters. $\sigma$ can differ in the real-world environment depending on the misspecification present.}
    \centering
    \begin{tabular}{cc}
    \hline
    \textbf{Parameter} & \textbf{Value} \\
    \hline
    $K$ & 2 \\
    $\alpha$ & 1 \\
    $b$ & 0.1 \\
    $m$ & 0.0001 \\
    $\sigma$ & 0.5
    \end{tabular}
\label{locfindparam}
\end{table}

The true parameters for the model are sampled from the prior at random (according to each random seed), as usually done in other work \citep{foster2021deep, ivanova2021implicit, blau2022optimizing}.

\paragraph{Under misspecification} In the Asymmetric Outliers scenario, the chance of contamination is 30\%, and observations to be contaminated are subtracted by $z \sim \mathcal{U}(3\sigma, 7\sigma)$. 30\% is used to assess robustness across a long experimental horizon $T = 30$. In the Misspecified Error Distribution scenario, the true value of $\sigma = 0.8$ (as opposed to the model's assumption that $\sigma = 0.5$). In both cases, performance is averaged over a single true model $(1 \text{ model} \times 20 \text{ seeds} = 20 \text{ replications in total})$.

\paragraph{Computational details} For computing the (Gibbs) EIG, we set $N = 10000, M = 100$ in our NMC estimator. We used Bayesian optimisation on the Gibbs EIG NMC estimator to select the optimal design in the design space for each experiment. A Mat\'ern52 kernel with lengthscale 15 and variance 4 was used alongside the GP-UCB1 algorithm \citep{srinivas2009gaussian}, with a UCB exploration parameter value of $\lambda = 12$. A total of 5000 function evaluations were made available of the EIG oracle. $\omega = 0.4$ unless stated otherwise. $q_{1} = 1.8$ and $q_{2} = 0.2$ when using our exponential decay method.

\subsection{Deployment Times}
\label{deploymenttimes}

This subsection covers the average duration of experimentation (deployment times) experienced in our experiments.

\paragraph{Pharmacokinetics} \Tableref{deploymenttimes1} presents the deployment times in all $T = 5$ experiments for the pharmacokinetics problem, including both selecting optimal designs and computing posteriors. We find that the deployment time of our experiments is affected by the choice and implementation of the loss function. Our considered loss functions can be more expensive than using BOED out of the box in Pyro \citep{bingham2018pyro} due to requiring more calculations, such as the predictive mean and $r_{\mathrm{IMQ}}$ as a whole. 

\begin{table}[ht!]
\caption{Deployment times in seconds for the pharmacokinetics problem under the well-specified setting, with $\omega = 0.8$ if not stated. Mean (and standard error) over 100 replications.}
\centering
\begin{tabular}{lc}
\toprule
\textbf{Method} & \textbf{Deployment Time (s)} \\
\midrule
BOED & $344.3563\,(0.5001)$ \\
Power-Like $\omega = 0.8$ & $363.7655\,(1.4246)$ \\
\citet{laplante2025robustconjugatespatiotemporalgaussian} & $459.8116\,(20.3834)$ \\
Exp-Decay $b = 0.30$ & $431.6239\,(1.4035)$ \\
\textcolor{blindred}{Random} + \textcolor{blindblue}{Exp-Decay} & $328.0011\,(0.8702)$ \\
\textcolor{blindred}{BEIG} + \textcolor{blindblue}{Exp-Decay} & $379.6731\,(0.5902)$ \\
\bottomrule
\end{tabular}
\label{deploymenttimes1}
\end{table}

\paragraph{Constant elasticity of substitution} \Tableref{deploymenttimes2} presents the deployment times for the CES problem, including both selecting optimal designs and computing posteriors. As in the pharmacokinetics problem, the choice of loss function can affect the deployment time of our experiments. This is a 6-dimensional design problem, which causes the design phase of BOED and GBOED to be much more expensive than simply randomly selecting designs.

\begin{table}[ht!]
\caption{Deployment times in seconds for the constant elasticity of substitution problem under the Asymmetric Outliers setting, with $\omega = 0.6$ if not stated. Mean (and standard error) over 50 replications.}
\centering
\begin{tabular}{lc}
\toprule
\textbf{Method} & \textbf{Deployment Time (s)} \\
\midrule
BOED & $7387.5378\,(26.0662)$ \\
Power-Like $\omega = 0.6$ & $8046.8724\,(36.0040)$ \\
\textcolor{blindred}{Random} + \textcolor{blindblue}{Power-Like} & \textbf{4825.8437\,(33.2226)} \\
\textcolor{blindred}{BEIG} + \textcolor{blindblue}{Power-Like} & $8095.9237\,(57.6731)$ \\
\bottomrule
\end{tabular}
\label{deploymenttimes2}
\end{table}

\paragraph{Location finding} \label{locdeploymenttime}
\Tableref{deploymenttimes3} presents the average duration of experimentation (deployment time) in all $T = 30$ experiments for the 2D location finding problem. The deployment times are extremely long due to MCMC being favoured over variational inference (see \Appref{mcmcvssvi}). MCMC with 10000 burn-in samples and 100000 posterior samples was used for a strong approximation of the Gibbs posterior. This resulted in extremely expensive, but accurate, experiments.

\begin{table}[ht!]
\caption{Deployment times in seconds for the 2D location finding problem under the well-specified setting, with $\omega = 0.4$ if not stated. Mean (and standard error) over 20 replications.}
\centering
\begin{tabular}{lc}
\toprule
\textbf{Method} & \textbf{Deployment Time (s)} \\
\midrule
BOED & \textbf{42711.3084\,(2231.0811)} \\
Power-Like $\omega = 0.4$ & $53347.7212\,(3465.5066)$ \\
\citet{laplante2025robustconjugatespatiotemporalgaussian} & $68970.6800\,(4554.4166)$ \\
Exp-Decay $b = 0.06$ & $49611.9172\,(3111.6843)$ \\
\textcolor{blindred}{Random} + \textcolor{blindblue}{Exp-Decay} & $49371.2993\,(6091.7013)$ \\
\textcolor{blindred}{BEIG} + \textcolor{blindblue}{Exp-Decay} & $57400.5504\,(4024.3568)$ \\
\bottomrule
\end{tabular}
\label{deploymenttimes3}
\end{table}

\subsection{Metrics}
\label{metrics}

In this subsection we discuss our choices of predictive performance metrics and how exactly they are computed.

For consistency, we sample $N$ values from the posterior predictive distribution for each design. This means that we have $N \times D$ values in total for the total number of designs $D$. We also sample $N$ values from the true (outlier-free) DGP, meaning $N \times D$ values in total. The same predictive distribution generated is used when required by the \citet{laplante2025robustconjugatespatiotemporalgaussian} and exponential decay methods, changing every posterior update.

$N = 1000$ in all experiments, $D = 100$ for the Bayesian linear regression experimental design problem, and $D = 500$ for the pharmacokinetics, constant elasticity of substitution, and location finding experimental design problems. The designs used to compute the metrics for the linear regression and pharmacokinetics problems are discretisations across the design space and are fixed throughout experimentation. For the other two problems, designs are sampled randomly as part of the Bayesian optimisation procedure, but remain fixed each time the metrics are computed. This preserves fairness amongst all methods (BOED vs GBOED), since the same designs are used when computing the metrics.

\paragraph{Root Mean Square Error (RMSE)} To compute the RMSE, for a single design, we take the difference between the posterior predictive samples and the samples from the true DGP. This returns $N$ errors for each design. After taking this difference, we square it and take the mean over all $N$ squared differences. We then take the square root to give us the RMSE for each design. We finally compute the mean RMSE over all designs. Mathematically, this is
$$\mathrm{RMSE} = \frac{1}{D} \sum_{d = 1}^{D}\sqrt{\frac{1}{N} \sum_{i = 1}^{N} \left( \hat{\boldsymbol{y}}_{di} - \boldsymbol{y}_{di} \right)^{2}},$$
where $\hat{\boldsymbol{y}}_{di}$ are samples from the predictive distribution, and $\boldsymbol{y}_{di}$ are samples from the true data-generating process both for design $d$ and predictive/true sample $i$.

\paragraph{Maximum Mean Discrepancy (MMD)} To compute the MMD, we follow guidance by \citet{mmd}. For a single design, we have $N$ samples from the posterior predictive, and $M$ samples from the true DGP. So to compute the unbiased $\mathrm{MMD}$, we average over all designs
\begin{align*}
\mathrm{MMD} = \frac{1}{D} \sum_{d = 1}^{D} \left[ 
\frac{1}{N(N-1)} \sum_{i \ne j} k(\hat{\boldsymbol{y}}_{di}, \hat{\boldsymbol{y}}_{dj})
 + \frac{1}{M(M-1)} \sum_{i \ne j} k(\boldsymbol{y}_{di}, \boldsymbol{y}_{dj}) 
 - \frac{2}{NM} \sum_{i = 1}^{N} \sum_{j = 1}^{M} k(\hat{\boldsymbol{y}}_{di}, \boldsymbol{y}_{dj}) 
\right],
\end{align*}
where $k(\boldsymbol{x}, \boldsymbol{y})$ is the radial basis function (RBF) $$k(\boldsymbol{x}, \boldsymbol{y}) = \exp\left(-\frac{\|\boldsymbol{x} - \boldsymbol{y}\|_{2}^{2}}{2\sigma^{2}}\right)$$ for bandwidth $\sigma$. \citet{mmd} recommend using the median heuristic to compute $\sigma$, which is the (empirical) median distance between points in the aggregate sample of $\boldsymbol{x}$ and $\boldsymbol{y}$. This is traditionally computed as $\sigma = \sqrt{\frac{H_{n}}{2}},$ where $H_{n} = \mathrm{Median} \left\{ \| XY_{n,i} - XY_{n,j} \|_{2}^{2} \,\middle|\, 1 \le i < j \le n \right\}$ for aggregate $XY$ \citep{garreau2017large}. $N = M$ in our computations.

\paragraph{Log-Likelihood} Use the $N \times D$ samples from the true DGP and calculate the mean log-likelihood $$\text{Log-Likelihood} = \frac{1}{D} \sum_{d = 1}^{D} \left( \frac{1}{N} \sum_{i = 1}^{N}\log\left( \frac{1}{M} \sum_{j=1}^{M} p(\boldsymbol{y}_{i} \mid \boldsymbol{\theta}_j, \boldsymbol{\xi}_{d}) \right) \right),$$
where $\boldsymbol{y}_{i}$ are the samples from the true DGP and $\boldsymbol{\theta}_{j}$ are samples from the posterior. The likelihood used is that from the assumed statistical model. We use the negative log-likelihood (NLL) to assess performance.

\section{Additional Experiments}
\label{additionalexperappendix}

\subsection{Tabulated Results}
\label{tabresults}

\paragraph{Linear regression} The results presented in the main paper with all methods and metrics can be found in \Tableref{bayesregresultsfull}. All three metrics generally agree with each other. A figure containing performance during the entire experimental horizon is provided in \Figref{plotallscenarioregres}, with all three metrics and for all considered scenarios. The figure complements \Figref{allplotsmmd}. GBOED with score matching and our exponential decay method for tuning $c$ in the IMQ kernel achieves the strongest performance. To be distinct from BOED, we assign $\omega = 0.8$ to the power likelihood. Power likelihoods achieve worse performance due to the smaller learning rate; see \Appref{altlrapp}.

R-IDeA performs well in the well-specified setting, but fails to reach the performance of BOED under outlier-contamination. R-IDeA uses a separate test set of real-world data to inform the next design choice. As a result, it may not discard information from outliers like GBOED, when paired with certain loss functions, can. This continues in the subsequent experimental design problems.

\begin{table}[ht!]
\caption{Comparison of methods under both well-specified and misspecified scenarios in the linear regression problem, with $\omega = 1$ if not stated. $(q_{1}, q_{2}) = (9, 1)$ and $b = 0.04$ for the exponential decay method. RMSE, MMD, and NLL are recorded with mean (and standard error) over 3 different models, each replicated under 30 random seeds $(3 \times 30 = 90 \text{ replications in total})$.}
\centering
\begin{tabular}{lccc}
\toprule
\textbf{Method} & \textbf{RMSE} & \textbf{MMD} & \textbf{NLL} \\
\midrule
\multicolumn{4}{c}{Well-Specified} \\[0.2em]
BOED & $1.7383\,(0.0476)$ & $0.1340\,(0.0119)$ & $1.8032\,(0.0499)$ \\
R-IDeA & $1.7470\,(0.0448)$ & $0.1275\,(0.0101)$ & $1.7798\,(0.0443)$ \\
Power-Like $\omega = 0.8$ & $1.8665\,(0.0577)$ & $0.1889\,(0.0157)$ & $1.9885\,(0.0669)$ \\
Unweighted-SM & $1.6090\,(0.0406)$ & $0.0807\,(0.0083)$ & $1.6423\,(0.0383)$ \\
\citet{laplante2025robustconjugatespatiotemporalgaussian} & $6.9500\,(0.4326)$ & $0.9378\,(0.0652)$ & $31.9929\,(2.3881)$ \\
Exp-Decay $b = 0.04$  & \textbf{1.6024\,(0.0402)} & \textbf{0.0768\,(0.0079)} & \textbf{1.6288\,(0.0371)} \\
\textcolor{blindred}{Random} + \textcolor{blindblue}{Exp-Decay} & $1.6843\,(0.0484)$ & $0.1068\,(0.0114)$ & $1.7419\,(0.0530)$ \\
\textcolor{blindred}{BEIG} + \textcolor{blindblue}{Exp-Decay} & $1.6107\,(0.0419)$ & $0.0810\,(0.0093)$ & $1.6448\,(0.0412)$ \\
\midrule
\multicolumn{4}{c}{Asymmetric Outliers} \\[0.2em]
BOED & $2.6952\,(0.1224)$ & $0.5479\,(0.0341)$ & $4.0477\,(0.3282)$ \\
R-IDeA & $3.1608\,(0.1202)$ & $0.7210\,(0.0327)$ & $5.1462\,(0.2769)$ \\
Power-Like $\omega = 0.8$ & $2.7546\,(0.1188)$ & $0.5705\,(0.0348)$ & $4.0425\,(0.2785)$ \\
Unweighted-SM & $2.5351\,(0.0960)$ & $0.5185\,(0.0301)$ & $3.6622\,(0.1907)$ \\
\citet{laplante2025robustconjugatespatiotemporalgaussian} & $6.8456\,(0.3924)$ & $1.0299\,(0.0611)$ & $28.6630\,(2.1086)$ \\
Exp-Decay $b = 0.04$  & \textbf{2.4716\,(0.1016)} & \textbf{0.4803\,(0.0314)} & \textbf{3.4828\,(0.2064)} \\
\textcolor{blindred}{Random} + \textcolor{blindblue}{Exp-Decay} & $2.7877\,(0.1211)$ & $0.5732\,(0.0329)$ & $4.5146\,(0.3413)$ \\
\textcolor{blindred}{BEIG} + \textcolor{blindblue}{Exp-Decay} & $2.5645\,(0.1186)$ & $0.4990\,(0.0346)$ & $3.9283\,(0.3206)$ \\
\midrule
\multicolumn{4}{c}{Misspecified Error Distribution} \\[0.2em]
BOED & $2.0862\,(0.0558)$ & $0.1615\,(0.0153)$ & $2.4333\,(0.0780)$ \\
R-IDeA & $2.0594\,(0.0503)$ & $0.1365\,(0.0104)$ & $2.2835\,(0.0527)$ \\
Power-Like $\omega = 0.8$ & $2.1284\,(0.0569)$ & $0.1774\,(0.0141)$ & $2.4446\,(0.0649)$ \\
Unweighted-SM & \textbf{1.9614\,(0.0465)} & \textbf{0.1124\,(0.0107)} & \textbf{2.2486\,(0.0493)} \\
\citet{laplante2025robustconjugatespatiotemporalgaussian} & $7.1135\,(0.4261)$ & $0.9350\,(0.0619)$ & $32.4439\,(2.4411)$ \\
Exp-Decay $b = 0.04$  & $1.9652\,(0.0487)$ & $0.1125\,(0.0115)$ & $2.2494\,(0.0534)$ \\
\textcolor{blindred}{Random} + \textcolor{blindblue}{Exp-Decay} & $2.0443\,(0.0570)$ & $0.1398\,(0.0138)$ & $2.3585\,(0.0690)$ \\
\textcolor{blindred}{BEIG} + \textcolor{blindblue}{Exp-Decay} & $2.0169\,(0.0604)$ & $0.1309\,(0.0172)$ & $2.4242\,(0.1291)$ \\
\bottomrule
\end{tabular}
\label{bayesregresultsfull}
\end{table}

\begin{figure}[ht!]
    \centering
    \includegraphics[width=\linewidth]{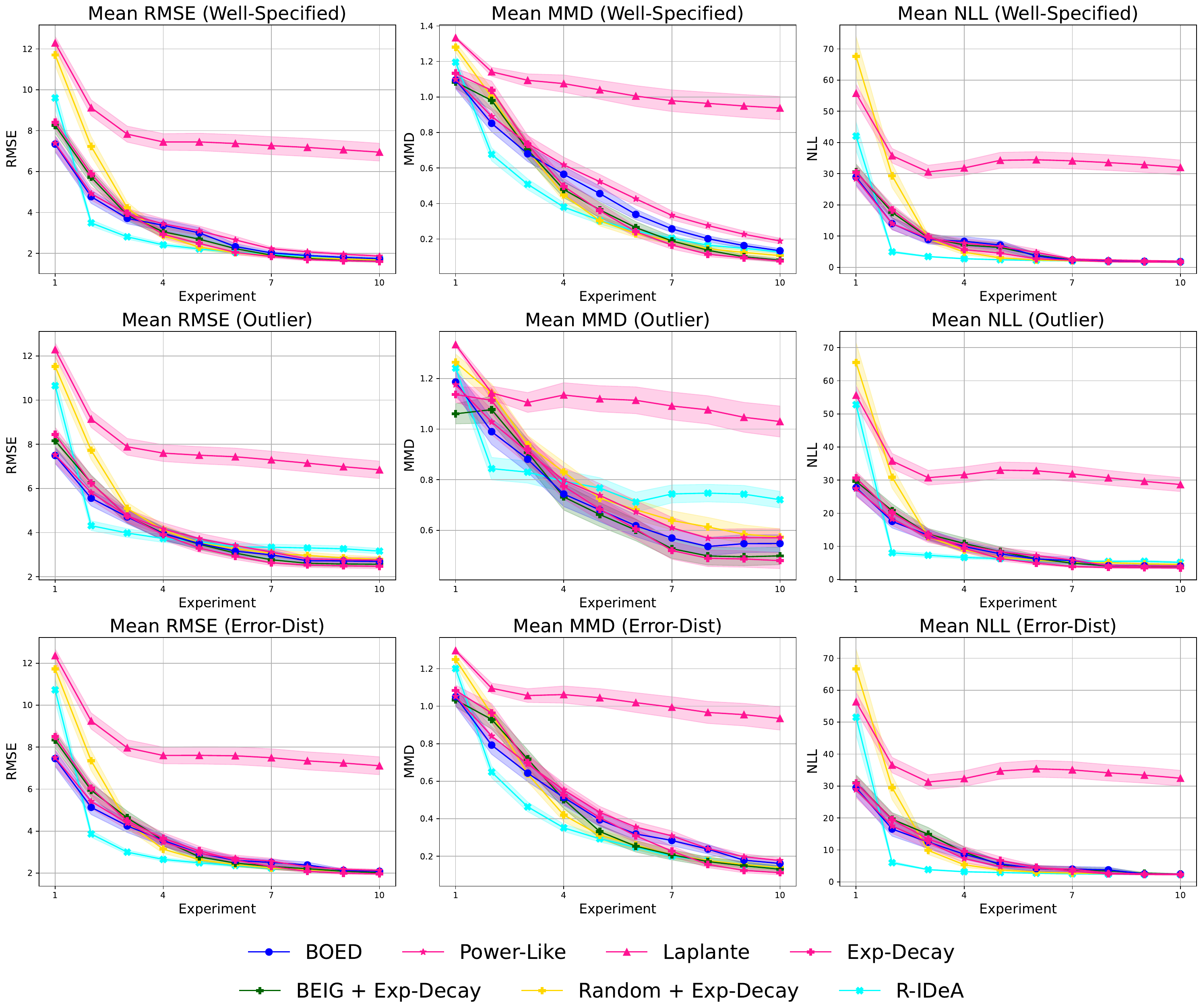}
    \caption{Methods compared on the well-specified scenario and the two misspecified scenarios for the Bayesian linear regression problem. RMSE, MMD, and NLL from left to right. \textit{Top}: Well-specified scenario. \textit{Middle}: Asymmetric Outliers scenario. \textit{Bottom}: Misspecified Error Distribution (Laplacian errors) scenario.}
    \label{plotallscenarioregres}
\end{figure}

\paragraph{Pharmacokinetics} \label{pharmaappendix} The results presented in the main paper with all methods and metrics can be found in \Tableref{pharmaresultsfull}. We see that BOED is strong in well-specified settings, but much weaker than GBOED when the model is misspecified. R-IDeA achieves the lowest RMSE but much higher MMD and NLL in the misspecified scenarios. Note that the MMD captures distributional differences not reflected by the RMSE. We therefore place more weight on the MMD when interpreting the results. 
On average, the Gibbs EIG achieves lower MMD in all considered scenarios than using the BEIG or random acquisition with GBI. \Figref{plotallscenariopharma} shows performance across the entire experimental horizon, with all three metrics, and for the well-specified, Asymmetric Outliers, and Misspecified Error Distribution scenarios. The plots in \Figref{plotallscenariopharma} contain the performance of R-IDeA, which was omitted from the main paper due to large MMD in comparison to the other methods. NLL for R-IDeA is omitted for the two misspecified scenarios due to being excessively large (see \Tableref{pharmaresultsfull} for their final average values compared to the other methods).

\begin{table}[ht!]
\caption{Comparison of methods under both well-specified and misspecified scenarios in the pharmacokinetics problem, with $\omega = 0.2$ if not stated. $(q_{1}, q_{2}) = (0.8, 0.2)$ and $b = 0.30$ for the exponential decay method. RMSE, MMD, and NLL are reported with mean (and standard error) over 100 replications.}
\centering
\begin{tabular}{lccc}
\toprule
\textbf{Method} & \textbf{RMSE} & \textbf{MMD} & \textbf{NLL} \\
\midrule
\multicolumn{4}{c}{Well-Specified} \\[0.2em]
BOED & $1.4027\,(0.0097)$ & \textbf{0.1112\,(0.0055)} & $1.3728\,(0.0184)$ \\
R-IDeA & \textbf{0.2441\,(0.0038)} & $0.1124\,(0.0098)$ & \textbf{-0.2322\,(0.1063)} \\
Power-Like $\omega = 0.2$ & $1.9750\,(0.0277)$ & $0.2636\,(0.0096)$ & $2.0029\,(0.0487)$ \\
\citet{laplante2025robustconjugatespatiotemporalgaussian} & $2.5063\,(0.0141)$ & $0.3378\,(0.0053)$ & $2.0432\,(0.0211)$ \\
Exp-Decay $b = 0.30$ & $2.6353\,(0.0153)$ & $0.3475\,(0.0061)$ & $2.0093\,(0.0207)$ \\
\textcolor{blindred}{Random} + \textcolor{blindblue}{Exp-Decay} & $3.2776\,(0.0331)$ & $0.4320\,(0.0074)$ & $2.1850\,(0.0204)$ \\
\textcolor{blindred}{BEIG} + \textcolor{blindblue}{Exp-Decay} & $3.0267\,(0.0305)$ & $0.3715\,(0.0048)$ & $2.0335\,(0.0116)$ \\
\midrule
\multicolumn{4}{c}{Asymmetric Outliers} \\[0.2em]
BOED & $1.9233\,(0.0480)$ & $0.5155\,(0.0346)$ & $4.3711\,(0.4065)$ \\
R-IDeA & \textbf{1.5297\,(0.0624)} & $1.3524\,(0.0298)$ & $114.3451\,(11.5569)$ \\
Power-Like $\omega = 0.2$ & $2.3114\,(0.0472)$ & $0.4992\,(0.0266)$ & $3.2370\,(0.2200)$ \\
\citet{laplante2025robustconjugatespatiotemporalgaussian} & $3.1165\,(0.1393)$ & $0.3985\,(0.0147)$ & $2.2880\,(0.1039)$ \\
Exp-Decay $b = 0.30$ & $3.0339\,(0.0416)$ & \textbf{0.3669\,(0.0060)} & \textbf{2.0401\,(0.0162)} \\
\textcolor{blindred}{Random} + \textcolor{blindblue}{Exp-Decay} & $3.3312\,(0.0315)$ & $0.4241\,(0.0066)$ & $2.1632\,(0.0174)$ \\
\textcolor{blindred}{BEIG} + \textcolor{blindblue}{Exp-Decay} & $3.2760\,(0.0365)$ & $0.3988\,(0.0063)$ & $2.1060\,(0.0167)$ \\
\midrule
\multicolumn{4}{c}{Misspecified Error Distribution} \\[0.2em]
BOED & $2.5272\,(0.0678)$ & $0.2370\,(0.0096)$ & $4.6730\,(0.1711)$ \\
R-IDeA & \textbf{0.8582\,(0.0250)} & $0.4954\,(0.0231)$ & $16.2933\,(1.2045)$ \\
Power-Like $\omega = 0.2$ & $2.8749\,(0.0619)$ & $0.2451\,(0.0078)$ & $4.3458\,(0.1368)$ \\
\citet{laplante2025robustconjugatespatiotemporalgaussian} & $3.3157\,(0.0450)$ & $0.2148\,(0.0089)$ & $3.3497\,(0.0616)$ \\
Exp-Decay $b = 0.30$ & $3.4241\,(0.0445)$ & \textbf{0.2040\,(0.0098)} & $3.1580\,(0.0564)$ \\
\textcolor{blindred}{Random} + \textcolor{blindblue}{Exp-Decay} & $3.9852\,(0.0600)$ & $0.2803\,(0.0084)$ & $3.2852\,(0.0614)$ \\
\textcolor{blindred}{BEIG} + \textcolor{blindblue}{Exp-Decay} & $3.7723\,(0.0530)$ & $0.2397\,(0.0086)$ & \textbf{3.1503\,(0.0542)} \\
\bottomrule
\end{tabular}
\label{pharmaresultsfull}
\end{table}

\begin{figure}[ht!]
    \centering
    \includegraphics[width=\linewidth]{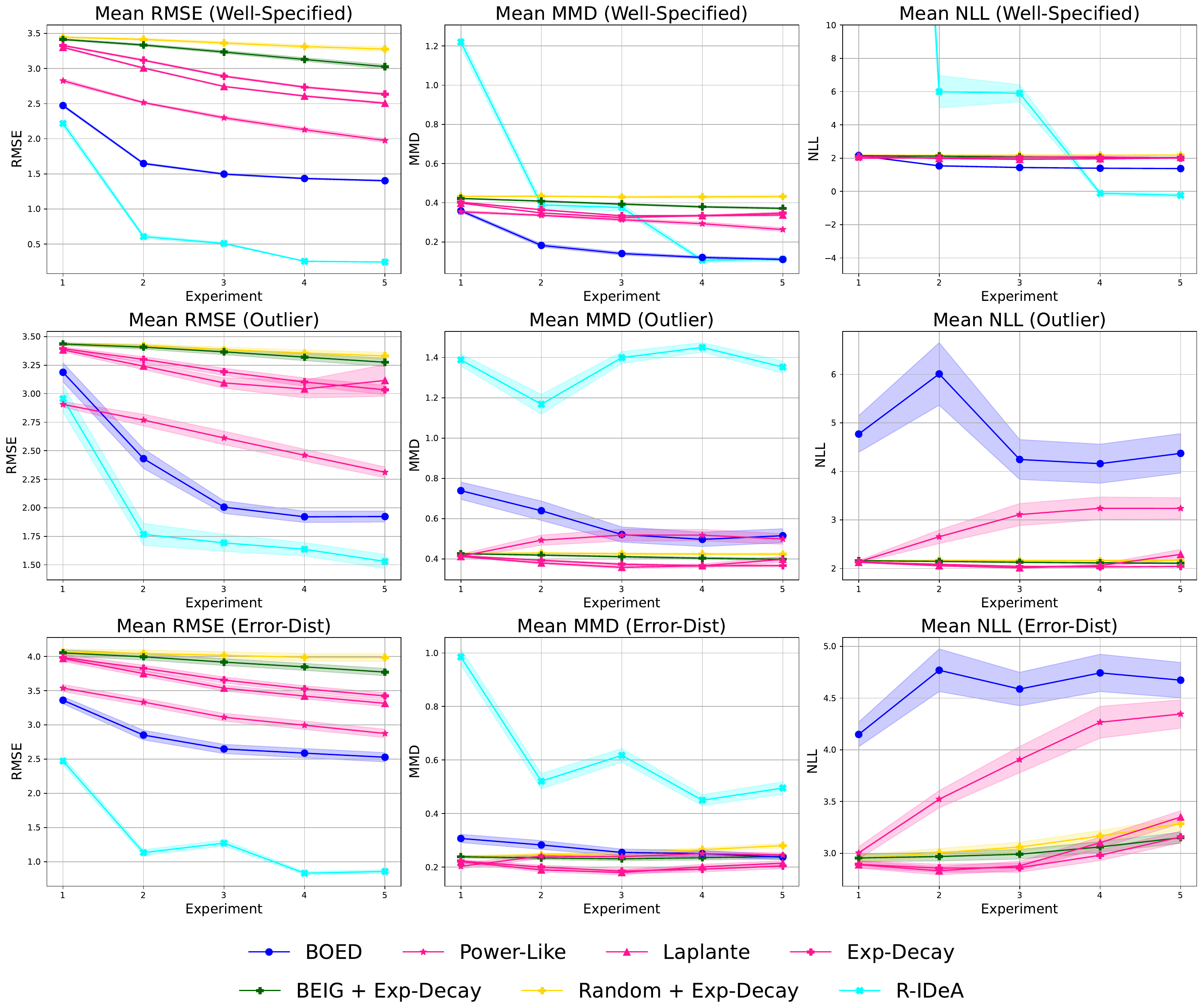}
    \caption{Methods compared on the well-specified scenario and the two misspecified scenarios for the pharmacokinetics problem. NLL for R-IDeA is too large to plot for reasonable comparison of the methods in the misspecified scenarios, and so is not visible in the plots. RMSE, MMD, and NLL from left to right. \textit{Top}: Well-specified scenario. \textit{Middle}: Asymmetric Outliers scenario. \textit{Bottom}: Misspecified Error Distribution scenario.}
    \label{plotallscenariopharma}
\end{figure}

\paragraph{Constant elasticity of substitution} \label{cesappendix} Results are presented in \Tableref{cesresultsfull}. The MMD is difficult to compute due to the nature of the metric and corresponding data. Since the data are in range $0 \leq \boldsymbol{y} \leq 1$, the MMD with our chosen kernel (see \Appref{metrics}) and bandwidth often results in \textit{NaN}. We do not include MMD results for this reason, and instead rely on the RMSE and NLL. A corresponding figure of the results is provided in \Figref{plotallscenarioces}. 

In this problem, the main competitor to GBOED is simply selecting designs at random with GBI rather than using the Gibbs EIG. BOED and the BEIG often fall behind in performance across all scenarios. With $\omega = 0.6$, power likelihoods offer the lowest RMSE and NLL in the Asymmetric Outliers scenario. The RMSE is also the smallest in the Misspecified Error Distribution scenario, with random design selection having better NLL.

\begin{table}[ht!]
\caption{Comparison of methods under both well-specified and misspecified scenarios in the constant elasticity of substitution problem, with $\omega = 0.6$ if not stated. RMSE and NLL are recorded with mean (and standard error) over 50 replications.}
\centering
\begin{tabular}{lcc}
\toprule
\textbf{Method} & \textbf{RMSE} & \textbf{NLL} \\
\midrule
\multicolumn{3}{c}{Well-Specified} \\[0.2em]
BOED & $0.1440\,(0.0125)$ & $1.1130\,(0.4905)$ \\
Power-Like $\omega = 0.6$ & $0.1389\,(0.0120)$ & $0.7708\,(0.5125)$ \\
\textcolor{blindred}{Random} + \textcolor{blindblue}{Power-Like} & \textbf{0.1307\,(0.0107)} & \textbf{-0.0398\,(0.3373)} \\
\textcolor{blindred}{BEIG} + \textcolor{blindblue}{Power-Like} & $0.1691\,(0.0153)$ & $3.0154\,(1.0611)$ \\
\midrule
\multicolumn{3}{c}{Asymmetric Outliers} \\[0.2em]
BOED & $0.2053\,(0.0185)$ & $7.4071\,(1.8552)$ \\
Power-Like $\omega = 0.6$ & \textbf{0.1501\,(0.0120)} & \textbf{1.3289\,(0.4835)} \\
\textcolor{blindred}{Random} + \textcolor{blindblue}{Power-Like} & $0.1553\,(0.0124)$ & $1.7402\,(0.9266)$ \\
\textcolor{blindred}{BEIG} + \textcolor{blindblue}{Power-Like} & $0.1888\,(0.0187)$ & $5.7697\,(1.6084)$ \\
\midrule
\multicolumn{3}{c}{Misspecified Error Distribution} \\[0.2em]
BOED & $0.3456\,(0.0230)$ & $9.5071\,(1.3707)$ \\
Power-Like $\omega = 0.6$ & \textbf{0.3223\,(0.0255)} & $9.4786\,(1.2444)$ \\
\textcolor{blindred}{Random} + \textcolor{blindblue}{Power-Like} & $0.3541\,(0.0237)$ & \textbf{8.5735\,(1.4093)} \\
\textcolor{blindred}{BEIG} + \textcolor{blindblue}{Power-Like} & $0.3539\,(0.0262)$ & $12.4672\,(1.8539)$ \\
\bottomrule
\end{tabular}
\label{cesresultsfull}
\end{table}

\begin{figure}[ht!]
    \centering
    \includegraphics[width=\linewidth]{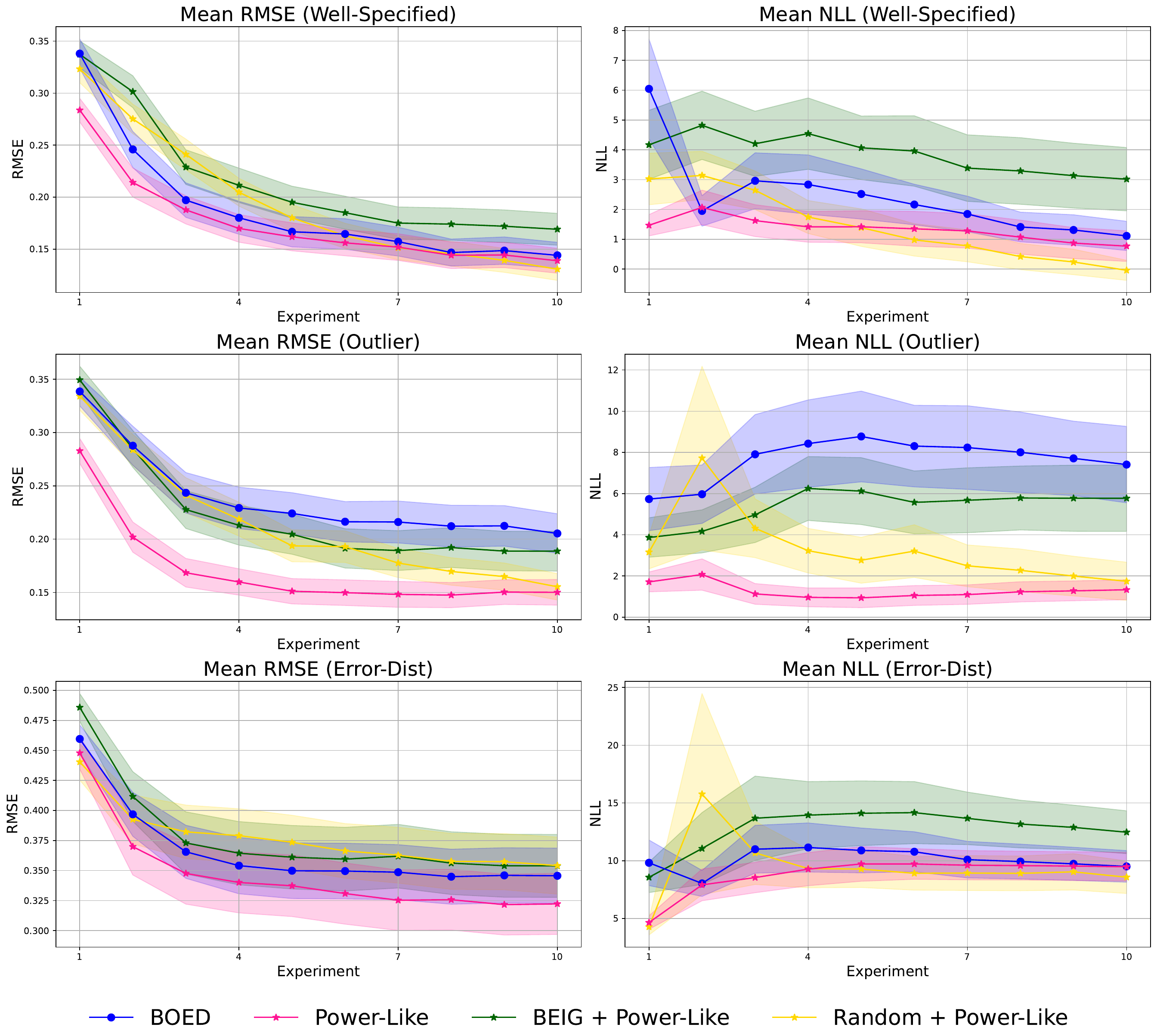}
    \caption{Methods compared on the well-specified scenario and the two misspecified scenarios for the constant elasticity of substitution problem. RMSE and NLL from left to right. \textit{Top}: Well-specified scenario. \textit{Middle}: Asymmetric Outliers scenario. \textit{Bottom}: Misspecified Error Distribution scenario.}
    \label{plotallscenarioces}
\end{figure}

\paragraph{Location finding} The results presented in the main paper with all methods and metrics can be found in \Tableref{locfindingresultsfull}. The metrics do not always agree with each other. \Figref{plotallscenariolocfind} shows performance across the entire experimental horizon for $d = 2$, with all three metrics, and for the well-specified, Asymmetric Outliers, and Misspecified Error Distribution scenarios.

BOED, as expected, performs best in the well-specified scenario across all dimensions. We find that GBOED beats BOED in the Asymmetric Outliers scenario and achieves the lowest MMD in all dimensions. BOED performs strongly in the Misspecified Error Distribution scenario. This is likely because of both the true error distribution being close to the assumed counterpart ($\sigma = 0.5$ vs $\sigma = 0.8$), and that a large amount of data are gathered to account for the mild deviation. The majority of metrics across the scenarios explain that the BEIG and random acquisition with GBI perform worse than GBOED. R-IDeA generally performs poorly, likely because of the method requiring discretisation of the design space. Location finding is often treated as a continuous design space problem --- GBOED can work in both discrete and continuous design spaces.

\begin{table}[ht!]
\caption{Comparison of methods across different dimensions $d$ for the location finding problem under well-specified and misspecified scenarios, with $\omega = 0.4$ if not stated. $(q_{1}, q_{2}) = (1.8, 0.2)$ and $b = 0.06$ for the exponential decay method. RMSE, MMD, and NLL are reported with mean (and standard error) over 20 replications.}
\centering
\begin{adjustbox}{max width=\textwidth}
\begin{tabular}{lccccccccc}
\toprule
\multirow{2}{*}{\textbf{Method}} 
& \multicolumn{3}{c}{$d = 2$} 
& \multicolumn{3}{c}{$d = 4$} \\
\cmidrule(lr){2-4} \cmidrule(lr){5-7} \cmidrule(lr){8-10}
& \textbf{RMSE} & \textbf{MMD} & \textbf{NLL} 
& \textbf{RMSE} & \textbf{MMD} & \textbf{NLL} \\
\midrule
\multicolumn{10}{c}{Well-Specified} \\[0.2em]
BOED & \textbf{0.7073\,(0.0001)} & \textbf{0.0001\,(0.0000)} & \textbf{0.7263\,(0.0003)} & \textbf{0.7264\,(0.0036)} & \textbf{0.0103\,(0.0021)} & \textbf{0.7574\,(0.0070)} \\
R-IDeA & $1.0187\,(0.0233)$ & $0.0939\,(0.0099)$ & $1.0191\,(0.0314)$ & $0.7663\,(0.0037)$ & $0.0282\,(0.0026)$ & $0.8058\,(0.0073)$ \\
Power-Like $\omega = 0.4$ & $0.7132\,(0.0052)$ & $0.0028\,(0.0025)$ & $0.7338\,(0.0068)$ & $0.7381\,(0.0028)$ & $0.0139\,(0.0016)$ & $0.7676\,(0.0050)$ \\
\citet{laplante2025robustconjugatespatiotemporalgaussian} 
    & $0.7078\,(0.0002)$ & $0.0004\,(0.0001)$ & $0.7269\,(0.0003)$ & $0.7365\,(0.0028)$ & $0.0139\,(0.0018)$ & $0.7662\,(0.0054)$ \\
Exp-Decay $b = 0.06$ 
    & $0.7123\,(0.0032)$ & $0.0011\,(0.0006)$ & $0.7293\,(0.0018)$ & $0.7382\,(0.0029)$ & $0.0136\,(0.0017)$ & $0.7663\,(0.0053)$ \\
\textcolor{blindred}{Random} + \textcolor{blindblue}{Exp-Decay} 
    & $0.8378\,(0.0107)$ & $0.0502\,(0.0056)$ & $0.8689\,(0.0163)$ & $0.7544\,(0.0024)$ & $0.0221\,(0.0020)$ & $0.7906\,(0.0060)$ \\
\textcolor{blindred}{BEIG} + \textcolor{blindblue}{Exp-Decay} 
    & $0.7164\,(0.0081)$ & $0.0030\,(0.0027)$ & $0.7347\,(0.0075)$ & $0.7371\,(0.0023)$ & $0.0141\,(0.0015)$ & $0.7687\,(0.0047)$ \\
\midrule
\multicolumn{10}{c}{Asymmetric Outliers} \\[0.2em]
BOED & $0.9062\,(0.0255)$ & $0.1243\,(0.0216)$ & $1.3135\,(0.1530)$ & $0.7956\,(0.0072)$ & $0.0528\,(0.0055)$ & $0.8967\,(0.0188)$ \\
R-IDeA & $1.0430\,(0.0213)$ & $0.1013\,(0.0097)$ & $1.0311\,(0.0280)$ & $0.7666\,(0.0037)$ & $0.0284\,(0.0026)$ & $0.8065\,(0.0074)$ \\
Power-Like $\omega = 0.4$ & $0.9164\,(0.0199)$ & $0.0869\,(0.0123)$ & $1.0288\,(0.0479)$ & $0.7681\,(0.0059)$ & $0.0282\,(0.0039)$ & $0.8083\,(0.0118)$ \\
\citet{laplante2025robustconjugatespatiotemporalgaussian} 
    & $0.8044\,(0.0211)$ & \textbf{0.0344\,(0.0088)} & \textbf{0.8336\,(0.0277)} & $0.7510\,(0.0036)$ & $0.0186\,(0.0018)$ & $0.7815\,(0.0057)$ \\
Exp-Decay $b = 0.06$ 
    & \textbf{0.8030\,(0.0186)} & $0.0388\,(0.0086)$ & $0.8541\,(0.0313)$ & $0.7483\,(0.0033)$ & \textbf{0.0184\,(0.0020)} & \textbf{0.7800\,(0.0058)} \\
\textcolor{blindred}{Random} + \textcolor{blindblue}{Exp-Decay} 
    & $0.8759\,(0.0116)$ & $0.0651\,(0.0075)$ & $0.9177\,(0.0224)$ & $0.7574\,(0.0022)$ & $0.0235\,(0.0019)$ & $0.7932\,(0.0054)$ \\
\textcolor{blindred}{BEIG} + \textcolor{blindblue}{Exp-Decay} 
    & $0.8376\,(0.0206)$ & $0.0509\,(0.0095)$ & $0.8831\,(0.0297)$ & \textbf{0.7456\,(0.0033)} & $0.0191\,(0.0019)$ & $0.7825\,(0.0058)$ \\
\midrule
\multicolumn{10}{c}{Misspecified Error Distribution} \\[0.2em]
BOED & \textbf{0.9448\,(0.0005)} & \textbf{0.0462\,(0.0003)} & $1.5093\,(0.0014)$ & \textbf{0.9663\,(0.0024)} & $0.0493\,(0.0007)$ & $1.5042\,(0.0046)$ \\
R-IDeA & $1.2529\,(0.0217)$ & $0.0714\,(0.0062)$ & \textbf{1.4801\,(0.0228)} & $0.9928\,(0.0030)$ & $0.0496\,(0.0014)$ & \textbf{1.4755\,(0.0063)} \\
Power-Like $\omega = 0.4$ & $0.9678\,(0.0138)$ & $0.0504\,(0.0041)$ & $1.5114\,(0.0111)$ & $0.9678\,(0.0020)$ & \textbf{0.0480\,(0.0010)} & $1.4949\,(0.0049)$ \\
\citet{laplante2025robustconjugatespatiotemporalgaussian} 
    & $0.9669\,(0.0109)$ & $0.0488\,(0.0023)$ & $1.5070\,(0.0080)$ & $0.9698\,(0.0022)$ & \textbf{0.0480\,(0.0008)} & $1.4914\,(0.0037)$ \\
Exp-Decay $b = 0.06$ 
    & $0.9823\,(0.0136)$ & $0.0485\,(0.0033)$ & $1.4973\,(0.0131)$ & $0.9721\,(0.0021)$ & \textbf{0.0480\,(0.0009)} & $1.4900\,(0.0045)$ \\
\textcolor{blindred}{Random} + \textcolor{blindblue}{Exp-Decay} 
    & $1.0858\,(0.0125)$ & $0.0725\,(0.0063)$ & $1.5484\,(0.0274)$ & $0.9827\,(0.0019)$ & $0.0485\,(0.0013)$ & $1.4830\,(0.0068)$ \\
\textcolor{blindred}{BEIG} + \textcolor{blindblue}{Exp-Decay} 
    & $0.9889\,(0.0190)$ & $0.0509\,(0.0036)$ & $1.5150\,(0.0139)$ & $0.9700\,(0.0019)$ & $0.0482\,(0.0010)$ & $1.4920\,(0.0048)$ \\
\bottomrule
\end{tabular}
\end{adjustbox}
\label{locfindingresultsfull}
\end{table}

\begin{figure}[ht!]
    \centering
    \includegraphics[width=\linewidth]{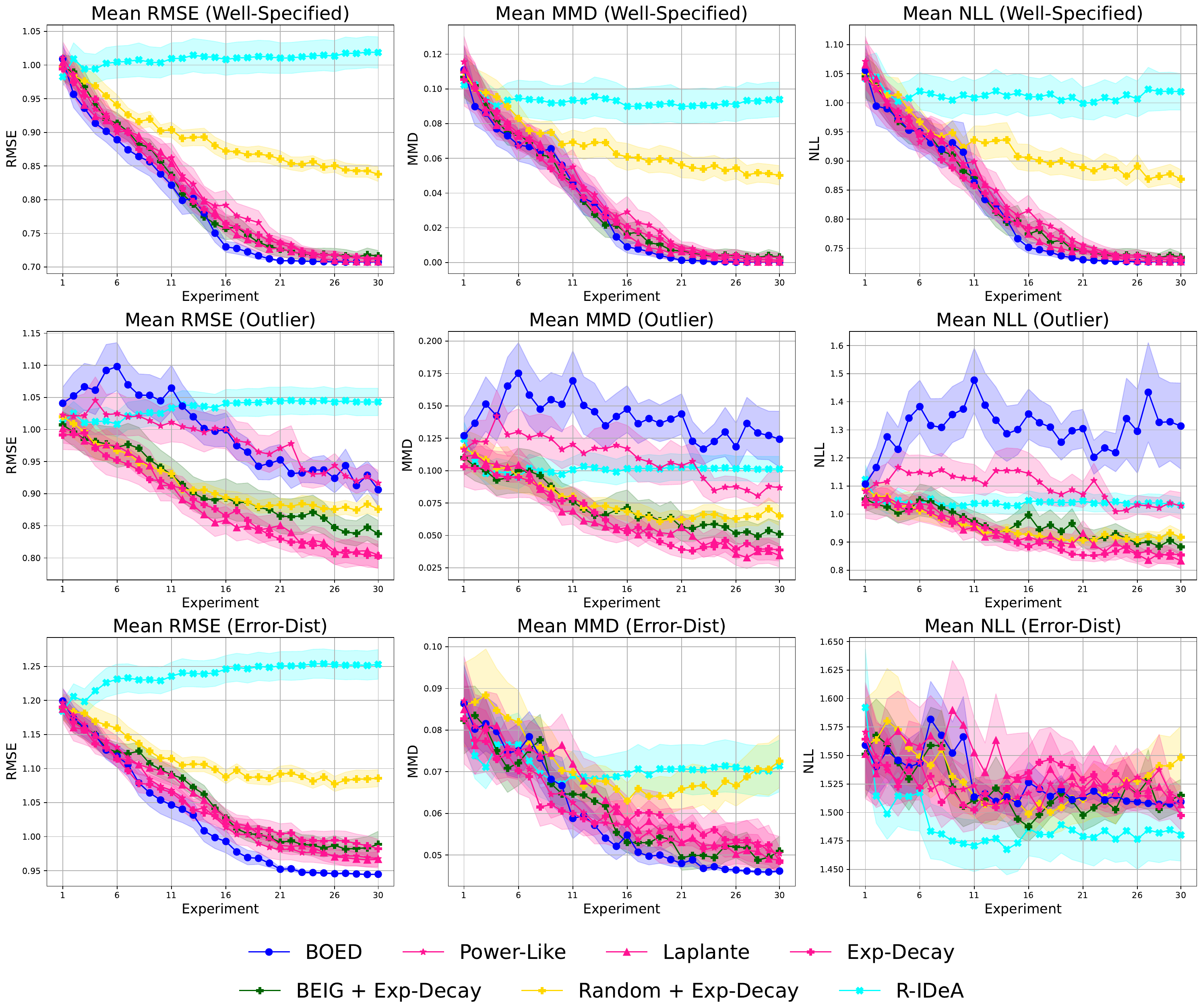}
    \caption{Methods compared on the well-specified scenario and the two misspecified scenarios for the 2D location finding problem. RMSE, MMD, and NLL from left to right. \textit{Top}: Well-specified scenario. \textit{Middle}: Asymmetric Outliers scenario. \textit{Bottom}: Misspecified Error Distribution scenario.}
    \label{plotallscenariolocfind}
\end{figure}

\paragraph{Inference capabilities on a fixed dataset} 
\label{beignllregression}

To isolate the effect of GBI, we here additionally show the results of experiments in which a dataset is already made available and inference is conducted according to a certain method. We perform our analyses on the linear regression problem. 

The columns in \Tableref{eig_nll_comparison} present the RMSE, MMD, and NLL from performing Gibbs (or Bayesian, in the case of BOED) inference with the indicated loss function on a fixed dataset of designs and observations (i.e., no new data are collected and so a utility function serves no purpose). Each data point was processed sequentially to enable the use of sequential inference methods like that by \citet{laplante2025robustconjugatespatiotemporalgaussian}, i.e., $N$ posterior updates are performed for each of the $N$ data points. The datasets were acquired using BOED; one dataset for the well-specified scenario, and the other for the Asymmetric Outliers scenario. The data are not necessarily the same across the two scenarios, but are fixed across methods within a specific scenario.

As we see in \Tableref{eig_nll_comparison}, with a fixed dataset of designs and observations, GBI has the ability to deliver strong predictive performance in terms of the MMD. A notable distinction between the results in the main paper and \Tableref{eig_nll_comparison} is that predictive performance is generally poorer when we use this fixed dataset, in contrast to using the Gibbs EIG to adaptively select designs. This means that the same sequence of designs chosen by BOED is not optimal for GBI.

\begin{table}[ht!]
\caption{Comparison of methods based on the RMSE, MMD, and NLL under a fixed set of designs and observations in the linear regression problem. Mean (and standard error) over 90 replications. $\omega = 1$ if not stated.}
\centering
\renewcommand{\arraystretch}{1.05}
\setlength{\tabcolsep}{4.5pt}
\begin{adjustbox}{max width=\linewidth}
\begin{tabular}{lccc}
\toprule
\textbf{Method} & \textbf{RMSE} & \textbf{MMD} & \textbf{NLL} \\
\midrule
\multicolumn{4}{c}{Well-Specified} \\[0.2em]
BOED & $1.7383\,(0.0476)$ & $0.1340\,(0.0119)$ & $1.8032\,(0.0499)$ \\
Power-Like $\omega = 0.8$ & $1.8439\,(0.0555)$ & $0.1787\,(0.0149)$ & $1.9474\,(0.0619)$ \\
Unweighted-SM & \textbf{1.5981\,(0.0395)} & \textbf{0.0757\,(0.0080)} & \textbf{1.6258\,(0.0365)} \\
\citet{laplante2025robustconjugatespatiotemporalgaussian} & $6.9071\,(0.4383)$ & $0.9425\,(0.0661)$ & $31.7845\,(2.4816)$ \\
$r_{\mathrm{IMQ}} \mid c = 2$ & $5.5147\,(0.3793)$ & $0.8503\,(0.0634)$ & $19.7778\,(1.7382)$ \\
$r_{\mathrm{IMQ}} \mid c = 10$ & $1.6037\,(0.0403)$ & $0.0781\,(0.0083)$ & $1.6331\,(0.0377)$ \\
Exp-Decay $b = 0.04$ & $1.6118\,(0.0419)$ & $0.0814\,(0.0091)$ & $1.6450\,(0.0405)$ \\
Exp-Decay $b = 0.08$ & $1.6544\,(0.0553)$ & $0.0966\,(0.0144)$ & $1.7353\,(0.0844)$ \\
Exp-Decay $b = 0.10$ & $1.7131\,(0.0764)$ & $0.1131\,(0.0193)$ & $1.9027\,(0.1740)$ \\
\midrule
\multicolumn{4}{c}{Asymmetric Outliers} \\[0.2em]
BOED & $2.6952\,(0.1224)$ & $0.5479\,(0.0341)$ & $4.0477\,(0.3282)$ \\
Power-Like $\omega = 0.8$ & $2.7981\,(0.1298)$ & $0.5775\,(0.0357)$ & $4.2296\,(0.3442)$ \\
Unweighted-SM & \textbf{2.5568\,(0.1125)} & $0.5121\,(0.0323)$ & \textbf{3.8674\,(0.3104)} \\
\citet{laplante2025robustconjugatespatiotemporalgaussian} & $6.8142\,(0.4061)$ & $1.0218\,(0.0595)$ & $29.0563\,(2.2280)$ \\
$r_{\mathrm{IMQ}} \mid c = 2$ & $5.3316\,(0.3580)$ & $0.8948\,(0.0600)$ & $17.3778\,(1.5928)$ \\
$r_{\mathrm{IMQ}} \mid c = 10$ & $2.5646\,(0.1181)$ & $0.5045\,(0.0340)$ & $3.9041\,(0.3199)$ \\
Exp-Decay $b = 0.04$ & $2.5854\,(0.1243)$ & \textbf{0.5031\,(0.0356)} & $3.9824\,(0.3358)$ \\
Exp-Decay $b = 0.08$ & $2.6330\,(0.1349)$ & $0.5056\,(0.0381)$ & $4.1507\,(0.3655)$ \\
Exp-Decay $b = 0.10$ & $2.6716\,(0.1421)$ & $0.5096\,(0.0397)$ & $4.2866\,(0.3872)$ \\
\bottomrule
\end{tabular}
\end{adjustbox}
\label{eig_nll_comparison}
\end{table}

\subsection{Alternative Learning Rates}
\label{altlrapp}

\begin{figure*}
    \centering
    \includegraphics[width=\linewidth]{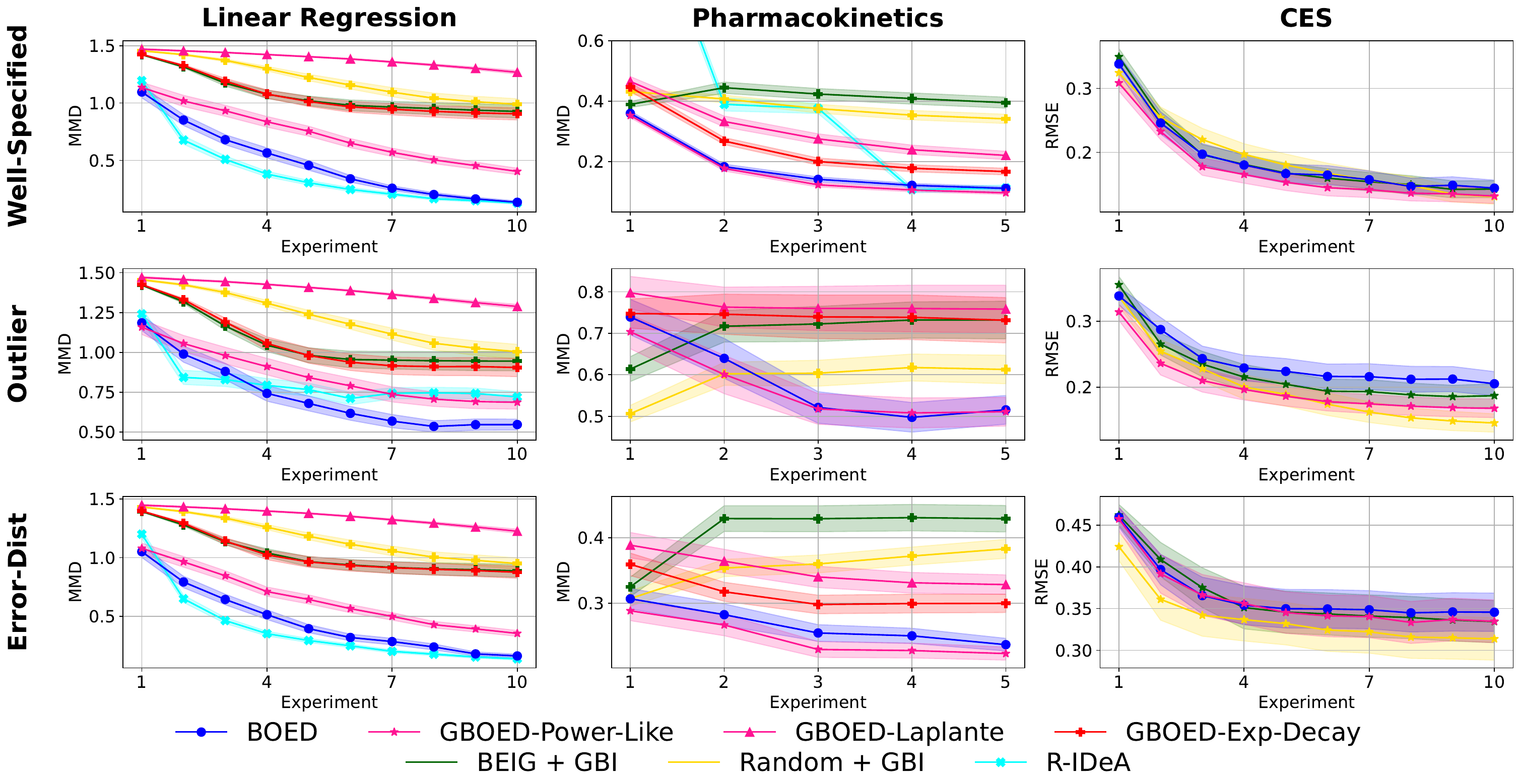}
    \caption{Performance across three experimental design problems (columns) and three model specification scenarios (rows); mean MMD/RMSE per experiment conducted and standard error shaded. Smaller $y$-axis values are better. Methods compared: BOED (\textcolor{blindblue}{blue}), GBOED with different loss functions (\textcolor{blindred}{red}/\textcolor{blindpink}{pink}), BEIG + GBI (\textcolor{blindgreen}{green}), Random + GBI (\textcolor{blindgold}{gold}), and R-IDeA (\textcolor{blindcyan}{cyan}). For GBI methods, markers distinguish the loss function used.}
    \label{allplotsmmd2}
\end{figure*}

\Figref{allplotsmmd2} displays the results for the worst-case, as a function of tested values of the learning rate $\omega$, performance of GBOED in the presence of misspecification. In other words, \Figref{allplotsmmd2} is identical to \Figref{allplotsmmd}, with the exception that the learning rate $\omega$ differs for each experimental design problem. Location finding is excluded; due to expensive computational costs of running experiments (see \Appref{locdeploymenttime}), we only tested one learning rate in this problem. 

The learning rate used in the experiments shown in \Figref{allplotsmmd2} is, in each experimental setting, the tested learning rate that delivers the worst predictive performance in the Asymmetric Outliers scenario when using the weighted score matching loss with exponential decay. This excludes constant elasticity of substitution, where only power likelihoods are considered. Pharmacokinetics uses the second worst learning rate to view the improvement in well-specified performance (that comes from using a larger learning rate). In all settings, this mostly coincides with the tested learning rate that delivers the worst predictive performance in the Misspecified Error Distribution scenario when using this loss function. In linear regression, this also coincides with the tested learning rate that delivers the worst predictive performance when using other loss functions.

The following paragraphs discuss the alternative learning rates considered.

\paragraph{Linear regression} \label{linreglrsel} \Figref{allplotsmmd2} shows results with $\omega = 0.4$ (while \Figref{allplotsmmd} shows results with $\omega = 1$), which results in slower learning due to increased robustness. This leads to GBOED falling behind BOED and R-IDeA, which both use an implicit learning rate $\omega = 1$. We infer that in this setting, fast learning is very beneficial when the prior is poorly chosen. This demonstrates that prior misspecification can adversely impact robust strategies for model misspecification in experimental design. When selecting the learning rate, one may need to consider the possibility of both situations occurring in practice. 

Tables \ref{bayesregresultsomega1} and \ref{bayesregresultsomega2} present tabulated results from using learning rates $\omega = 0.4$ and $\omega = 0.8$, respectively. Notice that in this case of prior misspecification, our exponential decay method for weighted score matching performs better than the method by \citet{laplante2025robustconjugatespatiotemporalgaussian}. GBOED (with the Gibbs EIG) consistently achieves lower MMD on average over the BEIG and random acquisition (with the same loss function for GBI) across all learning rates and scenarios considered. 

\begin{table}[ht!]
\caption{Comparison of methods under both well-specified and misspecified scenarios in the linear regression problem, with $\omega = 0.4$ if not stated. $(q_{1}, q_{2}) = (9, 1)$ and $b = 0.04$ for the exponential decay method. RMSE, MMD, and NLL are recorded with mean (and standard error) over 3 different models, each replicated under 30 random seeds $(3 \times 30 = 90 \text{ replications in total})$.}
\centering
\begin{tabular}{lccc}
\toprule
\textbf{Method} & \textbf{RMSE} & \textbf{MMD} & \textbf{NLL} \\
\midrule
\multicolumn{4}{c}{Well-Specified} \\[0.2em]
BOED & \textbf{1.7383\,(0.0476)} & $0.1340\,(0.0119)$ & $1.8032\,(0.0499)$ \\
R-IDeA & $1.7470\,(0.0448)$ & \textbf{0.1275\,(0.0101)} & \textbf{1.7798\,(0.0443)} \\
Power-Like $\omega = 0.4$ & $2.4107\,(0.0934)$ & $0.4035\,(0.0279)$ & $2.8378\,(0.1307)$ \\
\citet{laplante2025robustconjugatespatiotemporalgaussian} & $11.4690\,(0.4346)$ & $1.2685\,(0.0160)$ & $52.8847\,(3.4198)$ \\
Exp-Decay $b = 0.04$  & $6.7598\,(0.4049)$ & $0.9062\,(0.0533)$ & $22.8541\,(1.7444)$ \\
\textcolor{blindred}{Random} + \textcolor{blindblue}{Exp-Decay} & $8.2380\,(0.4819)$ & $0.9895\,(0.0484)$ & $31.5386\,(2.3806)$ \\
\textcolor{blindred}{BEIG} + \textcolor{blindblue}{Exp-Decay} & $6.5456\,(0.3931)$ & $0.9267\,(0.0564)$ & $21.0034\,(1.6113)$ \\
\midrule
\multicolumn{4}{c}{Asymmetric Outliers} \\[0.2em]
BOED & \textbf{2.6952\,(0.1224)} & \textbf{0.5479\,(0.0341)} & \textbf{4.0477\,(0.3282)} \\
R-IDeA & $3.1608\,(0.1202)$ & $0.7210\,(0.0327)$ & $5.1462\,(0.2769)$ \\
Power-Like $\omega = 0.4$ & $3.2148\,(0.1492)$ & $0.6875\,(0.0429)$ & $4.8087\,(0.2943)$ \\
\citet{laplante2025robustconjugatespatiotemporalgaussian} & $11.5935\,(0.4183)$ & $1.2883\,(0.0133)$ & $53.7434\,(3.2620)$ \\
Exp-Decay $b = 0.04$  & $6.6967\,(0.4025)$ & $0.9058\,(0.0577)$ & $21.7893\,(1.6899)$ \\
\textcolor{blindred}{Random} + \textcolor{blindblue}{Exp-Decay} & $8.1455\,(0.4592)$ & $1.0047\,(0.0461)$ & $30.4842\,(2.2552)$ \\
\textcolor{blindred}{BEIG} + \textcolor{blindblue}{Exp-Decay} & $6.3873\,(0.3897)$ & $0.9456\,(0.0610)$ & $18.8495\,(1.5150)$ \\
\midrule
\multicolumn{4}{c}{Misspecified Error Distribution} \\[0.2em]
BOED & $2.0862\,(0.0558)$ & $0.1615\,(0.0153)$ & $2.4333\,(0.0780)$ \\
R-IDeA & \textbf{2.0594\,(0.0503)} & \textbf{0.1365\,(0.0104)} & \textbf{2.2835\,(0.0527)} \\
Power-Like $\omega = 0.4$ & $2.6125\,(0.0939)$ & $0.3535\,(0.0249)$ & $3.2315\,(0.1477)$ \\
\citet{laplante2025robustconjugatespatiotemporalgaussian} & $11.5017\,(0.4321)$ & $1.2255\,(0.0178)$ & $52.4914\,(3.3631)$ \\
Exp-Decay $b = 0.04$  & $6.9716\,(0.4016)$ & $0.8756\,(0.0515)$ & $23.5516\,(1.7708)$ \\
\textcolor{blindred}{Random} + \textcolor{blindblue}{Exp-Decay} & $8.3831\,(0.4803)$ & $0.9496\,(0.0487)$ & $31.9240\,(2.3820)$ \\
\textcolor{blindred}{BEIG} + \textcolor{blindblue}{Exp-Decay} & $6.7207\,(0.3910)$ & $0.8876\,(0.0548)$ & $21.6892\,(1.6520)$ \\
\bottomrule
\end{tabular}
\label{bayesregresultsomega1}
\end{table}

\begin{table}[ht!]
\caption{Comparison of methods under both well-specified and misspecified scenarios in the linear regression problem, with $\omega = 0.8$ if not stated. $(q_{1}, q_{2}) = (9, 1)$ and $b = 0.04$ for the exponential decay method. RMSE, MMD, and NLL are recorded with mean (and standard error) over 3 different models, each replicated under 30 random seeds $(3 \times 30 = 90 \text{ replications in total})$.}
\centering
\begin{tabular}{lccc}
\toprule
\textbf{Method} & \textbf{RMSE} & \textbf{MMD} & \textbf{NLL} \\
\midrule
\multicolumn{4}{c}{Well-Specified} \\[0.2em]
BOED & \textbf{1.7383\,(0.0476)} & $0.1340\,(0.0119)$ & $1.8032\,(0.0499)$ \\
R-IDeA & $1.7470\,(0.0448)$ & \textbf{0.1275\,(0.0101)} & \textbf{1.7798\,(0.0443)} \\
Power-Like $\omega = 0.8$ & $1.8665\,(0.0577)$ & $0.1889\,(0.0157)$ & $1.9885\,(0.0669)$ \\
\citet{laplante2025robustconjugatespatiotemporalgaussian} & $7.4977\,(0.4345)$ & $1.0056\,(0.0570)$ & $34.7885\,(2.4839)$ \\
Exp-Decay $b = 0.04$  & $1.9185\,(0.0807)$ & $0.2006\,(0.0216)$ & $2.1103\,(0.1275)$ \\
\textcolor{blindred}{Random} + \textcolor{blindblue}{Exp-Decay} & $2.0525\,(0.0909)$ & $0.2336\,(0.0235)$ & $2.3403\,(0.1467)$ \\
\textcolor{blindred}{BEIG} + \textcolor{blindblue}{Exp-Decay} & $1.9483\,(0.0886)$ & $0.2103\,(0.0231)$ & $2.2058\,(0.1753)$ \\
\midrule
\multicolumn{4}{c}{Asymmetric Outliers} \\[0.2em]
BOED & \textbf{2.6952\,(0.1224)} & \textbf{0.5479\,(0.0341)} & \textbf{4.0477\,(0.3282)} \\
R-IDeA & $3.1608\,(0.1202)$ & $0.7210\,(0.0327)$ & $5.1462\,(0.2769)$ \\
Power-Like $\omega = 0.8$ & $2.7546\,(0.1188)$ & $0.5705\,(0.0348)$ & $4.0425\,(0.2785)$ \\
\citet{laplante2025robustconjugatespatiotemporalgaussian} & $7.5055\,(0.4116)$ & $1.0615\,(0.0537)$ & $32.9997\,(2.3091)$ \\
Exp-Decay $b = 0.04$  & $2.8741\,(0.1396)$ & $0.5932\,(0.0379)$ & $4.3679\,(0.3101)$ \\
\textcolor{blindred}{Random} + \textcolor{blindblue}{Exp-Decay} & $3.2428\,(0.1552)$ & $0.6744\,(0.0366)$ & $5.5099\,(0.4052)$ \\
\textcolor{blindred}{BEIG} + \textcolor{blindblue}{Exp-Decay} & $2.9920\,(0.1557)$ & $0.6215\,(0.0397)$ & $4.8571\,(0.4166)$ \\
\midrule
\multicolumn{4}{c}{Misspecified Error Distribution} \\[0.2em]
BOED & $2.0862\,(0.0558)$ & $0.1615\,(0.0153)$ & $2.4333\,(0.0780)$ \\
R-IDeA & \textbf{2.0594\,(0.0503)} & \textbf{0.1365\,(0.0104)} & \textbf{2.2835\,(0.0527)} \\
Power-Like $\omega = 0.8$ & $2.1284\,(0.0569)$ & $0.1774\,(0.0141)$ & $2.4446\,(0.0649)$ \\
\citet{laplante2025robustconjugatespatiotemporalgaussian} & $7.7228\,(0.4315)$ & $1.0038\,(0.0546)$ & $36.0098\,(2.5678)$ \\
Exp-Decay $b = 0.04$  & $2.2897\,(0.0934)$ & $0.2285\,(0.0227)$ & $2.9688\,(0.3173)$ \\
\textcolor{blindred}{Random} + \textcolor{blindblue}{Exp-Decay} & $2.3993\,(0.0981)$ & $0.2526\,(0.0233)$ & $2.9876\,(0.1713)$ \\
\textcolor{blindred}{BEIG} + \textcolor{blindblue}{Exp-Decay} & $2.3320\,(0.1045)$ & $0.2379\,(0.0256)$ & $3.1233\,(0.3414)$ \\
\bottomrule
\end{tabular}
\label{bayesregresultsomega2}
\end{table}

\paragraph{Pharmacokinetics} \label{pharmacolrselect} \Figref{allplotsmmd2} shows results with $\omega = 0.8$ (while \Figref{allplotsmmd} shows results with $\omega = 0.2$), which results in faster learning, but less robustness. Larger learning rates help when the prior is misspecified (as found for linear regression), leading to faster learning. For pharmacokinetics, the data-generating parameter values are well-represented by the prior, meaning we can use smaller learning rates to focus more on robustly learning from data under model misspecification. Recall that R-IDeA performance is omitted from \Figref{allplotsmmd2} due to much poorer MMD in comparison to the other methods under misspecification. 

For tabulated results, \Tableref{pharmaresultsfullomega1} presents the results of using $\omega = 0.6$. \Tableref{pharmaresultsfullomega2} presents the same when $\omega = 0.8$.

We find that in the well-specified scenario, the performance of GBOED is much stronger with a large learning rate, which is expected (in this scenario, the data provide reliable information about the best-fitting parameter values). In the misspecified scenarios, GBOED with weighted score matching, a loss function designed to discriminate between inliers and outliers, can fall behind competing methods. The larger learning rate $\omega = 0.8$ likely results in the weights assigned to the data by the IMQ kernel being too large, overwhelming the variance in the weights assigned to different data points. This would mean inliers and outliers have similar weights (when they should be more distinct). GBOED with power likelihoods, a loss function which does not discriminate between inliers and outliers, effectively benefits from the larger learning rate, often achieving lower MMD than BOED. We speculate that this is due to the negative log-likelihood loss (used in the power likelihood) being a better match for larger learning rates in this problem. Notice that the power likelihood (with $\omega = 0.8$) trails ahead of BOED (which uses $\omega = 1$), versus the power likelihood with the small $\omega = 0.2$, which instead fails to improve over BOED. \Tableref{pharmaresultsfullomega1} suggests the same when $\omega = 0.6$ (further improved performance over BOED). 

We find here that the larger learning rate results in power likelihoods being stronger than weighted score matching in both well-specified and misspecified scenarios. Power likelihoods are on par with BOED, with clear improvements in the incorrect error distribution scenario. One can suspect that power likelihoods (with suitable learning rates) may be a better match for pharmacokinetics than score matching, at least in the scenarios considered.

Under misspecification, random acquisition or the BEIG with GBI can perform better than using the Gibbs EIG with the learning rates used to compute the statistics in \Tableref{pharmaresultsfullomega1} or \Tableref{pharmaresultsfullomega2} --- particularly in the Asymmetric Outliers setting. However, the performance is still generally not better than using GBOED (with the Gibbs EIG) when the learning rates are those used to compute the statistics in \Tableref{pharmaresultsfull}. This makes learning rate selection crucial. With the learning rates used in \Tableref{pharmaresultsfullomega1} and \Tableref{pharmaresultsfullomega2}, GBOED with power likelihoods often still attains lower MMD under misspecification.

\begin{table}[ht!]
\caption{Comparison of methods under misspecified scenarios in the pharmacokinetics problem, with $\omega = 0.6$ if not stated. $(q_{1}, q_{2}) = (0.8, 0.2)$ and $b = 0.30$ for the exponential decay method. RMSE, MMD, and NLL are reported with mean (and standard error) over 100 replications.}
\centering
\begin{tabular}{lccc}
\toprule
\textbf{Method} & \textbf{RMSE} & \textbf{MMD} & \textbf{NLL} \\
\midrule
\multicolumn{4}{c}{Asymmetric Outliers} \\[0.2em]
BOED & $1.9233\,(0.0480)$ & $0.5155\,(0.0346)$ & $4.3711\,(0.4065)$ \\
R-IDeA & \textbf{1.5297\,(0.0624)} & $1.3524\,(0.0298)$ & $114.3451\,(11.5569)$ \\
Power-Like $\omega = 0.6$ & $1.9235\,(0.0480)$ & \textbf{0.4999\,(0.0336)} & \textbf{3.8985\,(0.3503)} \\
\citet{laplante2025robustconjugatespatiotemporalgaussian} & $6.1839\,(0.6953)$ & $0.7486\,(0.0576)$ & $10.8593\,(1.1219)$ \\
Exp-Decay $b = 0.30$ & $4.6345\,(0.4344)$ & $0.7935\,(0.0537)$ & $9.8364\,(0.8213)$ \\
\textcolor{blindred}{Random} + \textcolor{blindblue}{Exp-Decay} & $3.6112\,(0.2577)$ & $0.5922\,(0.0333)$ & $4.7512\,(0.3751)$ \\
\textcolor{blindred}{BEIG} + \textcolor{blindblue}{Exp-Decay} & $4.0378\,(0.4088)$ & $0.6426\,(0.0428)$ & $6.7932\,(0.6321)$ \\
\midrule
\multicolumn{4}{c}{Misspecified Error Distribution} \\[0.2em]
BOED & $2.5272\,(0.0678)$ & $0.2370\,(0.0096)$ & $4.6730\,(0.1711)$ \\
R-IDeA & \textbf{0.8582\,(0.0250)} & $0.4954\,(0.0231)$ & $16.2933\,(1.2045)$ \\
Power-Like $\omega = 0.6$ & $2.5236\,(0.0702)$ & \textbf{0.2180\,(0.0101)} & \textbf{4.4135\,(0.1624)} \\
\citet{laplante2025robustconjugatespatiotemporalgaussian} & $2.7631\,(0.0685)$ & $0.3113\,(0.0143)$ & $6.0006\,(0.2831)$ \\
Exp-Decay $b = 0.30$ & $2.8232\,(0.0668)$ & $0.3277\,(0.0153)$ & $6.2003\,(0.2830)$ \\
\textcolor{blindred}{Random} + \textcolor{blindblue}{Exp-Decay} & $3.1534\,(0.0744)$ & $0.3613\,(0.0132)$ & $5.7573\,(0.2198)$ \\
\textcolor{blindred}{BEIG} + \textcolor{blindblue}{Exp-Decay} & $2.9550\,(0.0612)$ & $0.3651\,(0.0126)$ & $6.4457\,(0.2918)$ \\
\bottomrule
\end{tabular}
\label{pharmaresultsfullomega1}
\end{table}

\begin{table}[ht!]
\caption{Comparison of methods under both well-specified and misspecified scenarios in the pharmacokinetics problem, with $\omega = 0.8$ if not stated. $(q_{1}, q_{2}) = (0.8, 0.2)$ and $b = 0.30$ for the exponential decay method. RMSE, MMD, and NLL are reported with mean (and standard error) over 100 replications.}
\centering
\begin{tabular}{lccc}
\toprule
\textbf{Method} & \textbf{RMSE} & \textbf{MMD} & \textbf{NLL} \\
\midrule
\multicolumn{4}{c}{Well-Specified} \\[0.2em]
BOED & $1.4027\,(0.0097)$ & $0.1112\,(0.0055)$ & $1.3728\,(0.0184)$ \\
R-IDeA & \textbf{0.2441\,(0.0038)} & $0.1124\,(0.0098)$ & \textbf{-0.2322\,(0.1063)} \\
Power-Like $\omega = 0.8$ & $1.3777\,(0.0085)$ & \textbf{0.0963\,(0.0053)} & $1.3128\,(0.0163)$ \\
\citet{laplante2025robustconjugatespatiotemporalgaussian} & $1.5445\,(0.0224)$ & $0.2203\,(0.0138)$ & $1.8882\,(0.0750)$ \\
Exp-Decay $b = 0.30$ & $1.4560\,(0.0188)$ & $0.1668\,(0.0109)$ & $1.6234\,(0.0566)$ \\
\textcolor{blindred}{Random} + \textcolor{blindblue}{Exp-Decay} & $2.0150\,(0.0412)$ & $0.3412\,(0.0153)$ & $2.9297\,(0.1084)$ \\
\textcolor{blindred}{BEIG} + \textcolor{blindblue}{Exp-Decay} & $1.9511\,(0.0299)$ & $0.3948\,(0.0181)$ & $2.9115\,(0.1554)$ \\
\midrule
\multicolumn{4}{c}{Asymmetric Outliers} \\[0.2em]
BOED & $1.9233\,(0.0480)$ & $0.5155\,(0.0346)$ & $4.3711\,(0.4065)$ \\
R-IDeA & \textbf{1.5297\,(0.0624)} & $1.3524\,(0.0298)$ & $114.3451\,(11.5569)$ \\
Power-Like $\omega = 0.8$ & $1.9150\,(0.0477)$ & \textbf{0.5110\,(0.0347)} & \textbf{4.1902\,(0.3826)} \\
\citet{laplante2025robustconjugatespatiotemporalgaussian} & $6.2491\,(0.7808)$ & $0.7582\,(0.0578)$ & $11.1267\,(1.1755)$ \\
Exp-Decay $b = 0.30$ & $4.2890\,(0.4094)$ & $0.7315\,(0.0546)$ & $9.4136\,(0.9299)$ \\
\textcolor{blindred}{Random} + \textcolor{blindblue}{Exp-Decay} & $3.6392\,(0.2602)$ & $0.6128\,(0.0344)$ & $5.4685\,(0.4816)$ \\
\textcolor{blindred}{BEIG} + \textcolor{blindblue}{Exp-Decay} & $4.8548\,(0.4513)$ & $0.7319\,(0.0455)$ & $8.8653\,(0.8218)$ \\
\midrule
\multicolumn{4}{c}{Misspecified Error Distribution} \\[0.2em]
BOED & $2.5272\,(0.0678)$ & $0.2370\,(0.0096)$ & $4.6730\,(0.1711)$ \\
R-IDeA & \textbf{0.8582\,(0.0250)} & $0.4954\,(0.0231)$ & $16.2933\,(1.2045)$ \\
Power-Like $\omega = 0.8$ & $2.5121\,(0.0695)$ & \textbf{0.2234\,(0.0100)} & \textbf{4.5181\,(0.1650)} \\
\citet{laplante2025robustconjugatespatiotemporalgaussian} & $2.7255\,(0.0707)$ & $0.3286\,(0.0149)$ & $6.7921\,(0.4005)$ \\
Exp-Decay $b = 0.30$ & $2.6703\,(0.0710)$ & $0.3000\,(0.0139)$ & $6.2790\,(0.3457)$ \\
\textcolor{blindred}{Random} + \textcolor{blindblue}{Exp-Decay} & $3.1165\,(0.0878)$ & $0.3832\,(0.0145)$ & $6.6596\,(0.2986)$ \\
\textcolor{blindred}{BEIG} + \textcolor{blindblue}{Exp-Decay} & $3.1270\,(0.1215)$ & $0.4291\,(0.0202)$ & $8.3231\,(0.4973)$ \\
\bottomrule
\end{tabular}
\label{pharmaresultsfullomega2}
\end{table}

\paragraph{Constant elasticity of substitution} \Figref{allplotsmmd2} shows results with $\omega = 0.8$ (while \Figref{allplotsmmd} shows results with $\omega = 0.6$), which again results in faster learning, but less robustness. We find here that random acquisition with GBI is very effective, unlike in the other three considered experimental design problems (where it generally fails). GBOED remains competitive in all scenarios, generally outperforming BOED. 

Tabulated results with alternative learning rates $\omega = 0.2$ and $\omega = 0.8$ are presented in \Tableref{cesresultslr1}. GBOED with $\omega = 0.2$ results in the lowest NLL in both misspecified scenarios. The RMSE and NLL can disagree with each other, but generally either random acquisition or the Gibbs EIG achieves the lowest.

\begin{table}[ht!]
\caption{Comparison of methods under both well-specified and misspecified scenarios in the constant elasticity of substitution problem. Power likelihoods are compared with learning rates $\omega = 0.2$ and $\omega = 0.8$. RMSE and NLL are recorded with mean (and standard error) over 50 replications.}
\centering
\begin{tabular}{lcc}
\toprule
\textbf{Method} & \textbf{RMSE} & \textbf{NLL} \\
\midrule
\multicolumn{3}{c}{Well-Specified} \\[0.2em]
BOED & $0.1440\,(0.0125)$ & $1.1130\,(0.4905)$ \\
Power-Like $\omega = 0.2$ & \textbf{0.1245\,(0.0107)} & $-0.2648\,(0.3617)$ \\
\textcolor{blindred}{Random} + \textcolor{blindblue}{Power-Like $\omega = 0.2$} & $0.1322\,(0.0100)$ & $-0.3875\,(0.3657)$ \\
\textcolor{blindred}{BEIG} + \textcolor{blindblue}{Power-Like $\omega = 0.2$} & $0.1302\,(0.0127)$ & $0.2129\,(0.4470)$ \\
Power-Like $\omega = 0.8$ & $0.1317\,(0.0117)$ & $0.5707\,(0.5557)$ \\
\textcolor{blindred}{Random} + \textcolor{blindblue}{Power-Like $\omega = 0.8$} & $0.1306\,(0.0116)$ & \textbf{-0.5914\,(0.3258)} \\
\textcolor{blindred}{BEIG} + \textcolor{blindblue}{Power-Like $\omega = 0.8$} & $0.1425\,(0.0133)$ & $1.5487\,(0.7563)$ \\
\midrule
\multicolumn{3}{c}{Asymmetric Outliers} \\[0.2em]
BOED & $0.2053\,(0.0185)$ & $7.4071\,(1.8552)$ \\
Power-Like $\omega = 0.2$ & $0.1523\,(0.0139)$ & \textbf{0.0729\,(0.3522)} \\
\textcolor{blindred}{Random} + \textcolor{blindblue}{Power-Like $\omega = 0.2$} & $0.1502\,(0.0118)$ & $0.5499\,(0.5286)$ \\
\textcolor{blindred}{BEIG} + \textcolor{blindblue}{Power-Like $\omega = 0.2$} & $0.1544\,(0.0146)$ & $1.3416\,(0.6257)$ \\
Power-Like $\omega = 0.8$ & $0.1680\,(0.0141)$ & $2.9163\,(0.8977)$ \\
\textcolor{blindred}{Random} + \textcolor{blindblue}{Power-Like $\omega = 0.8$} & \textbf{0.1458\,(0.0139)} & $1.4796\,(0.9783)$ \\
\textcolor{blindred}{BEIG} + \textcolor{blindblue}{Power-Like $\omega = 0.8$} & $0.1872\,(0.0179)$ & $6.1112\,(2.5055)$ \\
\midrule
\multicolumn{3}{c}{Misspecified Error Distribution} \\[0.2em]
BOED & $0.3456\,(0.0230)$ & $9.5071\,(1.3707)$ \\
Power-Like $\omega = 0.2$ & $0.3310\,(0.0267)$ & \textbf{6.3866\,(0.7981)} \\
\textcolor{blindred}{Random} + \textcolor{blindblue}{Power-Like $\omega = 0.2$} & $0.3276\,(0.0249)$ & $6.8719\,(1.0323)$ \\
\textcolor{blindred}{BEIG} + \textcolor{blindblue}{Power-Like $\omega = 0.2$} & $0.3189\,(0.0250)$ & $7.8018\,(1.1150)$ \\
Power-Like $\omega = 0.8$ & $0.3347\,(0.0256)$ & $7.8067\,(0.6632)$ \\
\textcolor{blindred}{Random} + \textcolor{blindblue}{Power-Like $\omega = 0.8$} & \textbf{0.3137\,(0.0255)} & $8.0205\,(1.1505)$ \\
\textcolor{blindred}{BEIG} + \textcolor{blindblue}{Power-Like $\omega = 0.8$} & $0.3344\,(0.0253)$ & $11.1690\,(1.7686)$ \\
\bottomrule
\end{tabular}
\label{cesresultslr1}
\end{table}

\subsection{Understanding and Interpreting the Gibbs EIG}

In this appendix, we give a qualitative treatment to understanding and interpreting the Gibbs EIG. \Appref{visgibbseig} illustrates the Gibbs EIG surfaces on two of the experimental design problems. \Appref{behaviouracq} presents the behaviour of each acquisition function (according to the specified loss) on the location finding problem. 

\subsubsection{Visualising the Gibbs EIG Surfaces}
\label{visgibbseig}

We can visualise the Gibbs EIG surfaces to understand qualitatively how much they differ from the BEIG, and the impact the choice of loss function has. This is done for the linear regression and pharmacokinetics problems.

\paragraph{Linear regression} \label{visualgibbs1} 
Figures \ref{eigcomparisonlearningr} and \ref{ceig_combined} display the effect of changing the learning rate $\omega$ and of changing $c$ under the weighted score matching loss, on the Gibbs EIG surfaces respectively. Tuning $c$ appears to behave similarly to tuning $\omega$ here, so we will focus discussion on \Figref{ceig_combined}.

For \Figref{ceig_combined}, a ``poor'' prior refers to simply using a unit Gaussian as a prior, which generally places very low probability on the true parameter values in our experimental setup. A ``good'' prior refers to a prior placing high density on the true data-generating parameter values, having lower variance as a result. The GEIG surfaces generally have similar qualitative shapes regardless of the prior choice.

The unweighted score matching loss and weighted score matching loss when $c = 10$ both exhibit behaviour similar to the BEIG, querying the extremes of the design space. As $c \to \infty$, the fraction in $r_{\mathrm{IMQ}}$ converges to zero, resulting in $r_{\mathrm{IMQ}}$ producing the same weight for all observations. This explains the closeness of unweighted score matching and weighted score matching with $c = 10$. In general, smaller values of $c$ appear to cause falls in the information gain one can expect to receive. When the prior is poor, smaller values of $c$ also appear to cause smaller relative differences in the GEIG values at and away from the extremes. When using a good prior, querying at the extremes seems to be preferred slightly more. 

\begin{figure}[ht!]
    \centering
    \includegraphics[width=0.5\linewidth]{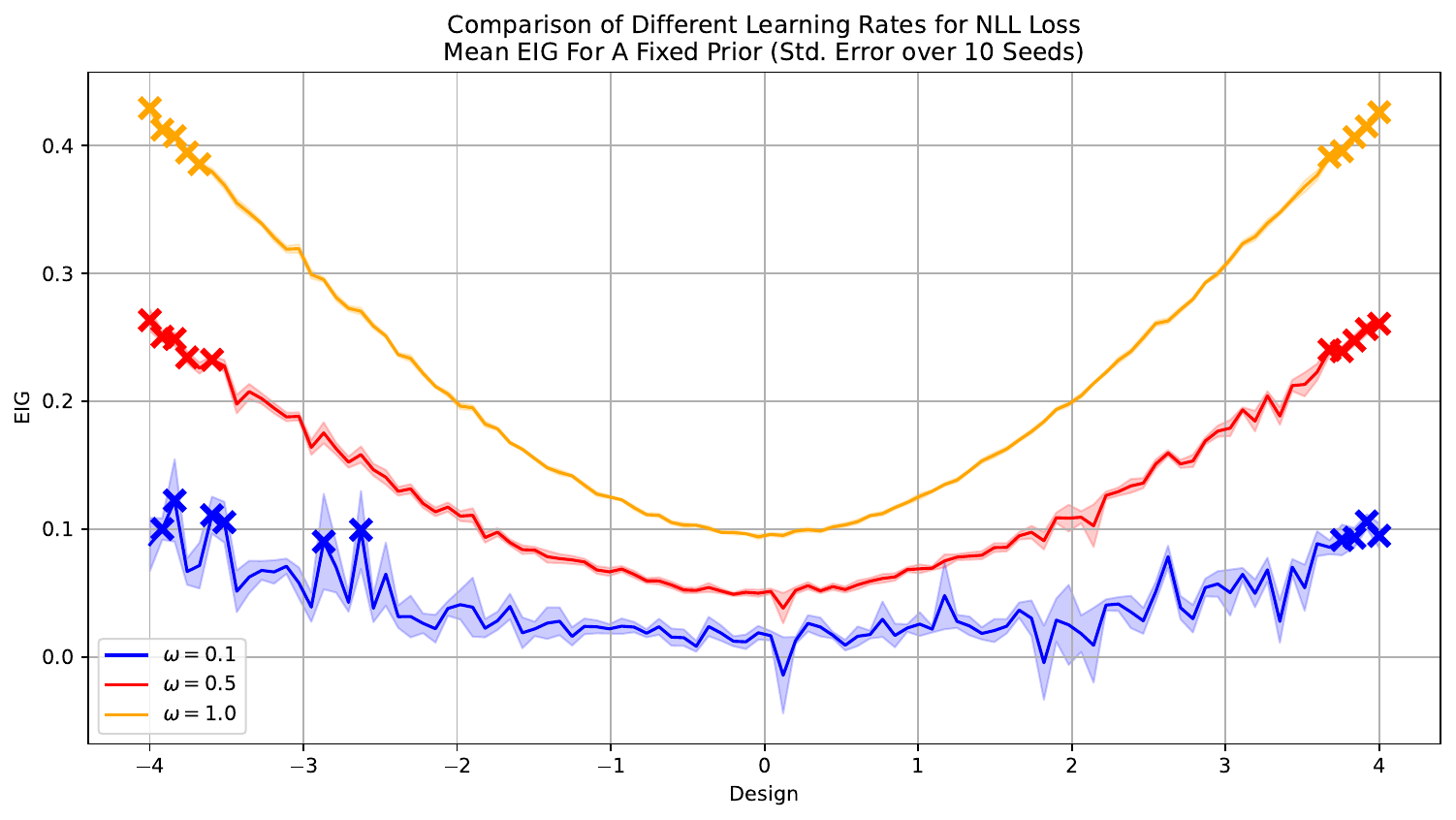}
    \caption{Comparison of different learning rates for computing the Gibbs EIG using the negative log-likelihood loss, in the context of the Bayesian linear regression problem. Marked crosses are the 10 designs with the greatest Gibbs EIG for a particular curve. \textcolor{blindblue}{\textit{Bottom}}: The Gibbs EIG at $\omega = 0.1$. EIG values are lower than the other two curves, and designs slightly further away from the extremes have greater EIG. \textcolor{blindred}{\textit{Middle}}: The Gibbs EIG at $\omega = 0.5$. Less smooth EIG curve than the curve above, with some indication of designs with greater EIG values being somewhat away from the extremes. \textcolor{blindorange}{\textit{Top}}: The BEIG ($\omega = 1$). Greatest EIG values are at the most extreme ends of the design space, and the EIG curve is much smoother than the other two.}
    \label{eigcomparisonlearningr}
\end{figure}

\begin{figure}[ht!]
    \centering
    \begin{minipage}[t]{0.48\linewidth}
        \centering
        \includegraphics[width=\linewidth]{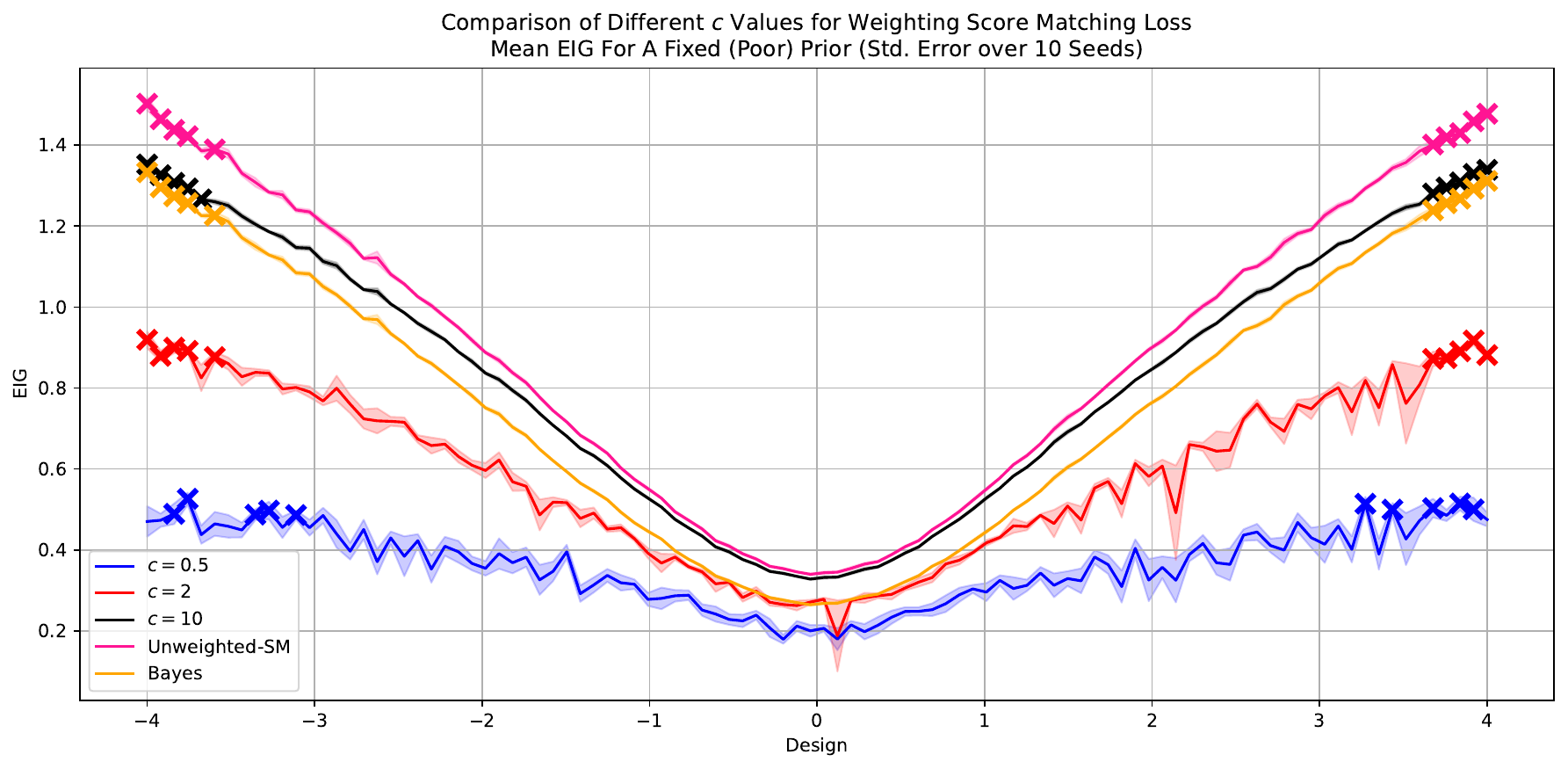}
        \label{ceig1}
    \end{minipage}
    \hfill
    \begin{minipage}[t]{0.48\linewidth}
        \centering
        \includegraphics[width=\linewidth]{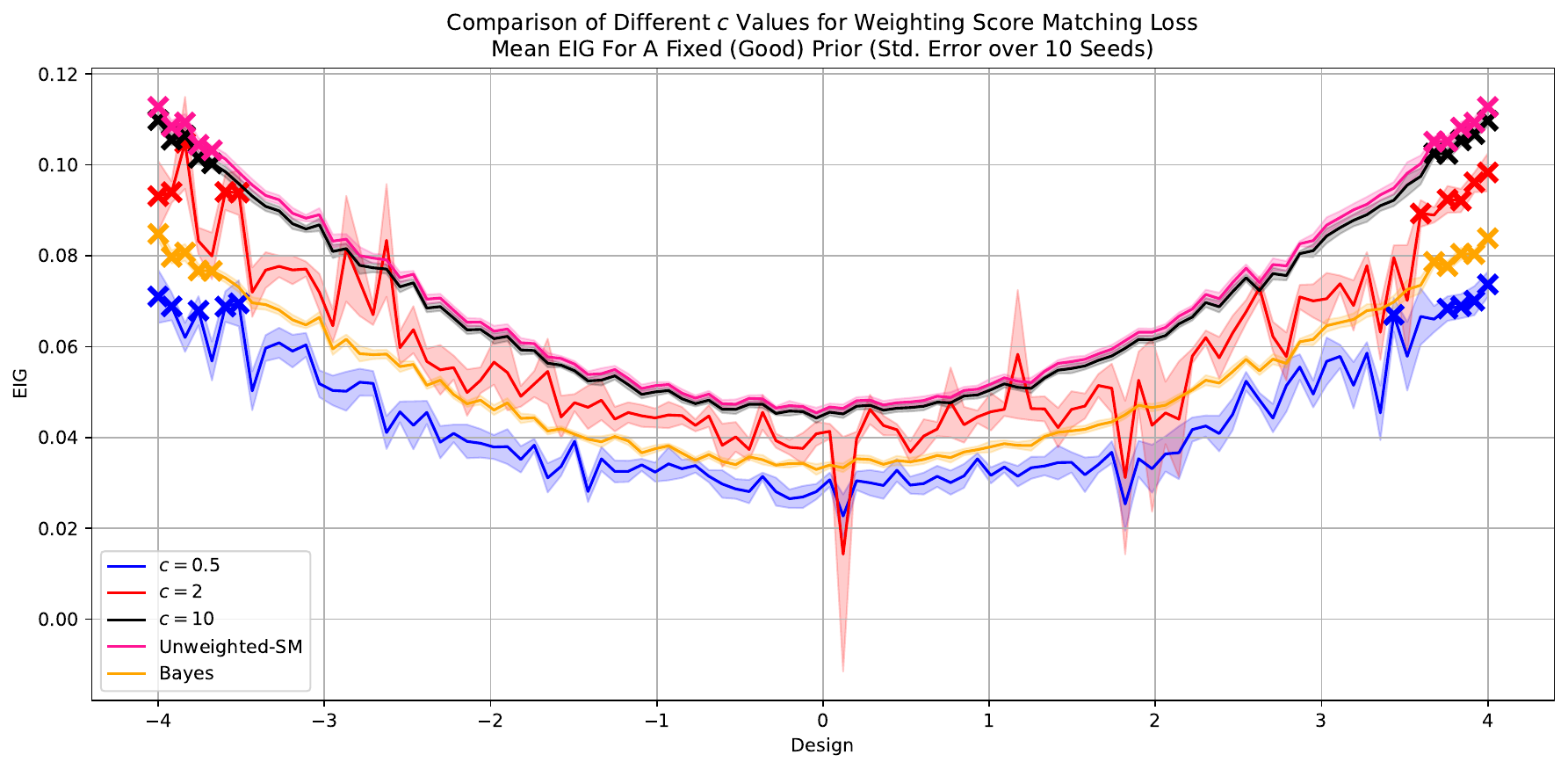}
        \label{ceig2}
    \end{minipage}
    \caption{Comparison of different downweighting rates $c$ on the Gibbs EIG under different priors, for the Bayesian linear regression problem. $\omega = 1$ across all curves. \textit{Left}: poor prior (unit Gaussian, larger variance). \textit{Right}: good prior (places high density on the data-generating parameter value, smaller variance). Smaller $c$ values tend to result in queries slightly away from the extremes, with more noise compared to the other curves. Large $c$ values have GEIG curves close to the unweighted score matching loss.}
    \label{ceig_combined}
\end{figure}

Values of $c$ that do not downweight observations too quickly (i.e., large values of $c$) would cause the Gibbs EIG to have the same parabolic form as the BEIG, and values of $c$ that seek to downweight observations much more cause deviations from the BEIG. This phenomenon is more pronounced in more complicated experimental design problems; see below for the pharmacokinetics problem. 

\paragraph{Pharmacokinetics} \label{visualgibbs2}
We present in \Figref{pharmaeigs} the (Gibbs) EIG for a fixed prior. When the loss function is the negative log-likelihood and $\omega = 1$, i.e., for the BEIG, the designs with maximal EIG are between $15 < \boldsymbol{\xi} < 20$. This remains the same when $\omega = 0.4$ (a power likelihood). When we use (unweighted) score matching, the Gibbs EIG is vastly different to the GEIG resulting from using other acquisition functions. The maximal GEIG is instead within $13 < \boldsymbol{\xi} < 17$, and the Gibbs EIG values are overall much larger than those computed under different loss functions. Weighted score matching instead assigns the greatest GEIG at $22 \leq \boldsymbol{\xi} \leq 24$ (the rightmost extreme), showing greater differences between the curves. This is unlike the curves plotted for linear regression in \Figref{eigcomparisonlearningr} and \Figref{ceig_combined}, where differences are less great.

\begin{figure}[ht!]
    \centering
    \includegraphics[width=120mm]{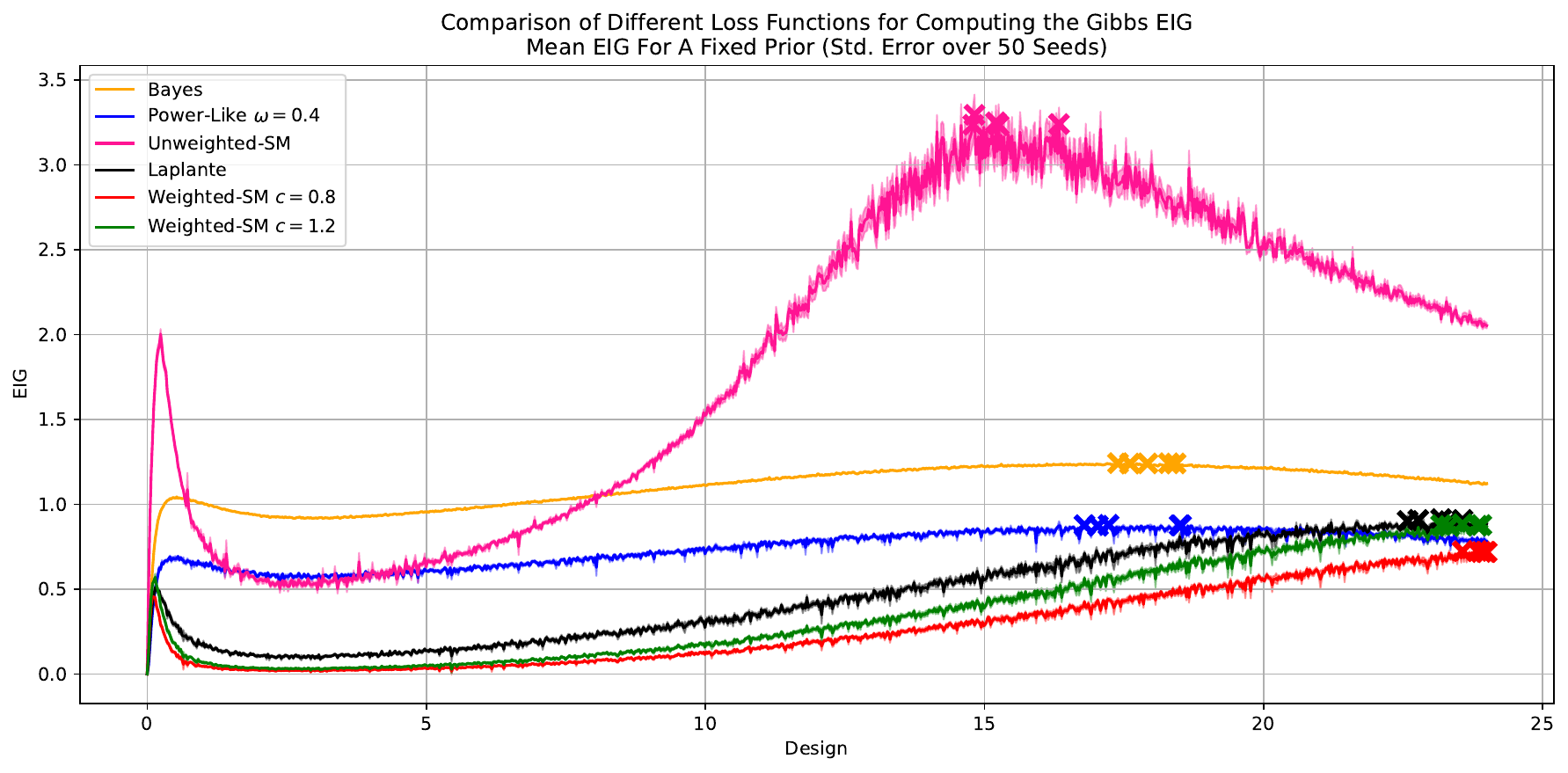}
    \caption{Comparison of different loss functions for computing the Gibbs EIG, under the pharmacokinetics problem. Marked crosses are the 5 designs with the greatest (Gibbs) EIG for a particular curve. The Gibbs EIG can vary greatly depending on the loss function used.}
    \label{pharmaeigs}
\end{figure}

\subsubsection{Behaviour of Each Acquisition Function}
\label{behaviouracq}

\paragraph{Location finding} Even though the Gibbs EIG is unique for each loss function, we can view how the GEIG acquired during experimentation varies between loss functions (which in turn deliver unique acquisition functions). To do so, we present \Figref{gibbseigcomp}, which displays the GEIG gained on the 2D location finding problem under both the well-specified and Asymmetric Outliers scenarios. Here, the GEIG computed is unique to each loss function.

The results are intriguing and explain how each method behaves at different points in time. Starting with the well-specified scenario, BOED (BEIG + Bayesian inference) appears to acquire information linearly until the 21st experiment, where it then starts to plateau. One can infer that there is less information that can be gathered about $\boldsymbol{\theta}$ going forward, which matches the MMD remaining at roughly zero from the 21st experiment. Using the BEIG with GBI, where the loss function for inference is based on our exponential decay method, suggests that the design sequence selected is on average expected to be more informative (according to the BEIG) than that chosen by BOED. This occurs due to the change in posterior inference, given that the BEIG tries to select optimal designs according to the Bayesian posterior, not the Gibbs posterior ultimately being used for inference. The BEIG believes there is more information to be gathered, given that GBI is learning from the data in a robust manner (lower learning rate, different loss function). GBOED (Gibbs EIG + GBI) can produce different behaviours depending on the loss function, as seen. Power likelihoods somewhat follow BOED in behaviour, given that the loss function remains the same across both, and only the learning rate is different.

In the Asymmetric Outliers scenario, we find that BOED remains linear in acquiring information during experimentation, rather than plateauing at a particular point in time. As we know from the predictive performance of BOED under misspecification, Bayesian inference is unable to handle outliers robustly. Since inference is done inadequately, little to no learning is achieved from the experiments conducted --- meaning that there is still a lot to be learnt about the environment. BEIG + GBI exhibits similar behaviour because it tries to measure informativeness according to a (Bayesian) posterior, which is poor at conducting inference in this setting. It deviates from BOED after a while as a result of actually conducting inference according to Gibbs posteriors, and assigns lower BEIG to the design sequence by the end of experimentation. Interestingly, only our exponential decay method manages to noticeably plateau, suggesting that there is less that can be learnt from the environment after conducting a dozen or so experiments --- as we would hope from performing experimental design.

\begin{figure}[ht!]
    \centering
    \includegraphics[width=\linewidth]{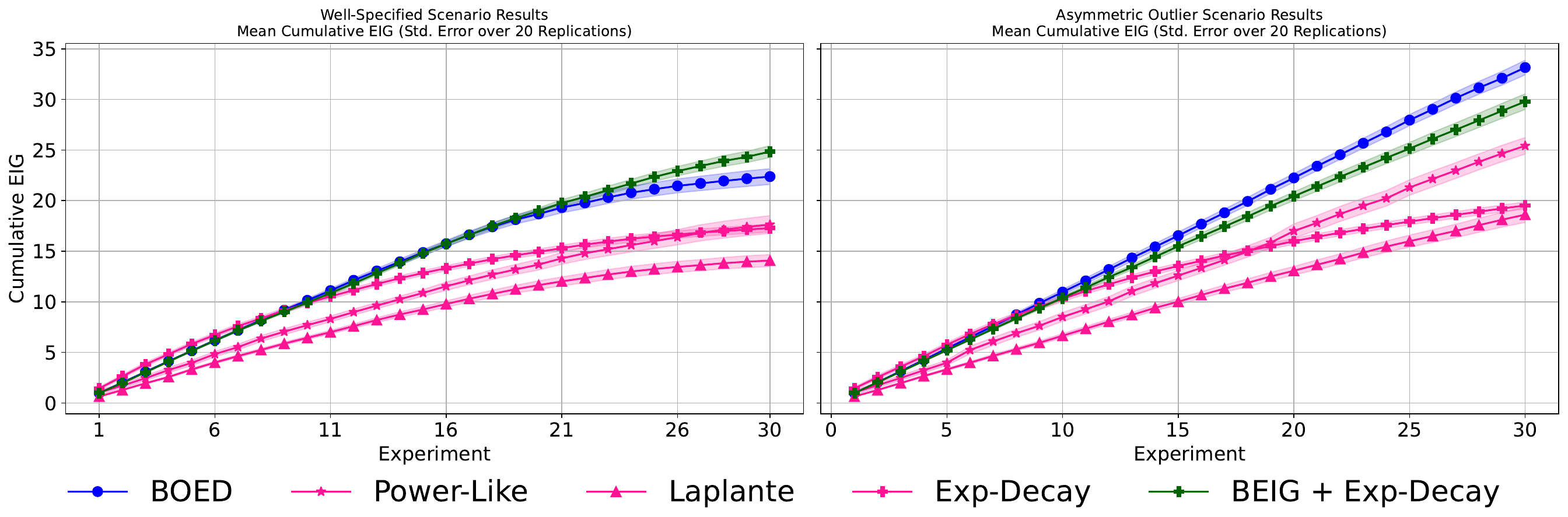}
    \caption{Comparison of methods in the (maximum) Gibbs EIG acquired during experimentation for the 2D location finding problem. Mean (and standard error) over 20 replications. \textit{Left}: Well-specified scenario. \textit{Right}: Asymmetric Outliers scenario. $\omega = 0.4$ if not stated.}
    \label{gibbseigcomp}
\end{figure}

\end{document}